%% file: main.tex
\documentclass[
    letterpaper,
    11pt,
]{article}

\usepackage[colorlinks=true,linkcolor=blue,citecolor = blue]{hyperref}
\usepackage[ruled,lined,algo2e,noend, linesnumbered]{algorithm2e}
\usepackage{new_macros}
\usepackage{figure_macros}
\usepackage{macros}
\usepackage{bbm}
\usepackage[numbers]{natbib}
\usepackage[margin=1in]{geometry}
\usepackage{graphicx}
\usepackage{empheq}
\usepackage[T1]{fontenc}
\usepackage{enumitem}
\usepackage{svg}
\usepackage{tikz}
\usepackage{xparse}
\usepackage{array}
\usepackage{booktabs}
\usetikzlibrary {arrows.meta,bending,positioning,matrix}

 \title{Efficient Near-Optimal Algorithm for Online Shortest Paths in Directed Acyclic Graphs with Bandit Feedback Against Adaptive Adversaries}

\author{Arnab Maiti\thanks{University of Washington. arnabm2@uw.edu (Equal contribution) } \and
Zhiyuan Fan \thanks{MIT. fanzy@mit.edu (Equal contribution)} \and
Kevin Jamieson \thanks{University of Washington. jamieson@cs.washington.edu}\and
Lillian J. Ratliff \thanks{University of Washington. ratliffl@uw.edu}
\and
Gabriele Farina \thanks{MIT. gfarina@mit.edu } 
}
\date{}
\begin{document}

\maketitle
\begin{abstract}%

In this paper, we study the online shortest path problem in directed acyclic graphs (DAGs) under bandit feedback against an adaptive adversary. Given a DAG $G = (V, E)$ with a source node $\source$ and a sink node $\sink$, let $\mathcal{X} \subseteq \{0,1\}^{|E|}$ denote the set of all paths from $\source$ to $\sink$. At each round $t$, we select a path $\bfx_t \in \mathcal{X}$ and receive bandit feedback on our loss $\langle \bfx_t, \bfy_t \rangle \in [-1,1]$, where $\bfy_t$ is an adversarially chosen loss vector. Our goal is to minimize regret with respect to the best path in hindsight over $T$ rounds. We propose the first computationally efficient algorithm to achieve a near-minimax optimal regret bound of $\tilde{\mathcal{O}}(\sqrt{|E|T\log |\mathcal{X}|})$ with high probability against any adaptive adversary, where $\tilde{\mathcal{O}}(\cdot)$ hides logarithmic factors in the number of edges $|E|$. Our algorithm leverages a novel loss estimator and a centroid-based decomposition in a nontrivial manner to attain this regret bound.

As an application, we show that our algorithm for DAGs provides state-of-the-art efficient algorithms for $m$-sets, extensive-form games, the Colonel Blotto game, shortest walks in directed graphs, hypercubes, and multi-task multi-armed bandits, achieving improved high-probability regret guarantees in all these settings.

\end{abstract}
\input{introduction}

\input{preliminaries}

\input{results}

\input{applications}
\input{conclusion}

\input{ack}

\bibliographystyle{plainnat}
\bibliography{refs.bib}
\appendix
\input{appendix}

\end{document}

%% file: introduction.tex
\section{Introduction}
Online decision-making is a well-studied area with applications in various domains, including recommendation systems, resource allocation, web ranking, shortest path planning, and portfolio selection (e.g., \citet{lin2020novel,chen2017online,frigo2022online,gordon2006no,das2014online}). In a typical online decision-making problem, a learner interacts with an adversary over multiple rounds. The learner is given a set of arms $\mathcal{X}$, and in each round $t$, selects an arm $\mathbf{x}_t \in \mathcal{X}$. Simultaneously, an adversary chooses a loss function $\mathbf{y}_t:\calX\to \mathbb{R}$. The learner then incurs a loss of $\mathbf{y}_t[\mathbf{x}_t] \in [-1,1]$ and observes only the incurred loss, not the full loss function $\mathbf{y}_t$. After $T$ rounds of interaction between the learner and the adversary, the learner’s regret relative to a fixed arm $\mathbf{x} \in \mathcal{X}$ is defined as  
\[
    R_T(\bfx) = \sum_{t=1}^{T} \mathbf{y}_t[\mathbf{x}_t] - \sum_{t=1}^{T} \mathbf{y}_t[\mathbf{x}].
\]  
The learner aims either to minimize the pseudo-regret, defined as $\max_{\mathbf{x} \in \mathcal{X}} \mathbb{E}[R_T(\bfx)]$, or to provide high-probability guarantees on $\max_{\mathbf{x} \in \mathcal{X}} R_T(\bfx)$. The latter is a stronger and more natural notion of regret, as $\mathbb{E}[\max_{\mathbf{x} \in \mathcal{X}} R_T(\mathbf{x})]$ can far exceed the pseudo-regret against an adaptive adversary.  

A major breakthrough in this area was achieved by \citet{auer2002nonstochastic}, who proposed the EXP3.P algorithm, which attains a regret of $\tilde{O}(\sqrt{KT})$ with high probability against an adaptive adversary, where $K$ is the number of arms in $\mathcal{X}$. This bound is optimal, as there exists a minimax lower bound of $\Omega(\sqrt{KT})$ for this problem.  

However, one can hope for better guarantees when the loss functions and the set of arms exhibit additional structure. \citet{awerbuch2004adaptive} took a step in this direction by considering the online shortest path problem in directed acyclic graphs (DAGs) under bandit feedback. In this setting, an adversary assigns loss values to each edge of a given DAG, and the learner must select a path from the source to the sink. Here, the set of arms consists of all such paths, and the loss of a path is defined as the sum of the losses of its edges. Applying EXP3.P to this problem results in regret that scales exponentially with the number of edges in the DAG, as the number of paths can be exponentially large. To overcome this, \citet{awerbuch2004adaptive} designed an algorithm that achieves a pseudo-regret of $T^{2/3}$ that also scales polynomially with the number of edges. Later, \citet{gyorgy2007line} extended this result, showing that a regret bound of $T^{2/3}$ can be achieved with high probability against adaptive adversaries, while maintaining polynomial dependence on the number of edges.  

Motivated by research on DAGs, a series of works have explored combinatorial linear bandits, where $\mathcal{X} \subseteq \{0,1\}^d$ and $\mathbf{y}_t$ is a linear loss function. Notably, the problem on DAGs is a special case of combinatorial linear bandits. Several algorithms have been developed in this setting, including Geometric Hedge \citep{dani2007price}, ComBand \citep{cesa2012combinatorial}, and EXP2 with John’s exploration \citep{bubeck2012towards}. The best known pseudo-regret bound, $\mathcal{O}(\sqrt{dT\log |\mathcal{X}|})$, is achieved by EXP2 with John’s exploration. Later, \citet{zimmert2022return} established a high-probability regret bound of $\mathcal{O}(\sqrt{dT\log |\mathcal{X}|})$ against adaptive adversaries for EXP3 with Kiefer-Wolfowitz exploration.  

While the above algorithms achieve low regret, they can be computationally inefficient. A series of works have addressed this issue for continuous sets in $\mathbb{R}^d$. \citet{abernethy2008competing} were the first to propose a computationally efficient algorithm that achieved a pseudo-regret of $\text{poly}(d) \cdot \sqrt{T}$. They also proposed a computationally efficient algorithm for the online shortest path problem on DAGs with the same pseudo-regret. The best known pseudo-regret for a computationally efficient algorithm is $\tilde{O}(d\sqrt{T})$, attained by the algorithms in \citet{hazan2016volumetric} and \citet{ito2020tight}. The efficient algorithm by \citet{hazan2016volumetric} also matches the pseudo-regret of EXP2 with John’s exploration for the online shortest path problem on DAGs. 

For continuous sets in $\mathbb{R}^d$, \citet{lee2020bias} proposed the first efficient algorithm achieving a \textit{high-probability} regret of $\text{poly}(d) \cdot \sqrt{T}$ against an adaptive adversary. Later, \citet{zimmert2022return} developed an efficient algorithm with a regret of $\tilde{\mathcal{O}}(d^2\sqrt{T})$ with high probability against an adaptive adversary, which remains the best known result to date. For a more detailed discussion of all the related works, we refer the reader to \Cref{appendix:related-works}.

In this paper, we revisit the online shortest path problem in DAGs—the motivation behind much of the prior work—and pose the following question:  
\begin{quote}  
\emph{Can we design a computationally efficient algorithm for the online shortest path problem in a directed acyclic graph that, under bandit feedback, achieves a minimax-optimal regret bound with high probability against adaptive adversaries, up to logarithmic factors in the number of edges?}  
\end{quote}  
\subsection{Contributions and Techniques}

In this paper, we answer the above question in the affirmative. For any directed acyclic graph (DAG) $G = (V, E)$ with a set of paths $\mathcal{X} \subseteq \{0,1\}^{E}$ from source to sink, we design the first computationally efficient algorithm to achieve a high-probability regret bound of $\tilde{\mathcal{O}}(\sqrt{|E|T\log |\mathcal{X}|})$ against an adaptive adversary under bandit feedback, where $\tilde{\mathcal{O}}(\cdot)$ hides logarithmic factors in $|E|$. We refer the reader to \Cref{table1} for a comparison of our result with previous algorithms. Moreover, for the class of DAGs with at most $d$ edges and at most $N$ paths, we establish a minimax lower bound of $\Omega\left(\sqrt{dT \log(N)/\log(d)}\right)$. Hence, our algorithm is minimax-optimal upto logarithmic factors.  %

\begin{table}[t]
    \centering
    \begin{tabular}{p{5.5cm}p{3.3cm}cc}
        \toprule
        \textbf{Reference} & \textbf{Regret} & \textbf{Efficient} & \textbf{Adaptive \& high-prob.} \\
         \midrule
        \citet{bubeck2012towards} & $\sqrt{|E|T\log |\calX|}$ & \no & \no \\
        \citet{zimmert2022return} & $\sqrt{|E|T\log |\calX|}$ & \no & \yes \\
        \citet{abernethy2008competing} & $\sqrt{|E|^3 T}$ & \yes & \no \\   
        \citet{hazan2016volumetric} & $\sqrt{|E|T\log |\calX|}$ & \yes & \no \\
        \citet{ito2020tight} & $\sqrt{|E|^2T}$ & \yes & \no \\
        \citet{lee2020bias} & $\sqrt{|E|^{7}T}$ & \yes & \;\yes$^\ddagger$ \\
        \citet{zimmert2022return} & $\sqrt{|E|^{4}T}$ & \yes & \;\yes$^\ddagger$ \\
        \midrule
        \textbf{This paper} (Theorem~\ref{actual:thm}) & $\sqrt{|E|T\log |\calX|}$ & \yes & \yes \\
        \bottomrule
    \end{tabular}
    \caption{Summary of regret guarantees for the online shortest path problem on a directed acyclic graph (DAG) $G = (V, E)$, with the set of paths $\mathcal{X} \subseteq \{0,1\}^E$ from source to sink, ignoring constants and logarithmic factors in $|E|$ and $T$. $^\ddagger$The high-probability guarantee was formally proved only for continuous sets. However, we believe that their analysis extends to discrete decision sets, such as paths in a DAG, using the same techniques as \citet{abernethy2008competing}.}\label{table1}
\end{table}

We further apply our efficient algorithm to combinatorial domains such as hypercubes, multi-task multi-armed bandits (MAB), extensive-form games, walks in directed graphs, the Colonel Blotto game, and $m$-sets, all of which can be represented as DAGs. This results in improved high-probability regret bounds in each setting compared to those in \citet{zimmert2022return}. For a detailed discussion of these improvements, we refer the reader to \Cref{sec:applications-main}. 

Our main technical contribution is a novel algorithmic approach for regret minimization on DAGs. Prior works relied on variants of exponential weights or FTRL, requiring mixing with a fixed distribution before selecting a path. Instead, we use a novel importance-sampling-inspired loss estimator to enable implicit exploration and apply centroid-based decomposition to modify the input graph, achieving a nearly minimax optimal bound.

Our algorithm proceeds in two steps. The first step is to design an algorithm for graphs with $ |V| $ vertices, $ |E| $ edges, and a longest path length of $ K $. While a path can be represented using $ |E| $ bits (one per edge), we introduce an extended representation with additional $\mathcal{O}(|V| + K)$ bits and denote the corresponding set of paths as $\mathcal{X}^\dagger$. We then solve the FTRL optimization problem:  
\[
\tilde{\mathbf{x}}_{t} \leftarrow \argmin_{\mathbf{x} \in \operatorname{co}(\mathcal{X}^\dagger)} \left( \eta \sum_{s=1}^{t-1} \langle \mathbf{x}, \widehat{\mathbf{y}}_s \rangle + F(\mathbf{x}) \right),
\]  
where $F(\cdot)$ is a Legendre function, and in our work, we use the Tsallis-1/2 entropy. We then efficiently sample a path $\mathbf{x}_t$ such that its expectation is $\tilde{\mathbf{x}}_t$. We then introduce a novel importance-sampling-inspired loss estimator $\tilde{\mathbf{y}}_t$, ensuring that the difference $\langle \bfx_1, \tilde{\mathbf{y}}_t \rangle - \langle \bfx_2, \tilde{\mathbf{y}}_t \rangle$ remains an unbiased estimate of the loss difference between any two paths encoded as $\bfx_1, \bfx_2$ in the DAG, even though $\tilde{\mathbf{y}}_t$ itself is not an unbiased estimator of the actual loss vector. Building on this, we perform implicit exploration, similar to standard multi-armed bandits \citep{neu2015explore}, by introducing a bias in $\tilde{\mathbf{y}}_t$ to construct our final estimator $\hat{\bfy}_t$, thereby achieving a high-probability regret bound of $\tilde{\mathcal{O}}(\sqrt{K |E| T})$ against any adaptive adversary.

The second step of our algorithmic approach considers the problem on an arbitrary DAG $G = (V, E)$ with the set of all paths from source to sink denoted by $\calX$ and reduces it to a problem on a newly constructed DAG $G^\dagger = (V^\dagger, E^\dagger)$ that satisfies several key properties. The length of any path from source to sink in $G^\dagger$ is $\calO(\log |\calX|)$, while the number of vertices and edges satisfy $|V^\dagger| = \calO(|V|)$ and $|E^\dagger| = \tilde{\calO}(|E|)$, respectively. Additionally, there exists a bijective mapping between the paths in $G$ and $G^\dagger$, which can be efficiently computed. To achieve this reduction, we introduce a novel centroid--based decomposition approach. Applying our FTRL method to $G^\dagger$, we obtain a high-probability regret bound of $\tilde{\calO}(\sqrt{|E|T\log |\calX|})$ against any adaptive adversary.

%% file: preliminaries.tex
\section{Preliminaries}
Let $G = (V, E)$ be a Directed Acyclic Graph (DAG), where $V$ is the set of vertices and $E \subseteq V \times V$ is the set of directed edges. A path $P = (v_0, e_1, v_1, \dots, e_k, v_k)$ of length $k > 0$ is an interleaved sequence of vertices and edges satisfying $v_i \in V$ for $i \in \{0,1,\ldots, k\}$ and $e_i = (v_{i-1}, v_i) \in E$ for $i \in \{1,\ldots,k\}$. Since $G$ is acyclic, no path $P$ can exist with $v_0 = v_k$.

For a vertex $v \in V$, the set of incoming edges is denoted by $\incoming(v) := \{(u, v) \in E\}$, and the set of outgoing edges is denoted by $\outgoing(v) := \{(v, u) \in E\}$. Given a weight function $w: E \to \mathbb{R}$ that assigns a weight to each edge, the shortest path problem seeks to find a path $P$ from a source vertex $\source$ to a sink vertex $\sink$ that minimizes the total weight of the edges along the path, given by  
\[
w(P) := \sum_{i=1}^k w(e_i).
\]  
Without loss of generality, we assume that every vertex $v$ is reachable from $\source$ and can reach $\sink$.

We consider the online shortest path problem with bandit feedback. In each round $t$, an agent selects a path $P_t$ from $\source$ to $\sink$, while an adversary simultaneously selects a weight function $w_t(\cdot)$. The agent then observes \textit{only} the loss, which is the path weight $\ell_t := w_t(P_t)$. The objective is to minimize the cumulative regret against the optimal path:  
\[
\mathrm{Regret}(T) := \sum_{t=1}^T w_t(P_t) - \min_{P \in \mathcal{P}} \sum_{t=1}^T w_t(P),
\]  
where $\mathcal{P}$ is the set of all paths from $\source$ to $\sink$.

Denote by $\mathcal{X} \subseteq \{0, 1\}^{V \cup E}$ the set of all paths in the graph $G$ from $\source$ to $\sink$, indexed by the vertices in $V$ and the edges in $E$. Each vector $\bfx \in \mathcal{X}$ encodes a path in the graph, where $\bfx[v] = 1$ indicates that $v \in V$ appears in the path, and $\bfx[e] = 1$ indicates that $e \in E$ appears in the path. The convex hull of $\mathcal{X}$ forms the flow polytope:
\begin{align*}
    \mathrm{co}(\mathcal{X}) = \left\{ \bfx \in  [0, 1]^{V \cup E}:\; 
    \bfx[\source] = \bfx[\sink] = 1, \;\text{and}\; \bfx[v]=\sum_{e\in \incoming(v)} \bfx[e] = \sum_{e\in \outgoing(v)} \bfx[e], \;  \forall v \in V 
    \right\}.
\end{align*}
Correspondingly, the weight function $w_t(\cdot)$ can be encoded as a vector $\bfy_t \in \mathbb{R}^{V \cup E}$, where $\bfy_t[e] = w_t(e)$ for all edges $e \in E$ and $\bfy_t[v] = 0$ for all vertices $v \in V$. In this formulation, the total path weight can be expressed as the inner product $w_t(P_t) = \langle \bfx_t, \bfy_t \rangle$, allowing the regret to be rewritten as:  
\[
    \mathrm{Regret}(T) = \sum_{t=1}^T \langle \bfx_t, \bfy_t \rangle - \min_{\bfx \in \mathcal{X}} \sum_{t=1}^T \langle \bfx, \bfy_t \rangle.
\]

Finally, we denote by $\mathcal{F}_{t} := \{\bfx_{\tau}, \bfy_{\tau}\}_{\tau=1}^{t}$ the filtration generated by the first $t$ rounds. We further use $\bbP_t[\cdot] := \bbP[\cdot| \calF_{t-1}]$ as the conditional probability and $\bbE_t[\cdot] := \bbE[\cdot|\calF_{t-1}]$ as the conditional expectation. The adversary is allowed to choose loss vector $\bfy_t$ that adapts to the past filtration and the agent's algorithm.

Throughout this paper, we impose the following standard assumption.  
\begin{assumption} \label{ass:reward}  
The adversary can only choose weight function $w$ such that the absolute weight of any path is at most 1. That is, it can only choose $\bfy$ satisfying $\langle \bfx, \bfy \rangle \in [-1, 1]$ for all $\bfx \in \mathcal{X}$.  
\end{assumption}  

\paragraph{General notations.}
We define $\llbracket k \rrbracket:=\{1, 2, \dots, k\}$ and $\llbracket a, b\rrbracket := \{a, a+1, \dots, b\}$. Denote by $2^C$ the power set of set $C$. Let $\emptyset$ denote the empty set. The logarithm of $x$ to base $2$ is denoted as $\log x$. For any pair of tuples $A = (a_1, \ldots, a_n)$ and $B = (b_1, \ldots, b_m)$, let $A \circ B$ denote the tuple $(a_1, \ldots, a_n, b_1, \ldots, b_m)$. Similarly, vectors $\bfx$ and $\bfy$, let $\bfz = \bfx \circ \bfy$ denote the vector obtained by concatenating $\bfy$ to the end of $\bfx$. 
For any edge $e$ and path $P$, $e \in P$ indicates $e$ is part of $P$.

%% file: results.tex
\section{Algorithm for Online Shortest Paths in DAGs}

In this section, we present our algorithm for the online shortest path problem in directed acyclic graphs (DAGs). Our approach differs from the standard method of using exponential weights combined with a fixed distribution, such as Kiefer-Wolfowitz exploration. We start by outlining an efficient algorithm for the case where all paths have equal lengths in \Cref{sec:alg-equal}. In \Cref{sec:alg-make-equal}, we introduce a method to relax this assumption. Finally, in \Cref{sec:alg-centroid}, we show how to achieve a regret bound of $\tilde \calO(\sqrt{|E|T\log|\calX|})$ while maintaining computational efficiency.

\subsection{The Case of Equal Path Lengths}\label{sec:alg-equal}

We first present an algorithm for a DAG $G = (V, E)$, where \emph{every} path from the source $\source$ to the sink $\sink$ contains exactly $K$ edges. In each round $t$, the algorithm chooses a strategy in the flow polytope $\mathrm{co}(\calX)$ by solving the following optimization problem:
\begin{align} \label{eq:optimization1}
    \tilde{\bfx}_{t} \leftarrow \argmin_{\bfx \in \mathrm{co}(\calX)} \left( \eta \sum_{\tau=1}^{t-1} \langle \bfx, \widehat{\bfy}_\tau \rangle + F(\bfx) \right),
\end{align}
where $F(\bfx)$ is some Legendre function, $\eta$ is some learning rate and $\widehat\bfy_\tau$ is some loss estimator that we define later. The actual path $\bfx_t$ is then sampled as follows: Starting from the source $\source$, we traverse to a node $v$ and select an edge $e \in \delta^{+}(v)$ among the outgoing edges of $v$ with probability proportional to $\tilde \bfx_{t}[e]$, moving to the endpoint of edge $e$. This process repeats until we reach the sink $\sink$. We denote the path traversed as $P_t$ and choose the corresponding vector in $\calX$ as $\bfx_t$. It can be easily verified that $\bbE_t[\bfx_t] =\tilde \bfx_t$. We then observe the loss $\ell_t := \langle \bfx_t, \bfy_t \rangle$, construct our loss estimator $\widehat{\bfy}_t$ as shown below, and proceed to the next round.

Recall that $\mathbb{E}_{t}[\cdot]:=\mathbb{E}[\cdot|\calF_{t-1}]$ and $\mathbb{P}_{t}[\cdot]: =\mathbb{P}[\cdot|\calF_{t-1}]$, where $\calF_{t-1}$ is the past filtration. Let $\boldsymbol{\gamma} \in \mathbb{R}_{>0}^{V \cup E}$ be a positive-valued vector indexed by the elements of $V \cup E$. We start with defining our estimator $\hat \bfy_t$ for the loss vector $\bfy_t$ upon receiving the loss $\ell_t:=\langle \bfx_t,\bfy_t\rangle$: 
\begin{align*}
    \hat{\bfy}_t[e] &:= \frac{(1 + \ell_t)\ind[\bfx_t[e] = 1]}{\mathbb{P}_t[\bfx_t[e] = 1]+\boldsymbol{\gamma}[e]}, \quad \forall e \in E, \\
    \hat{\bfy}_t[v] &:= \frac{(1 - \ell_t)\ind[\bfx_t[v] = 1]}{\mathbb{P}_t[\bfx_t[v] = 1]+\boldsymbol{\gamma}[v]}, \quad \forall v \in V\setminus \{\source,\sink\}, \quad\hat{\bfy}_t[\source]=\hat{\bfy}_t[\sink]:=0.
\end{align*}
Note that even though the loss vector satisfies $\bfy_t[v] = 0$ for any vertex $v \in V$, the estimator is still designed to assign weights to it. Next, let us define another estimator $\tilde\bfy_t$ as: 
\begin{align*}
    \tilde{\bfy}_t[e] &:= \frac{(1 + \ell_t)\ind[\bfx_t[e] = 1]}{\mathbb{P}_t[\bfx_t[e] = 1]}, \quad \forall e \in E, \\
    \tilde{\bfy}_t[v] &:= \frac{(1 - \ell_t)\ind[\bfx_t[v] = 1]}{\mathbb{P}_t[\bfx_t[v] = 1]}, \quad \forall v \in V\setminus \{\source,\sink\},\quad\tilde{\bfy}_t[\source]=\tilde{\bfy}_t[\sink]:=0.
\end{align*}
Observe that $\hat\bfy_t$ is the implicitly biased version of $\tilde \bfy_t$. Although $\tilde\bfy_t$ appears to be a biased estimator of $\bfy_t$, the next lemma shows that it can effectively compare the losses between different paths. 

\begin{lemma}\label{lm:est-expt}
    For any path with representation $\bfx \in \calX$, it holds that $\bbE_t[\langle \bfx, \tilde \bfy_t \rangle] = \langle \bfx, \bfy_t \rangle + \|\bfx\|_1-2$.
\end{lemma}

\begin{proof}
For some vertex $v \in V$, we denote by $\calE_{t,v}$ the event that vertex $v$ is chosen in the path in round $t$, i,e, $\ind[\bfx_t[v] = 1]$. Under event $\calE_{t,v}$, chosen path $\bfx_t$ can be divided into two subpath: one from $\source$ to $v$, and other from $v$ to $\sink$. Let $\ell^-_{t, v}$ and $\ell^+_{t, v}$ be the total weight of the path from $\source$ to $v$ and the path from $v$ to $\sink$, respectively. According to the linearity of expectation, it satisfies that
\begin{align} \label{eq:lm:est-expt-1}
    \bbE_t[\ell_t \mid \calE_{t,v}] = \bbE_t[\ell_{t,v}^- \mid  \calE_{t,v}] + \bbE_t[\ell_{t,v}^+ \mid \calE_{t,v}].
\end{align}

Note $\ell^-_{t, \source}=\ell^+_{t, \sink}=0$. For some edge $e = (v_-, v_+) \in E$, we similarly define by $\calE_{t,e}$ the event that $\bfx_t[e]=1$. The total weight of the path can also be decomposed into
\begin{align} \label{eq:lm:est-expt-2}
    \bbE_t[\ell_t \mid  \calE_{t,e}] = \bbE_t[\ell_{t,v_-}^- \mid  \calE_{t,e}] + \bfy_t[e] + \bbE_t[\ell_{t,v_+}^+ \mid  \calE_{t,e}].
\end{align}

Observe that our edge sampling procedure is Markovian. That is, under the event $\mathcal{E}_{t,u}$, the probability of choosing an outgoing edge from node $u$ does not depend on the path from $\source$ to $u$. This implies that:
\begin{align} \label{eq:lm:est-expt-3}
    \bbE_t[\ell_{t,v_-}^- \mid  \calE_{t,v_-}] = \bbE_t[\ell_{t,v_-}^- \mid  \calE_{t,v_-} \cap \calE_{t,e}] = \bbE_t[\ell_{t,v_-}^- \mid  \calE_{t,e}]
\end{align}
where the last inequality is given by $\calE_{t,e} \subseteq \calE_{t,v_-}$.
Similarly, we have that 
\begin{align} \label{eq:lm:est-expt-4}
    \bbE_t[\ell_{t,v_+}^+ \mid  \calE_{t,v_+}] = \bbE_t[\ell_{t,v_+}^+ \mid  \calE_{t,e}]
\end{align}

Let $P = (v_0, e_1, v_1, \dots, e_k, v_k)$ be the path that corresponds to the vector $\bfx\in\calX$, where $v_0 = \source$ and $v_k = \sink$. The expectation of the inner product $\langle \bfx, \tilde \bfy_t\rangle$ can be computed as follows:
\begin{align*}
    \bbE_t[\langle \bfx, \tilde \bfy_t \rangle] &= \sum_{i=0}^k \bbE_t\left[\tilde\bfy_t[v_i]\right] + \sum_{i=1}^k \bbE_t\left[\tilde\bfy_t[e_i]\right] \\
    &= \sum_{i=1}^{k-1} \bbE_t\left[\frac{(1 - \ell_t)\ind[\bfx_t[v_i] = 1]}{\bbP_t[\bfx_t[v_i] = 1]} \right] + \sum_{i=1}^k \bbE_t\left[\frac{(1 + \ell_t)\ind[\bfx_t[e_i] = 1]}{\bbP_t[\bfx_t[e_i] = 1]} \right] \\
    &= \sum_{i=1}^{k-1} \bbE_t[1 - \ell_t \mid  \calE_{t,v_i}] + \sum_{i=1}^k \bbE_t[1 + \ell_t \mid  \calE_{t,e_i}] \\
    &= 2k-1 - \sum_{i=1}^{k-1} \Big(\bbE_t[\ell_{t,v_i}^- \mid  \calE_{t,v_i}] + \bbE_t[\ell_{t,v_i}^+ \mid  \calE_{t,v_i}]\Big)  \\
        &\qquad\qquad + \sum_{i=1}^k \Big(\bbE_t[\ell_{t,v_{i-1}}^- \mid  \calE_{t,e_i}] + \bfy_t[e_i] + \bbE_t[\ell_{t,v_i}^+ \mid  \calE_{t,e_i}]\Big) \\
    &= 2k-1 - \sum_{i=1}^{k-1} \bbE_t[\ell_{t,v_i}^- \mid  \calE_{t,v_i}] - \sum_{i=1}^{k-1} \bbE_t[\ell_{t,v_i}^+ \mid  \calE_{t,v_i}]\\
        &\qquad\qquad +\sum_{i=1}^{k-1} \bbE_t[\ell_{t,v_{i}}^- \mid  \calE_{t,v_{i}}] + \sum_{i=1}^{k} \bfy_t[e_i] + \sum_{i=1}^{k-1} \bbE_t[\ell_{t,v_i}^+ \mid  \calE_{t,v_i}] \\
    &= \langle \bfx, \bfy_t \rangle+\|\bfx\|_1 -2.
\end{align*}
where the second equality follows from the definition of $\hat\bfy_t$, the third equality follows from the definition of the events $\calE_{t,v}$ and $\calE_{t,e}$, the fourth equality follows from equations \eqref{eq:lm:est-expt-1} and \eqref{eq:lm:est-expt-2}, and the fifth equality follows from equations \eqref{eq:lm:est-expt-3}, \eqref{eq:lm:est-expt-4}, and $\ell^-_{t, \source}=\ell^+_{t, \sink}=0$.

\end{proof}

If all paths have the same length $K$, then $\mathbb{E}_t[\langle \bfx - \bfx', \tilde{\bfy}_t \rangle] = \langle \bfx - \bfx', \bfy_t \rangle$ for all $\bfx, \bfx' \in \calX$. This equality is crucial for developing a framework in \Cref{appendix:ftrl-framework}, which enables implicit exploration—similar to the framework for standard multi-armed bandits by \citet{neu2015explore}—within an FTRL problem such as the one formulated in this section. The equality ensures the framework's correct application.  Using $F(\bfx) = -\sum_{v \in V} \sqrt{\bfx[v]} - \sum_{e \in E} \sqrt{\bfx[e]}$ as our regularizer, we can apply this framework to achieve a regret bound of at most $\mathcal{O}(\sqrt{K |E| T \log(|E|/\delta)})$ with probability at least $1 - \delta$ against any adaptive adversary. We refer the reader to \Cref{appendix:alg-equal} for the omitted details.

\subsection{Relaxing the Equal Path Length Assumption}\label{sec:alg-make-equal}
In this section, we relax the assumption that all paths have the same length. Denote $K(v)$ as the length of the longest path from $\source$ to vertex $v$. Let $K = K(\sink)$ denote the length of the longest path within the DAG. Note that $K(\source)=0$.

Construct the augmented vector as $\bfx^\dagger := \bfx \circ b(\bfx)$, where $b(\bfx) \in \{0, 1\}^{K-1}$ is given by
\[
b(\bfx)[i] := \ind\big[\exists (u, v) \in E, \bfx[(u, v)] = 1, K(u) < i < K(v)\big].
\]
We denote $\calX^\dagger := \{\bfx^\dagger \mid \bfx \in \calX\}$ as the augmented decision space. Let $\hat{\boldsymbol{\gamma}} \in \mathbb{R}_{>0}^{K-1}$ be a positive-valued vector.
Correspondingly, we construct the augmented loss estimator $\hat\bfy_t^\dagger := \hat\bfy_t \circ \hat \bfc_t$, where $\hat\bfc_t\in \mathbb{R}_{\geq 0}^{K-1}$ is defined as:
\[
\hat \bfc_t[i] := \frac{2 \cdot \ind[b(\bfx_t)[i] = 1]}{\bbP_t[b(\bfx_t)[i] = 1]+\hat{\boldsymbol{\gamma}}[i]},
\]
and $\bfx_t$ is the path chosen according to the selection procedure from the previous section.
We also define $\tilde \bfc_t\in \mathbb{R}_{\geq 0}^{K-1}$ as:
\[
 \tilde \bfc_t[i] := \frac{2 \cdot \ind[b(\bfx_t)[i] = 1]}{\bbP_t[b(\bfx_t)[i] = 1]},
\]
Observe that $\hat \bfc_t[i]$ is the implicitly biased version of $\tilde \bfc_t[i]$. The next lemma establishes a key property of the loss estimator $\tilde\bfy_t^\dagger := \tilde\bfy_t \circ \tilde \bfc_t$, in which the implicit biasing is absent.

\begin{lemma}\label{lm:est-expt-aug}
    For any path with representation $\bfx \in \calX$, it holds that $\bbE_t[\langle \bfx^\dagger, \tilde \bfy_t ^\dagger\rangle] = \langle \bfx, \bfy_t \rangle + 2K - 1$.
\end{lemma}

\begin{proof}
Consider any $\bfx \in \calX$, and its corresponding path $P = (v_0, e_1, v_1, \dots, e_k, v_k)$, where $v_0 = \source$ and $v_k = \sink$. The expectation of the inner product of the auxiliary bits, $\langle b(\bfx), \tilde \bfc_t \rangle$, satisfies  
\[
\bbE_t[\langle b(\bfx), \tilde \bfc_t \rangle] = \sum_{i=1}^{K-1} b(\bfx)[i] \cdot \bbE_t[\tilde \bfc_t[i]].
\]
By construction, $\bbE_t[\tilde \bfc_t[i]] = 2$.  
Since $K(v_{j-1}) < K(v_{j})$ for all $j \in \llbracket k\rrbracket$, for any given index $i \in \llbracket K-1\rrbracket$, there is at most one index $j_i \in \llbracket k\rrbracket$ such that $K(v_{j_i-1}) < i < K(v_{j_i})$. Consequently,
    \begin{align*} 
        \sum_{i=1}^{K-1} b(\bfx)[i] &= \sum_{i=1}^{K-1} \sum_{j=1}^{k} \ind[K(v_{j-1}) < i < K(v_j)] \\
        &= \sum_{j=1}^{k}\sum_{i=1}^{K-1} \ind[K(v_{j-1}) < i < K(v_j)] \\
        &= \sum_{j=1}^{k}\big(K(v_j) - K(v_{j-1}) - 1\big) = K(\sink) - K(\source) - k.
    \end{align*}
    Hence, using the fact that $\|\bfx\|_1 = 2k+1$, which follows from the mapping between $\bfx$ and $P$,  
    \begin{align}
        \bbE_t[\langle b(\bfx), \tilde \bfc_t\rangle] = 2(K(\sink) - K(\source) - k) = 2K - \|x\|_1 + 1.\label{eq:lm:est-expt-aug}
    \end{align}
    As a result,
    \begin{align*}
        \bbE_t[\langle \bfx^\dagger, \tilde \bfy_t ^\dagger\rangle] = \bbE_t[\langle \bfx, \tilde \bfy_t\rangle + \langle b(\bfx), \tilde \bfc_t\rangle]
        &= \langle \bfx, \bfy_t \rangle + \|x\|_1 -2 + 2K - \|x\|_1 + 1 \\
        &= \langle \bfx, \bfy_t \rangle + 2K - 1,
    \end{align*}
    where the second equality follows from Equation \eqref{eq:lm:est-expt-aug} and Lemma~\ref{lm:est-expt-aug}. %
\end{proof}

We can appropriately modify our FTRL algorithm from the previous section to work with the augmented decision space $\calX^\dagger := \{\bfx^\dagger \mid \bfx \in \calX\}$, augmented loss estimator $\hat{\bfy}_t^\dagger$ and augmented regularizer $F(\bfx)=-\sum_{v\in V}\sqrt{\bfx[v]}-\sum_{e\in E}\sqrt{\bfx[e]}-\sum_{i\in \llbracket K-1\rrbracket}\sqrt{\bfx[i]}$ for any $\bfx\in [0,1]^{V\cup E\cup\llbracket K-1\rrbracket}$. Thus, we can apply our FTRL framework for implicitly biased estimators from \Cref{appendix:ftrl-framework} to obtain a regret bound of at most $\mathcal{O}(\sqrt{K |E| T\log(|E|/\delta)})$ with probability at least $1 - \delta$ against any adaptive adversary. Furthermore, we assert that our FTRL approach can be implemented efficiently, as it can be easily shown that the set $\operatorname{co}(\mathcal{X}^\dagger)$ can be represented using a polynomial number of linear constraints. We refer the reader to \Cref{appendix:alg-make-equal} for the omitted details of this section.

\subsection{Achieving a Regret Upper Bound of $\tilde{\mathcal{O}}(\sqrt{|E|T \log |\mathcal{X}|})$}\label{sec:alg-centroid}

\begin{figure}
    \centering
    \centering\begin{tikzpicture}{object/.style={thin,double,<->}}
            \node[nodelbl] (A) at (2/2,2) {\tiny{$A$}};
            \node[nodelbl] (B) at (6/2+0.6,2) {\tiny{$B$}};
            \node[nodelbl] (C) at (1/2,1) {\tiny{$C$}};
            \node[nodelbl] (D) at (3/2,1) {\tiny{$D$}};
            \node[nodelbl] (E) at (5/2+0.6,1) {\tiny{$E$}};
            \node[nodelbl] (F) at (7/2+0.6,1) {\tiny{$F$}};
            \node[nodelbl] (G) at (2/2,0) {\tiny{$G$}};
            \node[nodelbl] (H) at (6/2+0.6,0) {\tiny{$H$}};

            \foreach \edge in {
                (A) -- (B),
                (A) -- (C),
                (C) -- (D),
                (D) -- (E),
                (D) -- (G),
                (E) -- (F),
                (F) -- (H),
            } {
                \StandardPath \edge;
            }

            \foreach \edge in {
                (A) -- (D),
                (B) -- (E),
                (B) -- (F),
                (C) -- (G),
                (E) -- (H),
                (G) -- (H),
            } {
                \StandardPath \edge;
            }

            \node at (2.3, -1) {$G$};
        \end{tikzpicture}
        \raisebox{2.2cm}{$\qquad\Rightarrow\qquad$}
        \begin{tikzpicture}
            \MakeTriplet A {2,2}
            \MakeTriplet B {6,2}
            
            \MakeTriplet C {1,1}
            \MakeTriplet D {3,1}
            \MakeTriplet E {5,1}
            \MakeTriplet F {7,1}
    
            \MakeTriplet G {2,0}
            \MakeTriplet H {6,0}
    
            \foreach \edge in {
                (A2) to[bend left=25] (B3),
                (A2) -- (C3),
                (E1) to[bend right=25] (F2),
                (F2) -- (H3),
                (A1) -- (D2),
                (C1) to[bend right=30] (D2),
                (D2) -- (G3),
                (D2) to[bend left=30] (E3),
                (D2) to[bend right=25] (F3),
                (D2) to[bend right=50] (H3),
            } {
                \StandardPath \edge;
            }
    
            \foreach \edge in {
                (A3) -- (D1),
                (C3) -- (G1),
                (E3) -- (H1),
                (G3) to[bend right=33] (H1),
                (B3) -- (E1),
                (B3) -- (F1),
            } {
                \StandardPath \edge;
            }  
            
            \node at (4, -1) {$G^\dagger$};
        \end{tikzpicture}  
    \caption{
        Example $G$ and $G^\dagger$ according to conversion in \Cref{sec:alg-centroid}. The longest path from source to sink in $G^\dagger$ is upper bounded by $\calO(\log |\calX|)$. See \Cref{fig:enter-label} in Appendix for more details.}\label{main:figure1}
\end{figure}
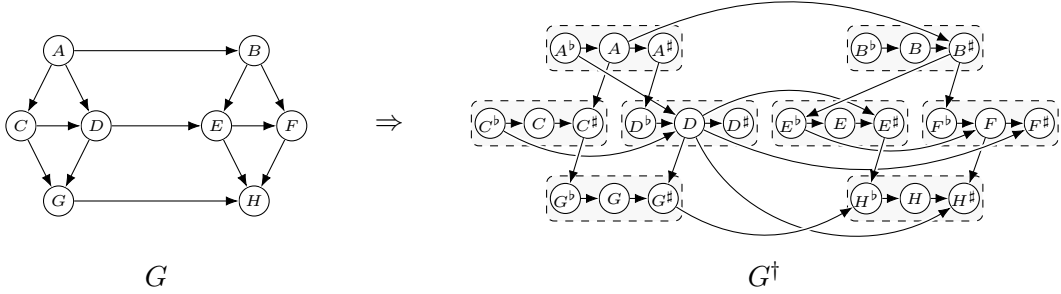

In this section, we transform the input DAG $G$ into a new DAG $G^\dagger$ with an equivalent decision space but reduced complexity. The core idea is to introduce ``express'' edges that compress long paths in $G$, ensuring the longest path in $G^\dagger$ is bounded by $\mathcal{O}(\log|\mathcal{X}|)$ while minimally increasing the number of edges and vertices. Consider a long path $P = (v_0, e_1, \dots, e_k, v_k)$ in $G$. We concisely represent all subpaths of $P$ with the help of the middle vertex $v_{\lfloor k/2\rfloor}$. For each $i < \lfloor k/2 \rfloor$, we add an edge $(v_i, v_{\lfloor k/2\rfloor})$ to represent the subpath from $v_i$ to $v_{\lfloor k/2\rfloor}$. Similarly, for each $j > \lfloor k/2 \rfloor$, we add an edge $(v_{\lfloor k/2\rfloor}, v_j)$. Thus, any subpath from $v_i$ to $v_j$ (where $i < \lfloor k/2\rfloor < j$) can be represented with just two edges, $(v_i, v_{\lfloor k/2\rfloor})$ and $(v_{\lfloor k/2\rfloor}, v_j)$. Recursively applying this method creates a hierarchical structure where every subpath of $P$ requires only $\mathcal{O}(1)$ edges.
Extending this concept using centroid-based decomposition for the spanning tree ensures that the longest path in $G^\dagger$ remains bounded by $\mathcal{O}(\log |\mathcal{X}|)$, with only a logarithmic increase in edges and vertices. Consequently, the online shortest path problem in $G$ reduces to $G^\dagger$, allowing our algorithm from \Cref{sec:alg-make-equal} to achieve a high-probability regret bound of $\tilde{\mathcal{O}}(\sqrt{|E|T\log |\calX|})$ against any adaptive adversary. Further details, including omitted proofs, are provided in \Cref{appendix:alg-centroid}.

We formally begin our transformation. Let $C: V \to \bbN$ denote the number of distinct paths from the source $\source$ to any vertex $v$. It holds that $C(\source) = 1$ and $C(v) := \sum_{(u, v) \in \delta^{-}(v)} C(u)$ for any $v \neq \source$. According to the definition, it satisfies that $C(\sink) = |\mathcal{X}|$. Let $h(v) := \argmax_{(u, v) \in \delta^{-}(v)} C(u)$ be the incoming edge that brings the maximum number of paths to vertex $v$, with ties broken arbitrarily. Let $E^\clubsuit := \{h(v) \mid v \in V \setminus \{\source\}\}$ be the set of all such edges. The underlying subgraph $S := (V, E^\clubsuit)$ forms a directed spanning tree of $G$. It can be easily shown that the number of non-tree edges (edges not in $E^\clubsuit$) on any path from $\source$ to $\sink$ in $G$ is at most $\log |\mathcal{X}|$.

We now introduce the \textit{centroid-based decomposition}: Given a directed tree $S = (V, E^\clubsuit)$, we identify a vertex $c \in V$ such that the connected components $\hat{S}_1, \dots, \hat{S}_k$ resulting from its removal satisfy $|\hat{V}_i| \leq |V|/2$ for all $i \in \llbracket k \rrbracket$, where $\hat{V}_i$ is the set of vertices in the subtree $\hat{S}_i$. Such a vertex $c$, known as the \textit{centroid}, always exists in any tree \cite{jordan1869assemblages,della2019new}. We associate the centroid $c$ with the tree $S$ by defining $S_c:= S$. The above procedure is then applied recursively to each component $\hat{S}_i$ for $i \in \llbracket k \rrbracket$. If a component reduces to a single vertex $c$, we designate $c$ as its centroid and terminate the recursion.  

Since the sets $\hat{V}_i$ resulting from the removal of $c$ form a partition of $V \setminus \{c\}$, each vertex $v \in V$ will eventually be assigned as the centroid of some subtree $S_v = (V_v, E^\clubsuit_v)$. Consequently, this procedure generates a collection of subtrees $\mathcal{T} := \{S_v : v \in V\}$, where each vertex $v$ is uniquely associated with a subtree of $S$ in which it serves as the centroid.  Furthermore, we define $\mathcal{T}(S_v) := \{S_w : w \in V_v\}$ as the centroid-based decomposition of the subtree $S_v$.  

We now state the construction for a new graph $G^\dagger = (V^\dagger, E^\dagger)$ using $\calT$ as follows:
\begin{enumerate}[noitemsep]
    \item Initialize $V^\dagger \gets \emptyset$ and $E^\dagger\gets\emptyset$.
    \item For each vertex $c\in V$:
    \begin{enumerate}[nosep]
        \item $V^\dagger \leftarrow V^\dagger \cup \{c^\flat, c, c^\sharp\}$.
        \item For each vertex $v \in V_c$:
            \begin{enumerate}[nosep]
                \item If there is a directed path from $v$ to $c$ in $S_c$, or if $v$ is $c$, update $E^\dagger \leftarrow E^\dagger \cup \{(v^\flat, c)\}$.
                \item If there is a directed path from $c$ to $v$ in $S_c$, or if $v$ is $c$, update $E^\dagger \leftarrow E^\dagger \cup \{(c, v^\sharp)\}$.
    \end{enumerate}
    \end{enumerate}
    \item For each non-tree edge $(u, v) \in E \setminus E^\clubsuit$, update $E^\dagger \leftarrow E^\dagger \cup \{(u^\sharp, v^\flat)\}$.
\end{enumerate}

We refer to \Cref{main:figure1} for one example of our conversion. It is easy to verify that the graph $G^\dagger$ is a Directed Acyclic Graph with source node $\source^\flat$ and sink node $\sink^\sharp$. We now demonstrate that the converted graph $G^\dagger$ is essentially equivalent to $G$. We define a mapping $\sigma: E^\dagger \to 2^E$ as follows:
\begin{itemize}[noitemsep]
    \item For $e^\dagger = (v^\flat, c)$, $\sigma(e^\dagger)$ consists of all edges on the unique path from $v$ to $c$ in the tree $S$.
    \item For $e^\dagger = (c, v^\sharp)$, $\sigma(e^\dagger)$ consists of all edges on the unique path from $c$ to $v$ in the tree $S$.
    \item For $e^\dagger = (u^\sharp, v^\flat)$, $\sigma(e^\dagger) = \{(u, v)\} \subseteq E \setminus E^\clubsuit$ contains the corresponding edge.
\end{itemize}
The above mapping assigns each edge $e^\dagger = (u^\dagger, v^\dagger) \in E^\dagger$ a path from $u$ to $v$ (which may be empty), as specified by $\sigma(e^\dagger)$, where $w^\dagger \in \{w^\flat, w, w^\sharp\}$ for $w \in \{u, v\}$. 
Denote by $\calP^\dagger$ the set of paths from $\source^\flat$ to $\sink^\sharp$ in $G^\dagger$.
The following lemma establishes an important property of $\sigma(e^\dagger)$.
\begin{Lemma}{lem:me1capme2-main}
    For any path $P^\dagger \in \calP^\dagger$, $\sigma(e_1^\dagger) \cap \sigma(e_2^\dagger) = \emptyset$ for any distinct edges $e_1^\dagger, e_2^\dagger \in P^\dagger$.
\end{Lemma}

The next lemma establishes that this mapping defines a bijection between the paths from $\source$ to $\sink$ in $G$ and the paths $P^\dagger$ from $\source^\flat$ to $\sink^\sharp$ in $G^\dagger$. We slightly abuse notation for $\sigma$.  

\begin{Lemma}{gnew:paths-main}  
There exists an efficiently computable bijection $\sigma: \calP^\dagger \to \calP$ such that an edge $e \in E$ belongs to $\sigma(P^\dagger)$ if and only if there exists an edge $e^\dagger \in P^\dagger$ with $e \in \sigma(e^\dagger)$.  
\end{Lemma}

 Let $w: E \to \bbR$ be a weight function in the graph $G = (V, E)$. Define $w^\dagger: E^\dagger \to \bbR$ as the weight function for the converted graph $G^\dagger = (V^\dagger, E^\dagger)$:
\begin{align}\label{eq:weight-conversion-main}
    w^\dagger(e^\dagger) := \ind[|\sigma(e^\dagger)|\geq 1]\cdot\sum_{e \in \sigma(e^\dagger)} w(e).
\end{align}
Using this mapping, we can convert a decision problem on $G$ to a decision problem on $G^\dagger$ as follows.
\begin{Lemma}{gnew:conversion-main}
    The online shortest path problem on $G = (V, E)$ can be efficiently reduced to the online shortest path problem on $G^\dagger = (V^\dagger, E^\dagger)$.
\end{Lemma}
\begin{proof}
    First, we can efficiently construct the DAG $G^\dagger$ using the DAG $G$. Next, given any weight function $w_t$ encoded as $\bfy_t \in \bbR^{V \cup E}$ in the graph $G$, we can convert it into a weight function $\bfy_t^\dagger \in \bbR^{V^\dagger \cup E^\dagger}$ corresponding to $w_t^\dagger$ for $G^\dagger$ according to \eqref{eq:weight-conversion-main}. For any chosen path $P^\dagger_t$ in $G^\dagger$ (encoded as $\bfx^\dagger_t$) , we can efficiently choose $\bfx_t \in \calX$ corresponding to the path $\sigma(P^\dagger)$ in $G$ following the bijective mapping $g$ in Lemma \ref{gnew:paths-main}. Due to Lemma \ref{lem:me1capme2-main} and Lemma \ref{gnew:paths-main}, we have:
    \[
        \langle \bfx^\dagger, \bfy_t^\dagger \rangle = \sum_{e^\dagger \in P^\dagger} \ind[|\sigma(e^\dagger)|\geq 1]\cdot\sum_{e \in \sigma(e^\dagger)} w(e) = \sum_{e \in P} w(e) = \langle \bfx, \bfy_t \rangle.
    \]
    The second equality follows from the fact that the set of edges in $\sigma(P^\dagger)$ is given by $\bigcup_{e^\dagger \in P^\dagger} \sigma(e^\dagger)$, and that $\sigma(e_1^\dagger) \cap \sigma(e_2^\dagger) = \emptyset$ for any distinct edges $e_1^\dagger, e_2^\dagger \in P^\dagger$. Hence, we can efficiently reduce the online shortest path problem on $G$ to the online shortest path problem on $G^\dagger$.
\end{proof}

Finally, the graph $G^\dagger$ satisfies the required size constraints, as stated in the following lemma.
\begin{Lemma}{gnew:size-main}
    The graph $G^\dagger = (V^\dagger, E^\dagger)$ contains $|V^\dagger| \leq \calO(|V|)$ vertices and $|E^\dagger| \leq \calO(|V| \log |V| + |E|)$ edges. Moreover, The number of edges on the longest path from $\source^\flat$ to $\sink^\sharp$ is upper bounded by $\calO(\log |\calX|)$.
\end{Lemma}

By combining Lemmas~\ref{gnew:conversion-main} and \ref{gnew:size-main}, and applying our FTRL algorithm from the previous section on the DAG $G^\dagger$, we establish the main theorem:
\begin{Theorem}{actual:thm}
    There exists an computationally efficient algorithm that incurs a regret bound of at most $\tilde \calO(\sqrt{|E|T\log(|\calX|/\delta)})$ with probability at least $1-\delta$ against any adaptive adversary, where $\tilde \calO(\cdot)$ only hides logarithmic factors in $|E|$.
\end{Theorem}

We refer the reader to \Cref{appendix:alg-centroid} for the omitted details of this section. Moreover, our algorithm is nearly minimax-optimal as for the class of DAGs with at most $d$ edges and at most $N$ paths, we establish a minimax lower bound of $\Omega(\sqrt{dT\log(N)/\log(d)})$ in \Cref{appendix:minimax-lower}.

%% file: applications.tex
\section{Applications}\label{sec:applications-main}

\begin{table}[t]
    \centering
    \begin{tabular}{p{6cm}p{5cm}p{4.2cm}}
        \toprule
        \textbf{Combinatorial set} & \textbf{Best known regret}$^\ddagger$ & \textbf{Our improved regret} \\
        
        \midrule
        Hypercube & $d^2\sqrt{T}$ & $d\sqrt{T}$ \\
        Multi-task MAB & $(\sum_{i=1}^md_i)^2\sqrt{T}$ & $\sum_{i=1}^m\sqrt{d_iT}$ \\
        $m$-sets & $d^2\sqrt{T}$ & $\sqrt{md^2T}$ \\   
        Shortest walk & $|E|^2\sqrt{T}$ & $\sqrt{K^2|E|T}$ \\
        Extensive-form games & $|\calZ|^2\sqrt{T}$ & $\sqrt{|\calZ|T\log(N)}$ \\
        Colonel Blotto game & $K^2N^2\sqrt{T}$ & $\sqrt{K^3NT}$ \\
        \bottomrule
    \end{tabular}
    \caption{Summary of high-probability regret guarantees for efficient algorithms across various combinatorial sets, ignoring constants and logarithmic factors. $^\ddagger$The best-known high-probability regret guarantee for efficient algorithms was formally proven only for continuous sets by \citet{zimmert2022return}. However, we believe their analysis extends to discrete decision sets, such as the combinatorial sets considered, using the same techniques as \citet{abernethy2008competing}.}\label{table2}
\end{table}

While it might not be apparent at first glance, learning in several structured domains $\calX \subseteq \{0,1\}^d$ can be efficiently reduced to online shortest paths in suitably-defined DAGs.\footnote{To our knowledge, we are the first to point out this fact in the case of $m$-sets and, more importantly, extensive-form games, for which the reduction is not immediate.} These include at least the following examples.

\begin{description}[nosep,font=\normalfont\itshape\textbullet~~]
\item[Hypercube:] $\calX := \{0,1\}^d$.  

\item[Multi-task MAB:] $\calX := \calX_1 \times \calX_2 \times \dots \times \calX_m$, where $\calX_i = \{e_1, \dots, e_{d_i}\}$ is a set of unit vectors.

\item[$\mathbf{m}$-sets:] $\calX := \{\bfx \in \{0,1\}^d : \|\bfx\|_1 = m\}$.  

\item[Shortest walk in directed graph.] We consider the online shortest walk problem in a directed graph $G = (V,E)$, where walks can have a length of at most $K\leq |E|$.

\item[Extensive-form games.] The game consists of decision nodes $\calX$, observation nodes $\calY$, and terminal nodes $\calZ$. We choose one of the $N \leq 2^{|\calZ|}$ possible strategies at the decision nodes and aim to minimize the total loss incurred.  

\item[Colonel Blotto games.] In this game, the goal is to assign $N$ soldiers across $K$ battlefields while minimizing the total loss incurred.  
\end{description}
We refer the reader to \Cref{appendix:applications} for a detailed discussion of each setting and its DAG reduction.
The important point is that in light of the connection to DAGs, our method applies directly to the above settings as well.
We summarize the results we obtain for these settings in \Cref{table2}, comparing the regret guarantees enjoyed by our method compared to the prior known high-probability regret guarantees achieved by efficient algorithms. %
We remark that our high-probability regret bound matches that of EXP3 with Kiefer-Wolfowitz exploration for the Hypercube and Extensive-form games. For Multi-task MAB, our high-probability regret bound significantly improves upon that of EXP3 with Kiefer-Wolfowitz exploration, and we also establish a matching lower bound, up to logarithmic factors. More details on previous approaches and implementation details of our methods in these settings are available in \Cref{appendix:applications}.

%% file: conclusion.tex
\section{Conclusion and Future Work}
In this paper, we studied the online shortest path problem on DAGs. We designed the first computationally efficient algorithm to achieve a high-probability nearly minimax-optimal regret bound of $\tilde{\mathcal{O}}(\sqrt{|E|T\log |\mathcal{X}|})$ against any adaptive adversary, where $\tilde{\mathcal{O}}(\cdot)$ hides logarithmic factors in $|E|$. Beyond shortest paths, our algorithm can be applied to various combinatorial sets in $\{0,1\}^d$, and we provided improved high-probability regret bounds for them.  

Our work raises several interesting open questions in combinatorial bandits. First, can our approach be further generalized to achieve high-probability regret bounds for any combinatorial set in $\{0,1\}^d$? Second, is there an efficient algorithm that achieves a high-probability minimax-optimal regret bound of $\mathcal{O}(\sqrt{dT\log |\calX|})$ for any combinatorial set $\calX \subseteq \{0,1\}^d$? Finally, given a fixed combinatorial set $\calX \subseteq \{0,1\}^d$, what are the tight upper and lower bounds on regret relative to $\calX$?

%% file: ack.tex
\section*{Acknowledgments}

The authors are grateful to Haipeng Luo for helpful discussion regarding the prior work \citep{lee2020bias}.

This work was supported in part by NSF TRIPODS CCF Award \#2023166, a Northrop Grumman University Research Award, ONR YIP award \# N00014-20-1-2571, NSF award \#1844729, and NSF award \# CCF-2443068.

%% file: appendix.tex
\section{Related Works}\label{appendix:related-works}
\paragraph{Multi-Armed Bandits.} For non-stochastic Multi-Armed Bandits, \citet{auer2002nonstochastic} introduced the EXP3 algorithm (short for ``exponential-weight algorithm for exploration and exploitation''), which achieves a pseudo-regret of $\mathcal{O}(\sqrt{KT\log K})$. They also established a regret lower bound of $\Omega(\sqrt{KT})$. Subsequently, \citet{audibert2010regret} introduced the implicitly normalized forecaster, achieving a pseudo-regret bound of $\mathcal{O}(\sqrt{KT})$. Building on this, \citet{bubeck2012best} initiated the study of ``best of both worlds'' algorithms, which attain near-optimal pseudo-regret bounds in both stochastic and non-stochastic settings. Finally, \citet{zimmert2021tsallis} demonstrated that Tsallis-1/2-INF achieves optimal pseudo-regret bounds for the best of both worlds problem.

\citet{auer2002nonstochastic} also introduced a variant of EXP3, called EXP3.P, which incorporates explicit exploration and achieves a regret of $\mathcal{O}(\sqrt{KT\log(KT/\delta)})$ with probability at least $1 - \delta$. \citet{bubeck2012regret} later analyzed a version of EXP3.P that attains a regret of $5.15\sqrt{KT\log(K/\delta)}$ with the same probability guarantee. Building on this, \citet{neu2015explore} proposed EXP3-IX (EXP3 with Implicit Exploration), which leverages implicit exploration to achieve a regret of $2\sqrt{2KT\log(K/\delta)}$ with probability at least $1 - \delta$.

\paragraph{Adversarial Linear Bandits.} For a bounded arm set $\mathcal{X} \subset \mathbb{R}^d$ and loss values in $[-1,1]$ for any arm, \citet{mcmahan2004online} were the first to design a sublinear regret algorithm, achieving an expected regret of $T^{3/4}$. A later, improved analysis by \citet{dani2006robbing} reduced this bound to $T^{2/3}$, while maintaining polynomial dependence on $d$. This result holds even against an adaptive adversary.  

For the special case of DAGs, \citet{awerbuch2004adaptive} designed the first algorithm with a pseudo-regret of $T^{2/3}$ and polynomial dependence on $d$. Subsequently, \citet{gyorgy2007line} extended this result by developing an algorithm that achieves a high-probability regret bound of $T^{2/3}$, also with polynomial dependence on $d$, even against an adaptive adversary. 

\citet{dani2007price} were the first to design an algorithm called Geometric Hedge, which achieves a regret of $T^{1/2}$ with polynomial dependence on $d$. Later, \citet{bartlett2008high} introduced a variant of Geometric Hedge that incurs a high-probability regret bound of $\tilde{\mathcal{O}}(d^{3/2}\sqrt{T})$.  

\citet{cesa2012combinatorial} followed up by designing an algorithm called Comband, which achieves a pseudo-regret of $\mathcal{O}(\sqrt{dT\log |\mathcal{X}|})$ for various combinatorial sets $\mathcal{X} \subseteq \{0,1\}^d$. Subsequently, \citet{bubeck2012towards} showed that EXP2 with John's exploration incurs a pseudo-regret of $\mathcal{O}(\sqrt{dT\log |\mathcal{X}|})$ for any finite set $\mathcal{X} \subseteq \mathbb{R}^d$. For a general regret analysis of a similar algorithm, EXP3 for Linear Bandits with any fixed exploration distribution, we refer the reader to \citet{lattimore2020bandit}. Later, \citet{zimmert2022return} designed a high-probability version called EXP3 with Kiefer-Wolfowitz exploration, which achieves a regret of $\mathcal{O}(\sqrt{dT\log(|\mathcal{X}|/\delta)})$ with probability at least $1-\delta$ for any finite set $\mathcal{X} \subseteq \mathbb{R}^d$.

For the combinatorial setting, where the loss of each individual coordinate is bounded between $-1$ and $1$, \citet{audibert2014regret} provided near-optimal worst-case upper bounds for both semi-bandit and bandit feedback. For combinatorial sets such as $m$-sets, DAGs, multi-task MAB, and maximum matching in bipartite graphs, \citet{cohen2017tight,ito2019improved} established tight worst-case lower bounds under bandit feedback. The techniques used in these works can be appropriately adapted to derive tight lower bounds for the standard bandit setting, where the loss value of any arm lies within $[-1,1]$.

\paragraph{Computationally Efficient Algorithms.} For compact convex sets $\mathcal{X} \subset \mathbb{R}^d$, \citet{abernethy2008competing} were the first to propose a computationally efficient algorithm that achieves a pseudo-regret of $\text{poly}(d) \cdot \sqrt{T}$. Their approach leveraged efficient self-concordant barriers. They also analyzed the online shortest path problem in DAGs, providing an efficient algorithm with a pseudo-regret of $\tilde{\mathcal{O}}(\sqrt{|E|^3T})$.  

\citet{cesa2012combinatorial} demonstrated computationally efficient implementations of ComBand for certain combinatorial sets. For general convex decision sets, \citet{hazan2016volumetric} designed a computationally efficient algorithm with $\tilde{\mathcal{O}}(d\sqrt{T})$ pseudo-regret using volumetric spanners. Their approach extends to the online shortest path problem in DAGs, where their efficient algorithm achieves a pseudo-regret of $\tilde{\mathcal{O}}(\sqrt{|E|T\log |\mathcal{X}|})$.  

Given access to an efficient linear optimization oracle, \citet{ito2020tight} proposed a computationally efficient algorithm based on continuous multiplicative weight updates, which achieves $\tilde{\mathcal{O}}(d\sqrt{T})$ pseudo-regret while also providing tight first- and second-order guarantees.

For compact convex sets $\mathcal{X} \subset \mathbb{R}^d$, \citet{lee2020bias} proposed the first efficient algorithm achieving a high-probability regret of $\text{poly}(d) \cdot \sqrt{T}$ against an adaptive adversary. Their approach leveraged an efficient self-concordant barrier and yielded a worst-case high-probability regret of $\tilde \calO(\sqrt{d^7T})$. Subsequently, \citet{zimmert2022return} developed an improved efficient algorithm with a regret bound of $\tilde{\mathcal{O}}(d^2\sqrt{T})$ with high probability against an adaptive adversary, which remains the best known result to date. Their method relied on the entropic barrier. Notably, both high-probability guarantees were formally established only for continuous decision sets. However, we believe their analysis extends to discrete decision sets, such as paths in a DAG, using the same techniques as \citet{abernethy2008competing}.

\section{Technical Lemmas}
\begin{lemma}[\citet{slivkins2019introduction}]
    Fix $\varepsilon \in (0,\frac{1}{4})$. Let $RC_{\varepsilon}$ denote a random coin with bias $\varepsilon$, i.e., a distribution over $\{0, 1\}$ with expectation $\frac{1}{2}+\varepsilon$. Then $\KL(RC_\varepsilon, RC_0) \leq 8\varepsilon^2$ and $\KL(RC_0, RC_\varepsilon) \leq 4\varepsilon^2$.
\end{lemma}

\begin{lemma}[Chain Rule]
        Let $f(x_1,x_2,\ldots,x_n)$ and $g(x_1,x_2,\ldots,x_n)$ be two joint PMFs for a tuple of random variables $(X_i)_{i\in[n]}$. Let the sample space be $\Omega= \{0,1\}^{n}$. Then we have the following:
        \begin{equation*}
            \KL(f,g)=\sum\limits_{\omega\in \Omega}f(\omega)\left(\KL(f(X_1),g(X_1))+\sum_{i=2}^n \KL(f(X_i|X_{-i}=\omega_{-i}),g(X_i|X_{-i}=\omega_{-i}))\right)\;
        \end{equation*}
        where $X_{-i}=(X_1,\ldots,X_{i-1})$, $\omega_{-i}=(\omega_1,\ldots,\omega_{i-1})$.
\end{lemma}
\begin{proof}\allowdisplaybreaks
    Let $\Omega^i=\{0,1\}^i$. Now we have the following:
    \begin{align*}
            \KL(f,g)&=\sum\limits_{\omega\in \Omega}f(\omega)\log\left(\frac{f(\omega)}{g(\omega)}\right)\;\\
            &=\sum\limits_{\omega\in \Omega}f(\omega)\log\left(\frac{f(\omega_1)\prod_{i=2}^nf(\omega_i|\omega_{-i})}{g(\omega_1)\prod_{i=2}^ng(\omega_i|\omega_{-i})}\right)\;\\
            &=\sum\limits_{\omega\in \Omega}f(\omega)\left(\log\left(\frac{f(\omega_1)}{g(\omega_1)}\right)+\sum_{i=2}^n\log\left(\frac{f(\omega_i|\omega_{-i})}{g(\omega_i|\omega_{-i})}\right)\right)\;\\
            &=\sum\limits_{\omega\in \Omega}f(\omega)\log\left(\frac{f(\omega_1)}{g(\omega_1)}\right)\;+\sum_{i=2}^n\sum\limits_{\omega\in \Omega}f(\omega)\log\left(\frac{f(\omega_i|\omega_{-i})}{g(\omega_i|\omega_{-i})}\right)\;\\
            &=\sum\limits_{\omega_1\in \mathbb{R}}f(\omega_1)\log\left(\frac{f(\omega_1)}{g(\omega_1)}\right)\;+\sum_{i=2}^n\sum\limits_{\omega\in \Omega^{i}}f(\omega)\log\left(\frac{f(\omega_i|\omega_{-i})}{g(\omega_i|\omega_{-i})}\right)\;\\
            &=\KL(f(X_1),g(X_1))+\sum_{i=2}^n\sum\limits_{\omega_{-i}\in \Omega^{i-1}}f(\omega_{-i})\sum\limits_{\omega_{i}\in \Omega^1}f(\omega_i)\log\left(\frac{f(\omega_i|\omega_{-i})}{g(\omega_i|\omega_{-i})}\right)\;\\
            &=\KL(f(X_1),g(X_1))+\sum_{i=2}^n\sum\limits_{\omega_{-i}\in \Omega^{i-1}}f(\omega_{-i})\KL(f(X_i|X_{-i}=\omega_{-i}),g(X_i|X_{-i}=\omega_{-i}))\;\\
            &=\sum_{\omega\in \Omega}f(\omega)\KL(f(X_1),g(X_1))\;+\sum_{i=2}^n\sum\limits_{\omega\in \Omega}f(\omega)\KL(f(X_i|X_{-i}=\omega_{-i}),g(X_i|X_{-i}=\omega_{-i}))\;\\
            &=\sum\limits_{\omega\in \Omega}f(\omega)\left(\KL(f(X_1),g(X_1))+\sum_{i=2}^n \KL(f(X_i|X_{-i}=\omega_{-i}),g(X_i|X_{-i}=\omega_{-i}))\right)\;
    \end{align*}
\end{proof}

\begin{lemma}[\cite{fiegel2023adapting}]\label{feigel-lem}
    Let $(u_t)_{t\in \llbracket T\rrbracket}$ be a random process adapted to the filtration $(\calF_t)_{t\in[T]}$ such that $0\leq u_t\leq H$ for all $t\in \llbracket T\rrbracket$. Then, with probability at least $1-\delta$, we have
    \begin{equation*}
        \sum_{t=1}^T[u_t-\mathbb{E}[u_t|\calF_{t-1}]]\leq H\sqrt{2T\log(1/\delta)}
    \end{equation*}
    Similarly, with probability at least $1-\delta$, we have
    \begin{equation*}
        \sum_{t=1}^T[\mathbb{E}[u_t|\calF_{t-1}]-u_t]\leq H\sqrt{2T\log(1/\delta)}
    \end{equation*}
\end{lemma}

\begin{corollary}[\cite{fiegel2023adapting}]\label{feigel-cor}
    Let $(u_t)_{t\in \llbracket T\rrbracket}$ be a random process adapted to the filtration $(\calF_t)_{t\in[T]}$ such that $-H\leq u_t\leq H$ for all $t\in \llbracket T\rrbracket$. Then, with probability at least $1-\delta$, we have
    \begin{equation*}
        \sum_{t=1}^T[u_t-\mathbb{E}[u_t|\calF_{t-1}]]\leq H\sqrt{8T\log(1/\delta)}
    \end{equation*}
    Similarly, with probability at least $1-\delta$, we have
    \begin{equation*}
        \sum_{t=1}^T[\mathbb{E}[u_t|\calF_{t-1}]-u_t]\leq H\sqrt{8T\log(1/\delta)}
    \end{equation*}
\end{corollary}
\begin{lemma}[\citet{lattimore2020bandit}]\label{technical-lem-hessian}
Let $\eta > 0$ and $f$ be Legendre and twice differentiable with positive definite Hessian in $A = \text{int}(\text{dom}(f))$. Then for all $x, y \in A$, there exists a $z \in [x, y] = \{(1 - \alpha)x + \alpha y : \alpha \in [0, 1]\}$ such that
\[
\langle x - y, u \rangle - \frac{D_f(x, y)}{\eta} \leq \frac{\eta}{2} \|u\|^2_{(\nabla^2 f(z))^{-1}}.
\]
where $D_f(x, y)$ with respect to $f$.
\end{lemma}
\section{General Framework for FTRL with Implicit Exploration}\label{appendix:ftrl-framework}
In this section, we extend the ideas from \citet{neu2015explore} to provide a general framework for using FTRL with implicit exploration in combinatorial bandits. In this setting, we are given a combinatorial set $\calX\subseteq\{0,1\}^d$. Define the $\ell_1$-norm of the set as $m:= \max_{\bfx \in \calX} \|\bfx\|_1$. 

In each round $t$, the algorithm selects $\bfx_t\in \calX$ and incurs a loss given by $\ell_t := \langle \bfx_t, \bfy_t \rangle\in [-1,1]$
where $\bfy_t$ is the loss vector chosen adaptively by an adversary based on the past filtration $\calF_{t-1}$ and the algorithm. The regret with respect to a fixed element $\bfx\in\calX$ is defined as  
\[
    R_T(\bfx) := \sum_{t=1}^T \langle \bfx_t, \bfy_t \rangle - \sum_{t=1}^T \langle \bfx, \bfy_t \rangle.
\]
The goal is to provide a high-probability regret guarantee on $\max_{\bfx\in\calX} R_T(\bfx).$ 

Let $\tilde{\bfy}_t \in \bbR^d$ be a \emph{relatively unbiased} estimator defined as  
\[
\tilde{\mathbf{y}}_t[i] := \frac{\ind[\bfx_t[i]=1] \cdot \ell_{t,i}}{\bbP_t[\bfx_t[i]=1]},
\]
where $\ell_{t,i}$ is some random variable based on $\ell_t$. Assume there is an absolute constant $b$ such that $\ell_{t,i} \in [0, b]$ for every $t \in \llbracket T \rrbracket$ and $i \in \llbracket d \rrbracket$. An estimator is \emph{relatively unbiased} when it satisfies  
\[
    \bbE_t[\langle \bfx - \bfx', \tilde{\bfy}_t \rangle] = \langle \bfx - \bfx', \bfy_t \rangle
\]
for any two elements $\bfx, \bfx' \in \calX$, preserving their differences.  

We analyze the FTRL algorithm, which follows the update rule  
\[
\tilde{\mathbf{x}}_{t} \leftarrow \arg\min_{\mathbf{x} \in \mathrm{co}(\mathcal{X})} \left( \eta \sum_{\tau=1}^{t-1} \langle \mathbf{x}, \widehat{\mathbf{y}}_\tau \rangle + F(\mathbf{x}) \right),
\]
where $F(\cdot)$ is a Legendre function such that $\nabla^{2} F(\cdot)$ is always a diagonal matrix with positive diagonal entries for any point on the chord $[\tilde x_{t},\tilde x_{t+1}]$. Moreover, $\hat\bfy_t \in \bbR^d$ is the loss estimator given by  
\[
    \widehat{\mathbf{y}}_t[i] := \frac{\ind[\bfx_t[i]=1] \cdot \ell_{t,i}}{\bbP_t[\bfx_t[i]=1] + \gamma_i}.
\]
The algorithm then samples $\bfx_t \in \calX$ such that  $\bbE_t[\bfx_t] = \tilde{\bfx}_t.$ Note that it holds $\bbP_t[\bfx_t[i] = 1] = \tilde{\bfx}_t[i]$ for every $i \in \llbracket d \rrbracket$.

Now, we begin our regret analysis.

\begin{lemma}
    Denote by $\calE_1$ the event that 
    \[
    \sum_{t=1}^T \langle \bfx_t - \tilde \bfx_t,\bfy_t\rangle\leq \sqrt{8T\log(1/\delta_0)}.
    \]
    It satisfies that $\bbP[\calE_1] \geq 1 - \delta_0$.
\end{lemma}

\begin{proof}
    First observe that 
    \[
        \sum_{t=1}^T\langle \bfx_t,\bfy_t\rangle=\sum_{t=1}^T\langle \tilde \bfx_t,\bfy_t\rangle+\sum_{t=1}^T(\langle \bfx_t,\bfy_t\rangle-\bbE_t[\langle \bfx_t,\bfy_t\rangle]).
    \] 
    As $\langle \bfx_t,\bfy_t\rangle\in [-1,1]$, according to Corollary~\ref{feigel-cor}, with probability at least $1-\delta_0$, we have 
    \[
        \sum_{t=1}^T(\langle \bfx_t,\bfy_t\rangle-\bbE_t[\langle \bfx_t,\bfy_t\rangle])\leq \sqrt{8T\log(1/\delta_0)},
    \] 
    which concludes the proof.
\end{proof}

\begin{lemma}
    Denote by $\calE_2$ the event that 
    \[
    \sum_{t=1}^T \langle \tilde\bfx_t,\bbE_t[\hat{\bfy}_t]-\hat{\bfy}_t\rangle\leq b\cdot m\sqrt{2T\log(1/\delta_0)}.
    \]
    It satisfies that $\bbP[\calE_2] \geq 1 - \delta_0$.
\end{lemma}

\begin{proof}
    The lemma can be proved directly by applying Lemma~\ref{feigel-lem} and the fact that $\langle \tilde \bfx_t, \hat \bfy_t \rangle \in [0, b \cdot m]$.
\end{proof}

\begin{lemma}
    It always holds that 
    \[
    \sum_{t=1}^T\langle \tilde \bfx_t,\bbE_t[\tilde\bfy_t]-\bbE_t[\hat{\bfy}_t] \rangle\leq b\cdot T\cdot \sum_{i=1}^d\gamma_i.
    \]
\end{lemma}

\begin{proof}
    From definition, 
    \begin{align*}
        \sum_{t=1}^T\langle \tilde \bfx_t,\bbE_t[\tilde\bfy_t]-\bbE_t[\widehat{\bfy}_t]\rangle&=\sum_{t=1}^T\sum_{i=1}^d \tilde \bfx_t[i]\cdot \bbE_t[\tilde\bfy_t[i]] \cdot \Big(1-\frac{\tilde \bfx_t[i]}{\tilde \bfx_t[i]+\gamma_i} \Big)\\
        &\leq b\cdot\sum_{t=1}^T\sum_{i=1}^d \frac{\tilde \bfx_t[i]\cdot \gamma_i}{\tilde \bfx_t[i]+\gamma_i}\tag{$\tilde \bfx_t[i] \leq 1, \bbE_t[\tilde\bfy_t[i]] \leq b$}\\
        & \leq b\cdot \sum_{t=1}^T\sum_{i=1}^d \gamma_i\\
        & = b\cdot T\cdot\sum_{i=1}^d \gamma_i
    \end{align*}
\end{proof}

\begin{lemma}
    Define $\beta_i := 2\gamma_i / b$.
    Let $\calE_3$ be the event that simultaneously for all $i\in \llbracket d\rrbracket$, 
    \[\sum_{t=1}^T(\widehat{\bfy}_t[i]- \bbE_t[\tilde\bfy_t[i]])\leq \frac{\log(d/\delta_0)}{\beta_i}\]
    It satisfies that $\bbP[\calE_3] \geq 1 - \delta_0$. Furthermore, under event $\calE_3$, it satisfies that 
    \[
    \sum_{t=1}^T\langle \bfx,\widehat{\bfy}_t-\bbE_t[\tilde{\bfy}_t]\rangle\leq \sum_{i=1}^d \bfx[i]\cdot\frac{\log(d/\delta_0)}{\beta_i}, \qquad \forall \bfx\in \calX.
    \]
\end{lemma}

\begin{proof}
    Fix $i\in \llbracket d\rrbracket$. First we have the following:
    \begin{align*}
        \widehat{\bfy}_t[i]&=\frac{\bfx_t[i]\cdot \ell_{t,i}}{\tilde \bfx_t[i]+\gamma_i}\\
        &\leq \frac{\bfx_t[i]\cdot \ell_{t,i}}{\tilde \bfx_t[i]+(\gamma_i/b)\cdot \ell_{t,i}}\tag{as $\ell_{t,i}\in [0,b]$}\\
        & =\frac{1}{\beta_i}\cdot \frac{\beta_i\cdot\bfx_t[i]\cdot \ell_{t,i}}{\tilde \bfx_t[i]+(\beta_i/2) \cdot \ell_{t,i}} \tag{as $\beta_i=\frac{2\gamma_i}{b}$}\\
        & \leq \frac{1}{\beta_i}\cdot \frac{\beta_i\cdot\bfx_t[i]\cdot \ell_{t,i}}{\tilde \bfx_t[i]+(\beta_i/2)\cdot \bfx_t[i]\cdot \ell_{t,i}}\tag{as $\bfx_t[i]\in\{0,1\}$}\\
        & =\frac{1}{\beta_i}\cdot \frac{\beta_i\cdot\tilde \bfy_t[i]}{1+ (\beta_i/2)\cdot \tilde \bfy_t[i]}\tag{as $\tilde \bfy_t[i]=\bfx_t[i]\cdot \ell_{t,i}/\tilde \bfx_t[i]$}\\
        &\leq \frac{1}{\beta_i}\cdot \ln{(1+\beta_i\cdot \tilde \bfy_t[i])}\tag{as $\frac{z}{1+z/2}\leq \ln{(1+z)}$ for all $z\geq 0$}\\
    \end{align*}
    
    Next, we have the following:
    \begin{align*}
        \bbE_t[\exp(\beta_i\widehat{\bfy}_t[i])]&\leq \bbE_t[(1+\beta_i\tilde{\bfy}_t[i])]\\
        &= 1+\beta_i \bbE_t[\tilde \bfy_t[i]]\\
        &\leq \exp(\beta_i \bbE_t[\tilde \bfy_t[i]]) \tag{as $1+z\leq\exp(z)$ for all $z\in\mathbb{R}$}
    \end{align*}
    
    Hence, the process $Z_0=1$ and $Z_t=\exp(\beta_i\sum_{\tau=1}^t(\widehat{\bfy}_\tau[i]-\bbE_t[\tilde \bfy_\tau[i]]))$ for all $t\geq 1$ is a supermartingale with respect to $(\calF_t)$ as $\mathbb{E}_t[Z_t]\leq Z_{t-1}$. Hence, we have $\mathbb{E}[Z_{t}]\leq\mathbb{E}[Z_{t-1}]\leq\ldots \leq 1$. Therefore, by Markov inequality we have,
    \begin{equation*}
        \mathbb{P}\left[\sum_{t=1}^T\widehat{\bfy}_t[i]-\bbE_t[\tilde \bfy_t[i]]>\frac{\log(d/\delta_0)}{\beta_i}\right]\leq \mathbb{E}\left[\exp\left(\beta_i\cdot\sum_{t=1}^T(\widehat{\bfy}_t[i]-\bbE_t[\tilde \bfy_t[i]])\right)\right]\cdot\exp\left(\log(d/\delta_0)\right)\leq \frac{\delta_0}{d}
    \end{equation*}
    By union bound over $i \in \llbracket d \rrbracket$, we get that the event $\calE_3$ holds with probability at least $1-\delta_0$.

    Finally, under event $\calE_3$,
    \begin{equation*}
        \sum_{t=1}^T\langle \bfx,\widehat{\bfy}_t-\bbE_t[\tilde{\bfy}_t]\rangle=\sum_{i=1}^d\bfx[i]\cdot\sum_{t=1}^T(\widehat{\bfy}_t[i]-\bbE_t[\tilde{\bfy}_t[i]])\leq \sum_{i=1}^d \bfx[i]\cdot\frac{\log(d/\delta_0)}{\beta_i}
    \end{equation*}
\end{proof}

\begin{lemma}
    Given Bregman divergence $\calD_F(p,q):=F(p)-F(q)-\langle \nabla F(q),p-q\rangle$, let 
    \[\mathtt{VAR}_t:=\langle \tilde \bfx_t-\tilde \bfx_{t+1},\widehat{\bfy}_t^+\rangle-\frac{1}{\eta}\cdot\calD_F(\tilde \bfx_{t+1},\tilde \bfx_t).\]
    Denote by $\calE_4$ the event that \[
    \sum_{t=1}^T\mathtt{VAR}_t\leq \sum_{t=1}^T\mathbb{E}_t\Big[\frac{\eta}{2}||\widehat{\bfy}_t^+||^2_{(\nabla^2F(\bfz_t))^{-1}}\Big]+b\cdot m \cdot\sqrt{2T\log(1/\delta_0)}
    \]
    where $\widehat{\bfy}_t^+\in\mathbb{R}^d$ is a vector defined as 
    \[
    \widehat{\bfy}_t^+[i]:=\widehat{\bfy}_t[i]\cdot\ind[\tilde\bfx_{t+1}[i]\leq\tilde \bfx_t[i]].
    \]
    It satisfies that $\bbP[\calE_4] \geq 1 - \delta_0$.
\end{lemma}

\begin{proof}
    Let $\mathtt{VAR}_t^+:=\max\{\mathtt{VAR}_t,0\}$.
    Since $\calD_F(\tilde \bfx_{t+1},\tilde \bfx_t)\geq 0$ and $\widehat{\bfy}_t^+[i]\geq 0$ for all $i\in\llbracket d\rrbracket$, we obtain the following:
    \begin{equation*}
        \mathtt{VAR}_t^+\leq \langle \tilde\bfx_t,\widehat{\bfy}_t^+\rangle \leq \langle\tilde\bfx_t,\widehat{\bfy}_t\rangle \leq b\cdot m. 
    \end{equation*}
    According to Lemma~\ref{feigel-lem}, with probability at least $1-\delta_0$, we have
    \begin{equation*}
        \sum_{t=1}^T\mathtt{VAR}_t^+\leq \sum_{t=1}^T\mathbb{E}_t[\mathtt{VAR}_t^+]+ b\cdot m \cdot\sqrt{2T\log(1/\delta_0)}.
    \end{equation*}
    
    Next, due to Lemma \ref{technical-lem-hessian}, we have $\mathtt{VAR}_t\leq \frac{\eta}{2}||\widehat{\bfy}_t^+||^2_{(\nabla^2F(\bfz_t))^{-1}}$, where $\bfz_t$ is some point on the chord $[\tilde \bfx_t,\tilde \bfx_{t+1}]$. Since $\frac{\eta}{2}||\widehat{\bfy}_t^+||^2_{(\nabla^2F(\bfz_t))^{-1}}\geq 0$, it follows that $\mathtt{VAR}_t^+\leq\frac{\eta}{2}||\widehat{\bfy}_t^+||^2_{(\nabla^2F(\bfz_t))^{-1}}$. 
    
    Since $\mathtt{VAR}_t\leq \mathtt{VAR}_t^+$ for all $t\in\llbracket T\rrbracket$, the event $\calE_4$ holds with probability at least $1-\delta_0$.
\end{proof}

Let $\calE:=\bigcup\limits_{i=1}^4 \calE_i$ be our good event. Due to union bound, the event $\calE$ holds with probability at least $1-5\delta_0$. Let us assume that the good event $\calE$ holds and let us fix $\bfx\in \calX$. First we have the following:
\begin{align*}
    R_T(\bfx)&=\sum_{t=1}^T\langle \bfx_t-\bfx,\bfy_t\rangle\\
    &\leq \sum_{t=1}^T\langle \tilde\bfx_t-\bfx,\bfy_t\rangle+\sqrt{8T\log(1/\delta_0)}\tag{as event $\calE_1$ holds}\\
    &= \sum_{t=1}^T\langle \tilde\bfx_t-\bfx,\bbE_t[\tilde{\bfy}_t]\rangle+\sqrt{8T\log(1/\delta_0)}\\
    &= \sum_{t=1}^T\langle \tilde \bfx_t-\bfx,\widehat{\bfy}_t\rangle+\sum_{t=1}^T\langle \tilde\bfx_t,\mathbb{E}_t[\widehat{\bfy}_t]-\widehat{\bfy}_t\rangle+\sum_{t=1}^T\langle \tilde \bfx_t,\bbE_t[\tilde{\bfy}_t]-\mathbb{E}_t[\widehat{\bfy}_t]\rangle\\
    &\qquad +\sum_{t=1}^T\langle \bfx,\widehat{\bfy}_t-\bbE_t[\tilde{\bfy}_t]\rangle+\sqrt{8T\log(1/\delta_0)}\\
    &\leq \sum_{t=1}^T\langle \tilde \bfx_t-\bfx,\widehat{\bfy}_t\rangle+\sqrt{8T\log(1/\delta_0)}+b\cdot m\sqrt{2T\log(1/\delta_0)}\\
    &\qquad +b\cdot T\sum_{i=1}^d\gamma_i+b\cdot \sum_{i=1}^d\bfx[i]\cdot \frac{\log(d/\delta_0)}{2\gamma_i} \tag{as events $\calE_2,\calE_3$ hold}
\end{align*}

Observe that if $\gamma_i=\gamma$ for all $i\in\llbracket[d\rrbracket]$, we also get the following as $||x||_1\leq m$:
\begin{equation*}
    R_T(\bfx)\leq \sum_{t=1}^T\langle \tilde \bfx_t-\bfx,\widehat{\bfy}_t\rangle+\sqrt{8T\log(1/\delta_0)}+b\cdot m\sqrt{2T\log(1/\delta_0)}+b\cdot T\cdot d\cdot \gamma+b\cdot m\cdot \frac{\log(d/\delta_0)}{2\gamma}
\end{equation*}

As $\tilde\bfx_t$ is the solution to our FTRL equation above, we get the following using the standard FTRL analysis from \citet{lattimore2020bandit} and the fact that event $\calE_4$ holds:
\begin{equation*}
    \sum_{t=1}^T\langle \tilde \bfx_t-\bfx,\widehat{\bfy}_t\rangle\leq \frac{\text{diam}_F}{\eta}+\sum_{t=1}^T\mathtt{VAR}_t\leq \frac{\text{diam}_F}{\eta}+\sum_{t=1}^T\mathbb{E}_t\left[\frac{\eta}{2}||\widehat{\bfy}_t^+||^2_{(\nabla^2F(\bfz_t))^{-1}}\right]+b\cdot m \cdot\sqrt{2T\log(1/\delta_0)}
\end{equation*}
where $\bfz_t$ is some point on the chord $[\tilde\bfx_t,\tilde \bfx_{t+1}]$ and $\text{diam}_F:=\max_{\bfx,\bfx'\in\operatorname{co}(\calX)}F(\bfx)-F(\bfx')$. Now we have the following theorem under our assumptions.
\begin{theorem}\label{appendix:ftrl-thm}
    Let $\delta_0=\frac{\delta}{5}$ and $\gamma_i=\sqrt{\frac{m\log(5d/\delta)}{dT}}$ for all $i\in \llbracket d\rrbracket$.  Given an FTRL algorithm satisfying the assumptions of this section, we have the following guarantee on regret with probability at least $1-\delta$ against any adaptive adversary:
    \begin{equation*}
        \max_{\bfx\in \calX}R_T(\bfx)\leq \frac{\text{diam}_F}{\eta}+\sum_{t=1}^T\mathbb{E}_t\left[\frac{\eta}{2}||\widehat{\bfy}_t^+||^2_{(\nabla^2F(\bfz_t))^{-1}}\right]+c\cdot b\cdot \sqrt{mdT\log(d/\delta)}
    \end{equation*}
    where $c$ is some absolute constant and $\bfz_t$ is some point on the chord $[\tilde\bfx_t,\tilde \bfx_{t+1}]$.
\end{theorem}
\section{DAGs: Additional Details}

\subsection{Omitted Details from \Cref{sec:alg-equal}}\label{appendix:alg-equal}
\begin{algorithm2e}[!ht] \label{alg:equal-path-length}
\caption{FTRL for online shortest path in a DAG $G=(V,E)$ with equal path lengths}
Let $K$ be the length of every path in the DAG $G$ from source to sink.

\For{$t = 1$ \KwTo $T$} {

Compute $\tilde{\bfx}_{t}\gets \arg\min_{\bfx \in \mathrm{co}(\calX)} \left( \eta \sum_{\tau=1}^{t-1} \langle \bfx, \widehat{\bfy}_\tau \rangle + F(\bfx) \right)$.

Initialize path $P_t \leftarrow (\source)$ and reset $v_0 \leftarrow \source$.

\For{$i = 1$ \KwTo $K$} {
    Sample outgoing edge $e_i = (v_{i-1}, v_i) \in \delta^+(v_{i-1})$ with probability proportional to $\tilde \bfx_{t}[e_i]$.

    Update $P_t \leftarrow P_t \circ (e_i, v_i)$
}

Choose the path $P_t$ and observe loss its $\ell_{t}$.

Determine $\bfx_{t} \in \calX$ corresponding to the path $P_t$ and construct the loss estimator $\hat\bfy_t$.

}

\end{algorithm2e}
First, we prove the following proposition.
\begin{proposition}\label{prop:expectation-bfxt}
    $\mathbb{E}_t[\bfx_t]=\tilde \bfx_t$
\end{proposition}
\begin{proof}
    Let $v_1,v_2,\ldots,v_{|V|}$ be the topological order of the vertices in the DAG $G$, where $v_1=\source$ and $v_{|V|}=\sink$.
    We now prove our proposition using mathematical induction. Let $P(i)$ be the statement that for all $j\in \llbracket i\rrbracket$, we have $\mathbb{E}_t[\bfx_t[v_j]]=\tilde \bfx_t[v_j]$ and $\mathbb{E}_t[\bfx_t[e]]=\tilde \bfx_t[e]$ for any outgoing edge $e$ from $v_j$.

    Consider the base case of $i=1$. First, observe that $\mathbb{E}_t[\bfx_t[v_1]]=1=\tilde\bfx_t[v_1]$. Next, due to our sampling procedure, we have $$\mathbb{E}_t[\bfx_t[e]]=\frac{\tilde\bfx_t[e]}{\sum_{e'\in \delta^{+}(v_1)}\tilde\bfx_t[e]}=\tilde \bfx_t[e]$$ for any outgoing edge $e$ from $v_1$. Hence, $P(1)$ is true.

    Next, let us make the inductive hypothesis that $P(k)$ is true. Now we show that $P(k+1)$ is true.  First, observe that 
    \begin{align*}
        \mathbb{E}_t[\bfx_t[v_{k+1}]]&=\mathbb{E}_t\Big[\sum_{e\in\delta^{-}(v_{k+1})}\bfx_t[e]\Big]\\
        &=\sum_{e\in\delta^{-}(v_{k+1})}\tilde \bfx_t[e]\tag{due to inductive hypothesis}\\
        &=\tilde \bfx_t[v_{k+1}]\tag{due to flow constraints}
    \end{align*}
    Next, observe that for any outgoing edge $e$ from $v_{k+1}$, we have:
    \begin{align*}
        \mathbb{E}_t[\bfx_t[e]]&=\mathbb{E}_t[\bfx_t[e] \mid \bfx_t[v_{k+1}]=1]\cdot \mathbb{P}_t[\bfx_t[v_{k+1}]=1]\\
        &=\frac{\tilde\bfx_t[e]}{\sum_{e'\in\delta^{+}(v_{k+1})}\tilde\bfx_[e']}\cdot \tilde\bfx_t[v_{k+1}]\tag{due to our sampling procedure}\\
        &=\frac{\tilde\bfx_t[e]}{\tilde\bfx_t[v_{k+1}]}\cdot \tilde\bfx_t[v_{k+1}]\tag{due to flow constraints}\\
        &=\tilde\bfx_t[e]
    \end{align*}
    Hence, $P(k+1)$ is true. Hence, due to principle of mathematical induction, we have $\mathbb{E}_t[\bfx_t]=\tilde \bfx_t$.
\end{proof}

Next, recall that if all the paths have equal length of $K$, then we have $\mathbb{E}_t[\langle \bfx-\bfx',\tilde\bfy_t\rangle]=\langle \bfx-\bfx',\bfy_t\rangle$ for all $\bfx,\bfx'\in\calX$. If we use $F(\bfx)=-\sum_{v\in V}\sqrt{\bfx[v]}-\sum_{e\in E}\sqrt{\bfx[e]}$ as our regularizer, it easily follows that \Cref{eq:optimization1} has a unique minimizer $\tilde\bfx_t$ and $\tilde\bfx_t[v]>0$ for all $v\in V$ and $\tilde\bfx_t[e]>0$ for all $e\in E$. Therefore,  $\nabla^2F(z)$ is a diagonal matrix for any point $z$ on the chord $[\tilde x_t,\tilde x_{t+1}]$. Let $\hat\bfy_t^+\in\mathbb{R}^{V\cup E}_{\geq 0}$ be a vector indexed by the elements in $V\cup E$ such that $\hat\bfy_t^+[v]=\hat\bfy_t[v]\cdot\ind[\tilde \bfx_{t+1}[v]\leq\tilde\bfx_t[v]]$ for all $v\in V$ and $\hat\bfy_t^+[e]=\hat\bfy_t[e]\cdot\ind[\tilde \bfx_{t+1}[e]\leq\tilde\bfx_t[e]]$ for all $e\in E$. Let us set every entry of $\boldsymbol\gamma$ to $\sqrt{\frac{K\log(5(|V|+|E|)/\delta)}{|E|T}}$. Now we apply Lemma \ref{appendix:ftrl-thm} from \Cref{appendix:ftrl-framework} to get the following:

\begin{theorem}\label{main:generic-thm}
    Algorithm \ref{alg:equal-path-length} has the following guarantee on regret with probability at least $1-\delta$ against any adaptive adversary:
    \begin{equation*}
        \max_{\bfx\in \calX}R_T(\bfx)\leq \frac{\operatorname{diam}_F(\operatorname{co}(\calX))}{\eta}+\sum_{t=1}^T\mathbb{E}_t\left[\frac{\eta}{2}||\widehat{\bfy}_t^+||^2_{(\nabla^2F(z_t))^{-1}}\right]+c\cdot \sqrt{K|E|T\log(|E|/\delta)}
    \end{equation*}
    where $c$ is some absolute constant, $\operatorname{diam}_F(\operatorname{co}(\calX)):=\max_{\bfx,\bfx'\in \operatorname{co}(\calX)}F(\bfx)-F(\bfx')$, and $z_t$ is some point on the chord $[\tilde\bfx_t,\tilde \bfx_{t+1}]$.
\end{theorem}

First we upper bound $\sum\limits_{v\in V}\sqrt{\bfx[v]}+\sum\limits_{e\in E}\sqrt{\bfx[e]}$ for any $\bfx\in\operatorname{co}(\calX)$ as follows:
\begin{align*}
    \sum\limits_{v\in V}\sqrt{\bfx[v]}+\sum\limits_{e\in E}\sqrt{\bfx[e]}&\leq \sqrt{(|V|+|E|)\cdot (\sum\limits_{v\in V}\bfx[v]+\sum\limits_{e\in E}\bfx[e])} \tag{Cauchy-Schwarz inequality}\\
    &\leq\sqrt{(2|E|+1)\cdot (\sum\limits_{v\in V}\bfx[v]+\sum\limits_{e\in E}\bfx[e])} \tag{as $|V|\leq |E|+1$}\\
    &\leq \sqrt{4(K+1)(|E|+1)}
\end{align*}
We get the last inequality due to the following. First by definition any point in $\bfx\in\operatorname{co}(\calX)$ is a convex combination of points in $\calX$. Therefore an upper bound on $\max_{\bfx'\in\calX}\sum\limits_{v\in V}\bfx'[v]+\sum\limits_{e\in E}\bfx'[e]$ is also an upper bound on $\sum\limits_{v\in V}\bfx[v]+\sum\limits_{e\in E}\bfx[e]$. As any path in the DAG $G$ has exactly $K$ edges, we have $\sum\limits_{v\in V}\bfx'[v]=K+1$ and $\sum\limits_{e\in E}\bfx'[e]=K$ for any $\bfx'\in\calX$.

\noindent
Now we upper bound the diameter as follows:
\begin{equation*}
    \operatorname{diam}_F(\operatorname{co}(\calX))\leq \max\limits_{\bfx\in \operatorname{co}(\calX)}\sum\limits_{v\in V}\sqrt{\bfx[v]}+\sum\limits_{e\in E}\sqrt{\bfx[e]}\leq \sqrt{4(K+1)(|E|+1)}
\end{equation*}

Next, we upper bound the second term of the regret upper bound above. Observe that $\nabla^2F(\bfz)=\operatorname{diag}(1/(4\bfz^{3/2}))$. Due to the definition of $\hat\bfy_t^+$, it follows that the term $||\hat\bfy_t^+||_{\nabla^2F(\bfz_t)^{-1}}^2$ is maximized when $\bfz_t=\tilde \bfx_t$. Hence, we have $\mathbb{E}_t\left[||\hat\bfy_t^+||_{\nabla^2F(z_t)^{-1}}^2\right]\leq 16\sum\limits_{v\in V}\sqrt{\tilde \bfx_t[v]}+16\sum\limits_{e\in E}\sqrt{\tilde\bfx_t[e]}\leq 32\sqrt{(K+1)(|E|+1)}$. Hence, we have the following by setting $\eta=\frac{1}{\sqrt{T}}$:
\begin{align*}
     \mathrm{Regret}(T)&\leq\frac{\operatorname{diam}_F(\operatorname{co}(\calX))}{\eta}+\sum_{t=1}^T\mathbb{E}_t\left[\frac{\eta}{2}||\widehat{\bfy}_t^+||^2_{(\nabla^2F(\bfz_t))^{-1}}\right]+c\cdot \sqrt{K|E|T\log(|E|/\delta)}\\
     &\leq \frac{2\sqrt{(K+1)|E|}}{\eta}+\frac{32T\eta}{2}\sqrt{(K+1)|E|}+c\cdot \sqrt{K|E|T\log(|E|/\delta)}\\
     &\leq 18\sqrt{(K+1)(|E|+1)T}+c\cdot \sqrt{K|E|T\log(|E|/\delta)}\\
\end{align*}
where we get the last inequality by setting $\eta=\frac{1}{\sqrt{T}}$.

\begin{theorem}\label{thm:equal-path-length-regret}
    Under the assumption that every path from the source to the sink has length $K$, Algorithm~\ref{alg:equal-path-length} incurs a regret of at most $\mathcal{O}(\sqrt{K|E|T\log(|E|/\delta)})$ against any adaptive adversary with probability at least $1 - \delta$.
\end{theorem}
\subsection{Omitted Details from \Cref{sec:alg-make-equal}}\label{appendix:alg-make-equal}
\begin{algorithm2e}[ht] \label{alg:equal-path-length-2}
\caption{FTRL for online shortest path in a DAG $G=(V,E)$}
Let $K$ be the longest path in the DAG $G$ from source to sink.

\For{$t = 1$ \KwTo $T$} {

Compute $\tilde{\bfx}_{t}\gets \arg\min_{\bfx \in \mathrm{co}(\calX^\dagger)} \left( \eta \sum_{\tau=1}^{t-1} \langle \bfx, \widehat{\bfy}_\tau^\dagger \rangle + F(\bfx) \right)$.

Initialize path $P_t \leftarrow (\source)$ and reset $v_0 \leftarrow \source$.

\For{$i = 1$ \KwTo $K$} {
    Sample outgoing edge $e_i = (v_{i-1}, v_i) \in \delta^+(v_{i-1})$ with probability proportional to $\tilde \bfx_{t}[e_i]$.

    Update $P_t \leftarrow P_t \circ (e_i, v_i)$
}

Choose the path $P_t$ and observe loss its $\ell_{t}$.

Determine $\bfx_{t}^\dagger \in \calX^\dagger$ corresponding to the path $P_t$ and construct the loss estimator $\hat\bfy_t^\dagger$.

}

\end{algorithm2e}

Recall that $K(v)$ is the length of the longest path from $\source$ to $v$ and $K(\source)=0$. For any edge $(u,v) \in E$, define $\calI((u,v)) := \{i \in \llbracket K-1\rrbracket : K(u) < i < K(v)\}$. Now we make the following assumption.
\begin{assumption}
    For any $i\in\llbracket K-1 \rrbracket$, $|e \in E : i \in \calI(e)| \geq 1$.
\end{assumption}

Note that if there exists an index $i\in\llbracket K-1 \rrbracket$ such that $|e \in E : i \in \calI(e)| =0$, then we have $b(\bfx)[i]=0$ for all $\bfx\in \calX$. Consequently, this bit can be excluded from our representation.

Next, we prove the following proposition.
\begin{proposition}\label{prop:expectation-bfxt-2}
    $\mathbb{E}_t[\bfx_t^\dagger]=\tilde \bfx_t$
\end{proposition}
\begin{proof}
    Due to Proposition \ref{prop:expectation-bfxt}, we have $\mathbb{E}_t[\bfx_t^\dagger[v]]=\tilde \bfx_t[v]$ for all $v\in V$ and $\mathbb{E}_t[\bfx_t^\dagger[e]]=\tilde \bfx_t[e]$ for all $e\in E$. Fix $i\in \llbracket K-1\rrbracket$. Later in this section, we prove that for any $\bfx\in \operatorname{co}(\calX^\dagger)$, we have $\bfx[i] = \sum_{e \in E: i \in \calI(e)} \bfx[e]$. Hence, we have $\tilde \bfx_t[i]=\sum_{e \in E: i \in \calI(e)} \tilde \bfx_t[e]=\mathbb{E}_t[\sum_{e \in E: i \in \calI(e)} \bfx_t^\dagger[e]]=\mathbb{E}_t[\bfx_t^\dagger[i]]$.
\end{proof}

Recall the definition of the loss estimators $\hat\bfy_t^\dagger$ and $\tilde \bfy_t^\dagger$ from \Cref{sec:alg-make-equal}. Next, recall that we have $\mathbb{E}_t[\langle \bfx_{(1)}^\dagger-\bfx_{(2)}^\dagger,\tilde\bfy_t^\dagger\rangle]=\langle \bfx_{(1)}-\bfx_{(2)},\bfy_t\rangle$ for all $\bfx_{(1)},\bfx_{(2)}\in\calX$. If we use $F(\bfx)=-\sum_{v\in V}\sqrt{\bfx[v]}-\sum_{e\in E}\sqrt{\bfx[e]}-\sum_{i\in\llbracket K-1\rrbracket}\sqrt{\bfx[i]}$ as our regularizer, it easily follows that our FTRL equation has a unique minimizer $\tilde\bfx_t$ and $\tilde\bfx_t[v]>0$ for all $v\in V$, $\tilde\bfx_t[e]>0$ for all $e\in E$ and $\tilde\bfx_t[i]>0$ for all $i\in\llbracket K-1\rrbracket$. Therefore,  $\nabla^2F(z)$ is a diagonal matrix for any point $z$ on the chord $[\tilde x_t,\tilde x_{t+1}]$. Let $\hat\bfy_t^+\in\mathbb{R}^{V\cup E\cup \llbracket K-1\rrbracket}_{\geq 0}$ be a vector indexed by the elements in $V\cup E\cup \llbracket K-1\rrbracket$ such that $\hat\bfy_t^+[v]=\hat\bfy_t^\dagger[v]\cdot\ind[\tilde \bfx_{t+1}[v]\leq\tilde\bfx_t[v]]$ for all $v\in V$, $\hat\bfy_t^+[e]=\hat\bfy_t[e]\cdot\ind[\tilde \bfx_{t+1}[e]\leq\tilde\bfx_t[e]]$ for all $e\in E$, and $\hat\bfy_t^+[i]=\hat\bfy_t[i]\cdot\ind[\tilde \bfx_{t+1}[i]\leq\tilde\bfx_t[i]]$ for all $i\in \llbracket K-1\rrbracket$. Let us set every entry of $\boldsymbol\gamma$ and $\hat{\boldsymbol\gamma}$ to $\sqrt{\frac{K\log(5(|V|+|E|+K)/\delta)}{|E|T}}$. Now we apply Lemma \ref{appendix:ftrl-thm} from \Cref{appendix:ftrl-framework} to get the following:

\begin{theorem}\label{main:generic-thm-2}
    Algorithm \ref{alg:equal-path-length} has the following guarantee on regret with probability at least $1-\delta$ against any adaptive adversary:
    \begin{equation*}
        \max_{\bfx\in \calX}R_T(\bfx)\leq \frac{\operatorname{diam}_F(\operatorname{co}(\calX^\dagger))}{\eta}+\sum_{t=1}^T\mathbb{E}_t\left[\frac{\eta}{2}||\widehat{\bfy}_t^+||^2_{(\nabla^2F(z_t))^{-1}}\right]+c\cdot \sqrt{K|E|T\log(|E|/\delta)}
    \end{equation*}
    where $c$ is some absolute constant, $\operatorname{diam}_F(\operatorname{co}(\calX^\dagger)):=\max_{\bfx,\bfx'\in \operatorname{co}(\calX^\dagger)}F(\bfx)-F(\bfx')$, and $z_t$ is some point on the chord $[\tilde\bfx_t,\tilde \bfx_{t+1}]$.
\end{theorem}

First we upper bound $\sum\limits_{v\in V}\sqrt{\bfx[v]}+\sum\limits_{e\in E}\sqrt{\bfx[e]}+\sum\limits_{i=1}^{K-1}\sqrt{\bfx[i]}$ for any $\bfx\in\operatorname{co}(\calX^\dagger)$ as follows:
\begin{align*}
    &\sum\limits_{v\in V}\sqrt{\bfx[v]}+\sum\limits_{e\in E}\sqrt{\bfx[e]}+\sum\limits_{i=1}^{K-1}\sqrt{\bfx[i]}\\
    &\leq \sqrt{(|V|+|E|+K-1)\cdot (\sum\limits_{v\in V}\bfx[v]+\sum\limits_{e\in E}\bfx[e]+\sum\limits_{i=1}^{K-1}\bfx[i])} \tag{Cauchy-Schwarz inequality}\\
    &\leq\sqrt{(3|E|)\cdot (\sum\limits_{v\in V}\bfx[v]+\sum\limits_{e\in E}\bfx[e]+\sum\limits_{i=1}^{K-1}\bfx[\ell+i])} \tag{as $|V|\leq |E|+1,\;K\leq |E|$}\\
    &\leq \sqrt{9K|E|}
\end{align*}
We get the last inequality due to the following. First by definition any point in $\bfx\in\operatorname{co}(\calX^\dagger)$ is a convex combination of points in $\calX$. Therefore, an upper bound on $\max_{x'\in \calX^\dagger}\sum\limits_{v\in V}\bfx'[v]+\sum\limits_{e\in E}\bfx'[e]+\sum\limits_{i=1}^{K-1}\bfx'[i]$ is also an upper bound on $\sum\limits_{v\in V}\bfx[v]+\sum\limits_{e\in E}\bfx[e]+\sum\limits_{i=1}^{K-1}\bfx[\ell+i]$. As any path in the DAG $G$ has at most $K$ edges, we have $\sum\limits_{v\in V}\bfx'[v]\leq K+1$ and $\sum\limits_{e\in E}\bfx'[e]\leq K$. We also have $\sum\limits_{i=1}^{K-1}\bfx'[i]\leq K-1$.

\noindent
Now we upper bound the diameter as follows:
\begin{equation*}
    \operatorname{diam}_F(\operatorname{co}(\calX^\dagger))\leq \max\limits_{\bfx\in \operatorname{co}(\calX^\dagger)}\sum\limits_{v\in V}\sqrt{\bfx[v]}+\sum\limits_{e\in E}\sqrt{\bfx[e]}+\sum\limits_{i=1}^{K-1}\sqrt{\bfx[i]}\leq \sqrt{9K|E|}
\end{equation*}

Next, we upper bound the second term of the regret upper bound above. Next observe that $\nabla^2F(z)=\operatorname{diag}(1/(4z^{3/2}))$. Due to the definition of $\hat\bfy_t^+$, it follows that the term $||\hat\bfy_t^+||_{\nabla^2F(z_t)^{-1}}^2$ is maximized when $z_t=\tilde \bfx_t$. Hence, we have $\mathbb{E}_t\left[||\bfy_t'||_{\nabla^2F(z_t)^{-1}}^2\right]\leq 16\sum\limits_{v\in V}\sqrt{\tilde \bfx_t[v]}+16\sum\limits_{e\in E}\sqrt{\tilde\bfx_t[e]}+16\sum\limits_{i=1}^{K-1}\sqrt{\tilde \bfx_t[i]}\leq 48\sqrt{K|E|}$.

Hence, we have the following by setting $\eta=\frac{1}{\sqrt{T}}$:
\begin{align*}
     \mathrm{Regret}(T)&\leq\frac{\operatorname{diam}_F(\operatorname{co}(\calX^\dagger))}{\eta}+\sum_{t=1}^T\mathbb{E}_t\left[\frac{\eta}{2}||\widehat{\bfy}_t^+||^2_{(\nabla^2F(z_t))^{-1}}\right]+c\cdot \sqrt{K|E|T\log(|E|/\delta)}\\
     &\leq \frac{3\sqrt{K|E|}}{\eta}+\frac{48T\eta}{2}\sqrt{K|E|}+c\cdot \sqrt{K|E|T\log(|E|/\delta)}\\
     &\leq 27\sqrt{K|E|T}+c\cdot \sqrt{K|E|T\log(|E|/\delta)}\\
\end{align*}

Our analysis leads to the following main theorem.
\begin{theorem}\label{thm:equal-path-length-regret-2}
    Under the assumption that every path from the source to the sink has length at most $K$, Algorithm~\ref{alg:equal-path-length-2} incurs a regret of at most $\mathcal{O}(\sqrt{K|E|T\log(|E|/\delta)})$ against any adaptive adversary with probability at least $1 - \delta$.
\end{theorem}

We now show that $\operatorname{co}(\calX^\dagger)$ can be represented using a polynomial number of linear constraints. Recall that $\calX^\dagger$ represents the set of all paths, including the appended $K-1$ bits. We now claim that, given a point $\bfx \in [0,1]^{V\cup E\cup\llbracket K-1\rrbracket}$, we can efficiently determine whether $\bfx$ lies within the convex hull of $\calX^\dagger$. The coordinate values corresponding to the edges and vertices must satisfy the flow constraints, which can be verified efficiently.

Recall that for any edge $(u,v) \in E$, we defined $\calI((u,v)) := \{i \in \llbracket K-1\rrbracket : K(u) < i < K(v)\}$. For any two edges $e_1, e_2$ on the same path, observe that $\calI(e_1) \cap \calI(e_2) = \emptyset$. Consequently, the coordinate values of $\bfx$ corresponding to the indices in $\llbracket K-1\rrbracket$ must satisfy the following condition:  
\[
\bfx[i] = \sum_{e \in E: i \in \calI(e)} \bfx[e] \quad \forall i \in \llbracket K-1\rrbracket.
\]  

If this condition fails for some $i \in \llbracket K-1\rrbracket$, then $\bfx$ does not lie in the convex hull of $\calX^\dagger$. Conversely, if all flow constraints and the above condition hold, then $\bfx$ belongs to the convex hull. This verification is efficient, as $\calI(e)$ can be computed efficiently.

\textbf{Remark:} One does not need to compute the exact minimizer $\tilde{\bfx}_t$ of the FTRL equation in each round $t$. Our analysis remains valid if we instead compute an approximate minimizer $\hat{\bfx}_t$ satisfying $ \|\hat{\bfx}_t - \tilde{\bfx}_t\|_{\infty} \leq \frac{1}{T^2} $ and $ \|\hat{\bfx}_t\| > 0 $. Such an approximate minimizer can be efficiently computed using standard optimization methods, such as the Ellipsoid method, in $\text{poly}(|E|, \log T)$ time steps, since our regularizer is both Legendre and strongly convex over $\operatorname{co}(\calX^\dagger)$, and $\operatorname{co}(\calX^\dagger)$ can be represented using a polynomial number of linear constraints.

\subsection{Omitted Details from \Cref{sec:alg-centroid}}\label{appendix:alg-centroid}
Here, we present the full version of \Cref{sec:alg-centroid}, including the omitted details.

Recall that $\delta^{-}(v)$ denotes the incoming edges and $\delta^{+}(v)$ denotes the outgoing edges of vertex $v$. Let $C: V \to \bbN$ denote the number of distinct paths from the source $\source$ to any vertex $v$. It holds that $C(\source) = 1$ and $C(v) := \sum_{(u, v) \in \delta^{-}(v)} C(u)$ for any $v \neq \source$. According to the definition, it satisfies that $C(\sink) = |\mathcal{X}|$. 

Let $h(v) := \arg\max_{(u, v) \in \delta^{-}(v)} C(u)$ be the incoming edge that brings the maximum number of paths to vertex $v$, with ties broken arbitrarily. Let
$$E^\clubsuit := \{h(v) \mid v \in V \setminus \{\source\}\}$$ 
be the set of all such edges. The underlying subgraph $S := (V, E^\clubsuit)$ forms a directed spanning tree of $G$. The next lemma shows that the number of non-tree edges (edges not in $E^\clubsuit$) on any path from $\source$ to $\sink$ in $G$ is at most $\log |\mathcal{X}|$.

\begin{lemma} \label{lm:centroid-decomposition-logX}
    Let $P = (v_0 = \source, e_1, v_1, \dots, e_k, v_{k} = \sink)$ be a path from the source to sink. We have that the number of non-tree edges on the path $P$ is upper bounded by 
    \begin{align*}
        \sum_{i = 1}^k \ind[e_i \notin E^\clubsuit] \leq \log\left(|\calX|\right).
    \end{align*}
\end{lemma}

\begin{proof}
Since the number of distinct paths is always non-negative, for any $ i \in \llbracket k \rrbracket $, we have  
\[
C(v_i) = \sum_{(u, v_i) \in \delta^{-}(v_i)} C(u) \geq C(v_{i-1}),
\]  
where $ v_{i - 1} \in \delta^{-}(v_i) $. Moreover, for any non-tree edge $ e_i = (v_{i - 1}, v_i) \notin E^\clubsuit $, consider the tree edge $ h(v_i) = (u_{i - 1}, v_i) $, which is an incoming edge to vertex $ v_i $. The number of distinct paths from $ \source $ to $ v_i $ can then be lower bounded by:
\[
C(v_i) = \sum_{(u, v_i) \in \delta^{-}(v_i)} C(u) \geq C(u_{i-1}) + C(v_{i-1}) \geq 2C(v_{i - 1}),
\]
where the last inequality follows from the selection criteria of the tree edge $h(v_i)$. More generally, we have that

\[
\ind[e_i \notin E^\clubsuit] \leq \log\left(\frac{C(v_i)}{C(v_{i-1})}\right).
\]
Since $C(v_k) = C(\sink)= |\mathcal{X}|$ and $C(v_0) = C(\source) = 1$, summing these inequalities yields

\[
\sum_{i = 1}^k \ind[e_i \notin E^\clubsuit] \leq \log\left(\frac{C(\sink)}{C(v_0)}\right) = \log\left(|\mathcal{X}|\right).
\]
\end{proof}

We now introduce the \textit{centroid-based decomposition}: Given a directed tree $S = (V, E^\clubsuit)$, we identify a vertex $c \in V$ such that the connected components $\hat{S}_1, \dots, \hat{S}_k$ resulting from its removal satisfy $|\hat{V}_i| \leq |V|/2$ for all $i \in \llbracket k \rrbracket$, where $\hat{V}_i$ is the set of vertices in the subtree $\hat{S}_i$. Such a vertex $c$, known as the \textit{centroid}, always exists in any tree (see \citep{jordan1869assemblages,della2019new}).  

We associate the centroid $c$ with the tree $S$ by defining $S_c:= S$. The above procedure is then applied recursively to each component $\hat{S}_i$ for $i \in \llbracket k \rrbracket$. If a component reduces to a single vertex $c$, we designate $c$ as its centroid and terminate the recursion.  

Since the sets $\hat{V}_i$ resulting from the removal of $c$ form a partition of $V \setminus \{c\}$, each vertex $v \in V$ will eventually be assigned as the centroid of some subtree $S_v = (V_v, E^\clubsuit_v)$. Consequently, this procedure generates a collection of subtrees $\mathcal{T} := \{S_v : v \in V\}$, where each vertex $v$ is uniquely associated with a subtree of $S$ in which it serves as the centroid.  Furthermore, we define $\mathcal{T}(S_v) := \{S_w : w \in V_v\}$ as the centroid-based decomposition of the subtree $S_v$.  

The above construction transforms $S$ into a hierarchy of subtrees. The following folklore lemma establishes that the centroid-based decomposition systematically organizes every path in $S$.  

\begin{lemma} \label{lm:centroid-decomposition-uniqueness}  
    Let $\mathcal{T}$ be the centroid-based decomposition of the directed tree $S$.  
    For any pair of vertices $(u, v) \in V \times V$, there exists a unique subtree $S_w \in \mathcal{T}$ with centroid $w$ such that:  
    \begin{itemize}  
        \item Both $u$ and $v$ belong to the subtree $S_w$, i.e., $u, v \in V_w$, where $V_w$ is the vertex set of $S_w$.  
        \item The path from $u$ to $v$ in the underlying undirected graph of $S$ passes through $w$.  
    \end{itemize}  
\end{lemma}  
\begin{proof}
    We will prove the statement by induction, showing that it holds for any tree $S = (V, E^\clubsuit)$ with at most $k$ vertices, along with its corresponding centroid-based decomposition $\mathcal{T}$.
    
    \paragraph{Base Case:} When the tree $S = (\{c\}, \emptyset)$ contains only a single vertex $v = c$, the statement holds trivially by $S_c$. The only valid pair is $(v, v)$, and the path from $v$ to itself contains $c$ by definition.
    
    \paragraph{Inductive Step:} Assume the statement holds for all trees with at most $k$ vertices. Consider a tree $S = (V, E^\clubsuit)$ with $|V| = k+1$ vertices and a pair of vertices $(u, v) \in V \times V$. Let $c$ be the centroid of $S$, and consider the path from $u$ to $v$ in the underlying undirected graph of $S$. We distinguish two cases:
    
    \begin{enumerate}
        \item \textbf{If the path from $u$ to $v$ contains $c$:} Then $S_c$ is the desired subtree. Since the undirected path from $u$ to $v$ passes through $c$, $u$ and $v$ must either lie in different connected components formed after removing $c$ from $S$, or one of them is $c$ itself. In both cases, no other subtree $S_{w'} \in \calT$ can contain both $u$ and $v$. Hence, $S_c$ is the only subtree satisfies the condition.
        
        \item \textbf{If the path from $u$ to $v$ does not contain $c$:} In this case, both $u$ and $v$ lie entirely within one of the connected components $\hat S_i$ formed by removing $c$ from $S$. By the induction hypothesis, there exists one subtree $S_w \in  \calT(\hat S_i)$ with the statement holds. For any other subtree $S_{w'} \in  \calT \setminus \calT(\hat S_i)$, we have neither $u$ nor $v$ contained in $S_{w'}$. Thus, $S_w$ is the only subtree satisfies the condition.
    \end{enumerate}
    
    \paragraph{Conclusion:} By mathematical induction, the statement holds for any tree $S = (V, E^\clubsuit)$. Thus, for any pair of vertices $(u, v) \in V \times V$, there exists a unique vertex $c \in V$ such that $u, v \in V_c$ and the path from $u$ to $v$ in the underlying undirected graph of $S$ passes through $c$.
\end{proof}

The following folklore lemma demonstrates that the total number of vertices introduced by the centroid-based decomposition is nearly linear in the number of vertices of the original tree:
\begin{lemma} \label{lm:centroid_total_size}
    Let $\calT$ be the centroid-based decomposition of the directed tree $S$.
    The total number of vertices among $S_v = (V_v, E^\clubsuit_v) \in \calT$ in centroid-based decomposition is upper-bounded by
    \begin{align*}
        \sum_{S_v \in \calT} |V_v| \leq  (1 + \log |V|) |V|.
    \end{align*}
\end{lemma}

\begin{proof}
We will prove the statement by induction, showing that it holds for any tree $S = (V, E^\clubsuit)$ with at most $k$ vertices, along with its corresponding centroid-based decomposition $\mathcal{T}$.

\paragraph{Base Case:} When the tree $S = (\{c\}, \emptyset)$ contains only a single vertex $v = c$, we have:
\[
\sum_{S_v \in \calT} |V_{v}| = |V| = 1 \leq (1 + \log |V|) |V|,
\]
so the inequality holds.

\paragraph{Inductive Step:} Assume the statement holds for all trees with at most $k$ vertices. Consider a tree $S = (V, E^\clubsuit)$ with $|V| = k+1$ vertices. Denote by $\hat S_1, \dots, \hat S_k$ the subtrees after the removal of centroid $c$ from $S$. Let $\hat V_i$ be the set of vertices of $\hat S_i$.  By the induction hypothesis, we have
\[
\sum_{v\in \hat V_i} |V_{v}| \leq (1 + \log |\hat V_i|) |\hat V_i| \leq |\hat V_i| \log |V|,
\]
where the second inequality follows from $|\hat V_i| \leq |V_c| / 2$ which is the property of the centroid.

Therefore, the summation $\sum_{v \in V} |V_v|$ can be upper-bounded via:
\[
\sum_{S_v \in \calT} |V_v| = |V_c| + \sum_{i=1}^k \sum_{v \in \hat V_i} |V_{v}|
\leq |V| + \sum_{i=1}^k |\hat V_i| \log |V| 
\leq  (1 + \log |V|) |V|.
\]
where the last equality holds as $\sum_{i=1}^k |\hat V_i| = |V| - 1$.

\paragraph{Conclusion:} By mathematical induction on the size of the tree, we have that 
\[
        \sum_{S_v \in \calT} |V_v| \leq  (1 + \log |V|) |V|.
\]
\end{proof}

\begin{figure}[htbp]
    \centering
    \begin{tabular}{m{6.5cm}|m{8.2cm}}
    \toprule
        \centering\begin{tikzpicture}{object/.style={thin,double,<->}}
            \node[nodelbl, label=above:{\tiny{$C(A)=1$}}] (A) at (2/2,2) {\tiny{$A$}};
            \node[nodelbl, label=above:{\tiny{$C(B)=1$}}] (B) at (6/2+0.6,2) {\tiny{$B$}};
            \node[nodelbl, label=left:{\tiny{$C(C)=1$}}] (C) at (1/2,1) {\tiny{$C$}};
            \node[nodelbl, label=above right:{\tiny{$C(D)=2$}}] (D) at (3/2,1) {\tiny{$D$}};
            \node[nodelbl, label=below left:{\tiny{$C(E)=3$}}] (E) at (5/2+0.6,1) {\tiny{$E$}};
            \node[nodelbl, label=right:{\tiny{$C(F)=4$}}] (F) at (7/2+0.6,1) {\tiny{$F$}};
            \node[nodelbl, label=below:{\tiny{$C(G)=3$}}] (G) at (2/2,0) {\tiny{$G$}};
            \node[nodelbl, label=below:{\tiny{$C(H)=10$}}] (H) at (6/2+0.6,0) {\tiny{$H$}};

            \foreach \edge in {
                (A) -- (B),
                (A) -- (C),
                (C) -- (D),
                (D) -- (E),
                (D) -- (G),
                (E) -- (F),
                (F) -- (H),
            } {
                \StandardPath \edge;
            }

            \foreach \edge in {
                (A) -- (D),
                (B) -- (E),
                (B) -- (F),
                (C) -- (G),
                (E) -- (H),
                (G) -- (H),
            } {
                \HighlightPath \edge;
            }

            \node at (2.3, -1) {$G$};
        \end{tikzpicture} 
    &
        \begin{tikzpicture}
            \MakeTriplet A {2,2}
            \MakeTriplet B {6,2}
            
            \MakeTriplet C {1,1}
            \MakeTriplet D {3,1}
            \MakeTriplet E {5,1}
            \MakeTriplet F {7,1}
    
            \MakeTriplet G {2,0}
            \MakeTriplet H {6,0}
    
            \foreach \edge in {
                (A2) to[bend left=25] (B3),
                (A2) -- (C3),
                (E1) to[bend right=25] (F2),
                (F2) -- (H3),
                (A1) -- (D2),
                (C1) to[bend right=30] (D2),
                (D2) -- (G3),
                (D2) to[bend left=30] (E3),
                (D2) to[bend right=25] (F3),
                (D2) to[bend right=50] (H3),
            } {
                \StandardPath \edge;
            }
    
            \foreach \edge in {
                (A3) -- (D1),
                (C3) -- (G1),
                (E3) -- (H1),
                (G3) to[bend right=33] (H1),
                (B3) -- (E1),
                (B3) -- (F1),
            } {
                \HighlightPath \edge;
            }  
            \node at (4, -1) {$G^\dagger$};
        \end{tikzpicture} 
    \\ \midrule
        \centering\begin{tikzpicture}
            \node[nodelbl] (A) at (2/2,2) {\tiny{$A$}};
            \node[nodelbl] (B) at (6/2+0.6,2) {\tiny{$B$}};
            \node[nodelbl] (C) at (1/2,1) {\tiny{$C$}};
            \node[nodelblb] (D) at (3/2,1) {\textcolor{white}{\tiny{$D$}}};
            \node[nodelbl] (E) at (5/2+0.6,1) {\tiny{$E$}};
            \node[nodelbl] (F) at (7/2+0.6,1) {\tiny{$F$}};
            \node[nodelbl] (G) at (2/2,0) {\tiny{$G$}};
            \node[nodelbl] (H) at (6/2+0.6,0) {\tiny{$H$}};

            \foreach \edge in {
                (A) -- (B),
                (A) -- (C),
                (C) -- (D),
                (D) -- (E),
                (D) -- (G),
                (E) -- (F),
                (F) -- (H),
            } {
                \StandardPath \edge;
            }
            \node at (2.3, -1) {$S = S_D$};
        \end{tikzpicture} 
    &
    \begin{tikzpicture}
            \MakeTriplet A {2,2}
            \MakeTriplet B {6,2}
            
            \MakeTriplet C {1,1}
            \MakeTriplet D {3,1}
            \MakeTriplet E {5,1}
            \MakeTriplet F {7,1}
    
            \MakeTriplet G {2,0}
            \MakeTriplet H {6,0}
    
            \foreach \edge in {
                (A2) to[bend left=30] (B3),
                (A2) -- (C3),
                (E1) to[bend right=30] (F2),
                (F2) -- (H3),
            } {
                \StandardPath \edge;
            }
    
            \foreach \edge in {
                (A1) -- (D2),
                (C1) to[bend right=30] (D2),
                (D2) -- (G3),
                (D2) to[bend left=30] (E3),
                (D2) to[bend right=25] (F3),
                (D2) to[bend right=50] (H3),
            } {
                \HighlightPath \edge;
            } 
            \node at (4, -1) {$S^\dagger = S^\dagger_D$};
        \end{tikzpicture}
    \\ \midrule
        \centering\begin{tikzpicture}
            \node[nodelblb] (A) at (2/2,2) {\textcolor{white}{\tiny{$A$}}};
            \node[nodelbl] (B) at (6/2+0.6,2) {\tiny{$B$}};
            \node[nodelbl] (C) at (1/2,1) {\tiny{$C$}};
            \node[nodelbl] (E) at (5/2+0.6,1) {\tiny{$E$}};
            \node[nodelblb] (F) at (7/2+0.6,1) {\textcolor{white}{\tiny{$F$}}};
            \node[nodelblb] (G) at (2/2,0){\textcolor{white}{\tiny{$G$}}};
            \node[nodelbl] (H) at (6/2+0.6,0) {\tiny{$H$}};

            \foreach \edge in {
                (A) -- (B),
                (A) -- (C),
                (E) -- (F),
                (F) -- (H),
            } {
                \StandardPath \edge;
            }
            \node at (2.3, -1) {$S_A \cup S_G \cup S_F$};
        \end{tikzpicture}
    &
        \begin{tikzpicture}
            \MakeTriplet A {2,2}
            \MakeTriplet B {6,2}
            
            \MakeTriplet C {1,1}
            \MakeTriplet E {5,1}
            \MakeTriplet F {7,1}
    
            \MakeTriplet G {2,0}
            \MakeTriplet H {6,0}

            \foreach \edge in {
            } {
                \StandardPath \edge;
            }
    
            \foreach \edge in {
                (A2) to[bend left=30] (B3),
                (A2) -- (C3),
                (E1) to[bend right=30] (F2),
                (F2) -- (H3),
            } {
                \HighlightPath \edge;
            } 
            \node at (4, -1) {$S_A^\dagger \cup S_G^\dagger \cup S_F^\dagger$};
        \end{tikzpicture}
    \\ \bottomrule
    \end{tabular}
    \caption{An example graph conversion from $G$ to $G^\dagger$ is shown. The non-tree edges $E \setminus E^\clubsuit$ are shaded in $G$, and they correspond to the shaded edges in $G^\dagger$. The graph $S = (V, E^\clubsuit)$ has a centroid vertex $D$. Removing $D$ from $S$ results in three subtrees: $S_A$, $S_G$, and $S_F$. The new linked edges for the corresponding centroids are shaded in the graphs on the right. Recall $C(\cdot)$ is the number of distinct path from source to the vertex.}
    \label{fig:enter-label}
\end{figure}
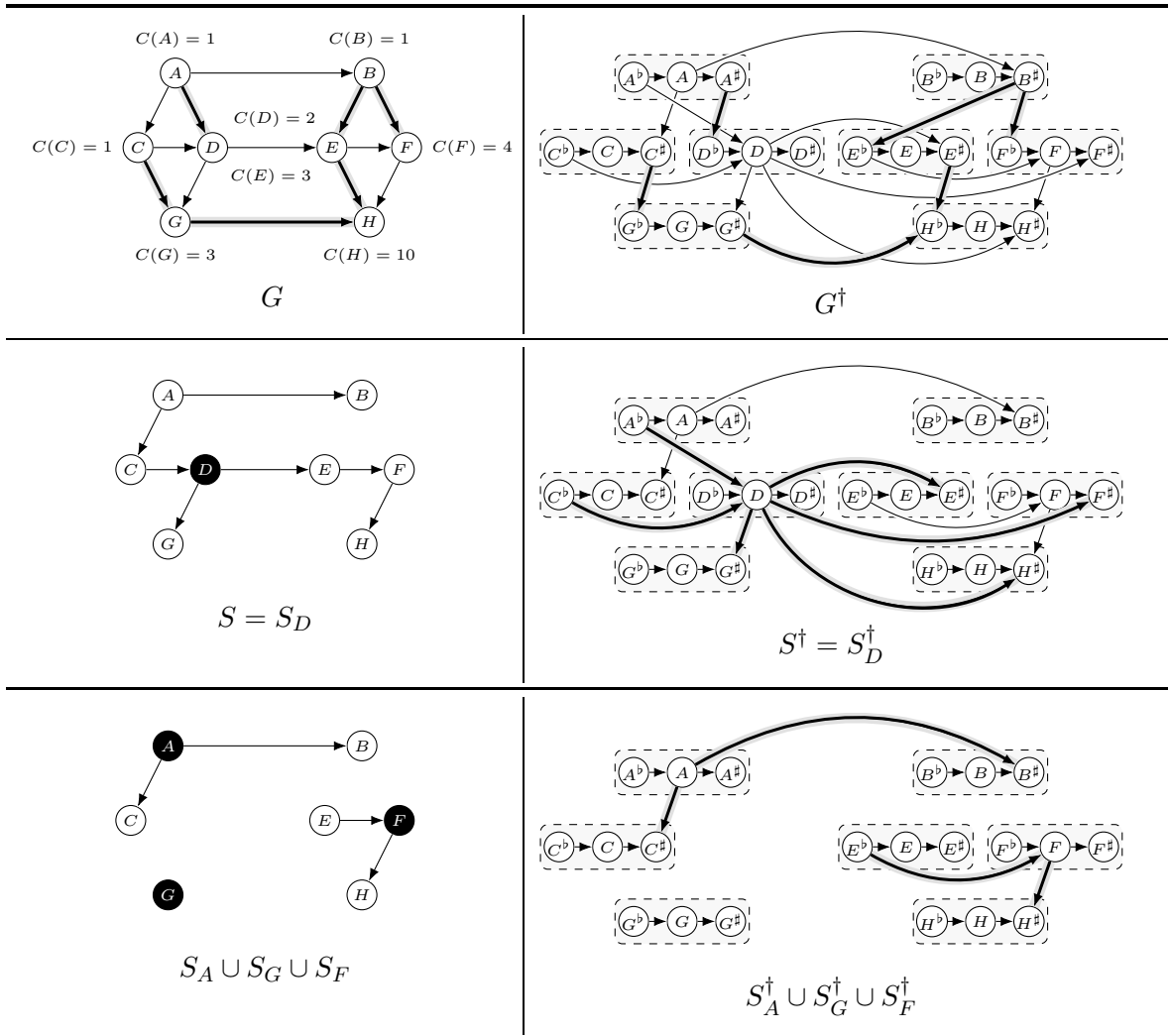

Starting from the tree $S=(V, E^\clubsuit)$, we start by transforming a selected tree $S_c = (V_c, E^\clubsuit_c)$ with centroid vertex $c$ into an equivalent graph $S^\dagger_c = (V^\dagger_c, E^\dagger_c)$ in a recursive way, using the centroid-based decomposition. Let $\hat{S}_1, \dots, \hat{S}_k$ be the subtrees obtained after removing the centroid $c$ from the tree $S_c$. Now, we do the following in a recursive way:
\begin{enumerate}
    \item Initialize $S^\dagger_c \leftarrow \bigcup_{i=1}^k \hat{S}^\dagger_i$, where $\hat S^\dagger_i$ is the transformed graph of $\hat{S}_i$.
    \item Update $V^\dagger_c \leftarrow V^\dagger_c \cup \{c^\flat, c, c^\sharp\}$.
    \item For each vertex $v \in V_c$:
    \begin{enumerate}
        \item If there is a directed path from $v$ to $c$ in $S_c$, or if $v$ is $c$, update $E^\dagger_c \leftarrow E^\dagger_c \cup \{(v^\flat, c)\}$.
        \item If there is a directed path from $c$ to $v$ in $S_c$, or if $v$ is $c$, update $E^\dagger_c \leftarrow E^\dagger_c \cup \{(c, v^\sharp)\}$.
    \end{enumerate}
\end{enumerate}
Finally, the graph $G^\dagger = (V^\dagger, E^\dagger)$ is generated as follows:
\begin{enumerate}
    \item Initialize $G^\dagger \leftarrow S^\dagger$.
    \item For each non-tree edge $(u, v) \in E \setminus E^\clubsuit$, update $E^\dagger \leftarrow E^\dagger \cup \{(u^\sharp, v^\flat)\}$.
\end{enumerate}
We refer the reader to \Cref{fig:enter-label} for one example of such a conversion.
We now demonstrate that the converted graph $G^\dagger$ is essentially equivalent to $G$. We define a mapping $\sigma: E^\dagger \to 2^E$ as follows:
\begin{itemize}
    \item For $e^\dagger = (v^\flat, c)$, $\sigma(e^\dagger)$ consists of all edges on the unique path from $v$ to $c$ in the tree $S$.
    \item For $e^\dagger = (c, v^\sharp)$, $\sigma(e^\dagger)$ consists of all edges on the unique path from $c$ to $v$ in the tree $S$.
    \item For $e^\dagger = (u^\sharp, v^\flat)$, $\sigma(e^\dagger) = \{(u, v)\} \subseteq E \setminus E^\clubsuit$ contains the corresponding edge.
\end{itemize}
The above mapping assigns each edge $e^\dagger = (u^\dagger, v^\dagger) \in E^\dagger$ a path from $u$ to $v$ (which may be empty), as specified by $\sigma(e^\dagger)$, where $w^\dagger \in \{w^\flat, w, w^\sharp\}$ for $w \in \{u, v\}$. 
Denote by $\calP^\dagger$ the set of paths from $\source^\flat$ to $\sink^\sharp$ in $G^\dagger$.
The following lemma establishes an important property of $\sigma(e^\dagger)$.
\restateLemma{lem:me1capme2-main}
\begin{proof}
    Consider a path $P^\dagger$ in $G^\dagger$. Suppose, for the sake of contradiction, that there exist two distinct edges $e_1^\dagger$ and $e_2^\dagger$ in $P^\dagger$ such that $\sigma(e_1^\dagger) \cap \sigma(e_2^\dagger) \neq \emptyset$. Let $(u, v)$ be an edge in the intersection. This implies that there is a sub-path in $G$ that starts at $u$, passes through $v$, and ends at $u$, contradicting our assumption that $G$ is a DAG.
\end{proof}

The next lemma shows that $G^\dagger$ is a DAG:
\begin{lemma}\label{gnew:dag}
    The graph $G^\dagger$ is a Directed Acyclic Graph with source node $\source^\flat$ and sink node $\sink^\sharp$.
\end{lemma}
\begin{proof}
First, we show that $G^\dagger$ is a directed acyclic graph. Note that by construction, there is no directed cycle involving only the three vertices $c^\flat, c, c^\sharp$ for any $c\in V$. Suppose, for the sake of contradiction, that there exists a path $P^\dagger = (v_0^\dagger, e_1^\dagger, v_1^\dagger, \dots, v_{\ell-1}^\dagger, e_\ell^\dagger, v_\ell^\dagger)$ in $G^\dagger$ such that $v_0^\dagger = v_\ell^\dagger$ and $w^\dagger\in \{w^\flat,w,w^\sharp\}$ for all $w\in V$. Let $i$ be the smallest index such that $v_0 \neq v_i$. This implies that there is a path in $G$ from $v_0$ to $v_i$ using the edges $\bigcup_{j \in \llbracket  i \rrbracket} \sigma(e_j^\dagger)$, and a path from $v_i$ back to $v_0$ using the edges $\bigcup_{j \in \llbracket i+1, \ell \rrbracket} \sigma(e_j^\dagger)$, which contradicts our assumption that $G$ is a DAG.

By construction, there is no incoming edges to $\source^\flat$ in $G^\dagger$ as there are no incoming edges to $\source$ in $G^\dagger$. Similarly, there is no outgoing edges from $\sink^\flat$ in $G^\dagger$ as there are no outgoing edges from $\sink$ in $G^\dagger$.
\end{proof}

The next lemma demonstrates that this mapping establishes a bijection between the paths from $\source$ to $\sink$ in $G$ and the paths $P^\dagger$ from $\source^\flat$ to $\sink^\sharp$ in $G^\dagger$. Note that we slightly abuse the notation $\sigma$.
\restateLemma{gnew:paths-main}
\begin{proof}
    First, observe that for a path $P^\dagger$ in $G^\dagger$ from $\source^\flat$ to $\sink^\sharp$, the set of edges $\bigcup_{e^\dagger \in P^\dagger} \sigma(e^\dagger)$ forms a path $P$ from $\source$ to $\sink$, where the union is over all the edges in $P^\dagger$. This is because any edge $(u^\dagger,v^\dagger)\in E^\dagger$ corresponds to a path from $u$ to $v$ in $G$. As $\sigma(e^\dagger)$ can be computed efficiently for any edge $e^\dagger$, $\sigma(P^\dagger)$ can be computed efficiently for any path $P^\dagger$. Let $g$ denote this mapping from $P^\dagger$ to $P$. We now show that the mapping $g$ is a bijection.

    Consider a path $P = (v_0, e_1, v_1, \dots, e_k, v_k)$ in $G$, where $v_0 = \source$ and $v_k = \sink$. Let us assume that there is at least one non-tree edges. An analogous proof exists if there are no non-tree edges. Let $e_{i_1}, e_{i_2}, \dots, e_{i_t}$ be the sequence of non-tree edges in the path, that is, $e_{i_j} \in E \setminus E^\clubsuit$ for each $j \in \llbracket t \rrbracket$. These edges partition the path $P$ into several segments:
    \[
        (v_0, e_1, \dots, v_{i_1 - 1}), e_{i_1}, (v_{i_1}, e_{i_1+1}, \dots, v_{i_2 - 1}), \dots, e_{i_t}, (v_{i_t}, e_{i_t+1}, \dots, v_k),
    \]
    where each segment $(v_{i_j}, e_{i_j+1}, \dots, v_{i_{j+1} - 1})$ is a path in $G$ that consists only of edges from the directed tree $S = (V, E^\clubsuit)$. 
    
    Let $i_0=0$ and $i_{t+1}=k+1$. By \Cref{lm:centroid-decomposition-uniqueness}, for any $j\in \llbracket 0,t\rrbracket$, there exists a unique subtree $S_{c_j} \in \calT$ with centroid $c_j$ such that $S_{c_j}$ contains the path from $v_{i_j}$ to $v_{i_{j+1}-1}$, and this path contains $c_j$. By the construction of the graph $G^\dagger$, the edges
    \[
    e_j^\flat := (v^\flat_{i_j}, c_j) \in E^\dagger \quad \text{and} \quad e_j^\sharp := (c_j, v^\sharp_{i_{j+1}-1}) \in E^\dagger
    \]
    are present in $G^\dagger$. Furthermore, for any $j\in \llbracket t\rrbracket$, the graph $G^\dagger$ also contains the edge $e_{i_j}^\dagger := (v^\sharp_{i_{j}-1}, v^\flat_{i_j}) \in E^\dagger$, since $e_{i_j} = (v_{i_j-1}, v_{i_j}) \in E \setminus E^\clubsuit$ is a non-tree edge. As a result, $G^\dagger$ contains the following path from $v_{i_0}^\flat=\source^\flat$ to $v_{i_{t+1}-1}^\sharp=\sink^\sharp$ : 
    \[
        P^\dagger := (v_{i_0}^\flat, e_{i_0}^\flat, c_{i_0}, e_{i_0}^\sharp, v_{i_1-1}^\sharp, e_{i_1}^\dagger, v_{i_1}^\flat, \dots, e_{i_t}^\sharp, v_{i_{t+1}-1}^\sharp).
    \]

\end{proof}

The previous lemma shows that the decision problem for the shortest path in $G$ can be converted to the shortest path problem in $G^\dagger$, and vice versa. Let $w: E \to \bbR$ be a weight function in the graph $G = (V, E)$. Define $w^\dagger: E^\dagger \to \bbR$ as the weight function for the converted graph $G^\dagger = (V^\dagger, E^\dagger)$:
\begin{align}\label{eq:weight-conversion}
    w^\dagger(e^\dagger) := \ind[|\sigma(e^\dagger)|\geq 1]\cdot\sum_{e \in \sigma(e^\dagger)} w(e).
\end{align}
Using this mapping, we can convert a decision problem on $G$ to a decision problem on $G^\dagger$:
\restateLemma{gnew:conversion-main}
The proof of the above lemma can be found in the main body.
Finally, we need to show that the graph $G^\dagger$ satisfies the required size constraints:
\restateLemma{gnew:size-main}
\begin{proof}
    Each vertex $v \in V$ corresponds to three vertices $v^\flat$, $v$, and $v^\sharp$ in the graph $G^\dagger$. Thus, $G^\dagger$ contains a total of $3|V|$ vertices. Moreover, for each node $c \in V$, we add at most $|V_c| + 1$ edges of the form $(v^\flat, c)$ or $(c, v^\sharp)$. For each non-tree edge $e \in E \setminus E^\clubsuit$, we add one additional edge to $E^\dagger$. Therefore, the total number of edges can be bounded by:
    \[
        |E^\dagger| \leq \sum_{c \in V} (|V_c| + 1) + |E \setminus E^\ddagger| \leq |V| \log |V| + 2|V| + |E|,
    \]
    where the last inequality follows from Lemma~\ref{lm:centroid_total_size}.

    Recall from the proof of Lemma \ref{gnew:paths-main} that any path $P^\dagger$ can be represented as 
    \[
        P^\dagger := (v_0^\flat, e_0^\flat, c_0, e_0^\sharp, v_{i_1-1}^\sharp, e_{i_1}^\dagger, v_{i_1}^\flat, \dots, e_t^\sharp, v_{i_t}^\sharp),
    \]
    where $e_{i_j}^\dagger$ corresponds to some non-tree edge $e_{i_j} \in E \setminus E^\clubsuit$. According to Lemma \ref{lm:centroid-decomposition-logX}, we have $t \leq \log |\mathcal{X}|$. Since there are exactly $3t + 2$ edges in $P^\dagger$, the longest path in $G^\dagger$ is upper bounded by $\calO(\log |\calX|)$.
\end{proof}

By combining Lemma~\ref{gnew:conversion-main}, Lemma~\ref{gnew:size-main}, and applying the regret guarantee of our FTRL algorithm from the previous section, we establish the main theorem:
\restateTheorem{actual:thm}

\input{appendix-applications}

\section{Multi-Task MAB: Additional Details}
\subsection{Multi-Task MAB Lower Bound}\label{appendix:multi-MAB-lower-bound}
Consider the Multi-task MAB instance where the set of arms $\calX$ is defined as follows: 
\begin{equation*}
    \calX=\left\{\bfx\in \{0,1\}^d: \forall j\in \llbracket  m \rrbracket \sum_{i=d_{1:j-1}+1}^{d_{1:j}} \bfx[i]=1\right\}
\end{equation*}
where $d_i\geq 2$, $d=\sum_{i=1}^md_i$, $d_{1:j}=\sum_{i=1}^jd_i$ and $d_{1:0}=0$ for all $j\in \llbracket m\rrbracket$.

Let us fix one regret minimizing algorithm, say $\calA$ and assume that $\calA$ is deterministic. Now we show that algorithm $\calA$ incurs a regret of $\Omega(\sum_{i=1}^m\sqrt{d_iT})$. We later extend the result to randomized algorithms using Yao's lemma. For all $j\in\llbracket m\rrbracket$, let $\varepsilon_j>0$ be a parameter that we fix later in the proof.

Let $\tilde\calX=\left\{\bfx\in \{0,1\}^d: \forall j\in \llbracket  m \rrbracket \sum_{i=d_{1:j-1}+1}^{d_{1:j}} \bfx[i]\leq 1\right\}$.
First we describe an instance $I_{\tilde \bfx}$, where $\tilde \bfx\in \tilde \calX$. In this instance, in each round $t$, we choose a loss function $\bfy_t:\llbracket d\rrbracket\to [-1,1]$ as follows. First we choose an index $j\in \llbracket m\rrbracket$ uniformly at random. Now we define $\bfy_t[i]=0$ if $i\notin \llbracket d_{1:j-1}+1,d_{1:j}\rrbracket$. Next for all $i\in \llbracket d_{j}\rrbracket$, we sample $v_i\sim \mathrm{Ber}(\frac{1}{2}-\varepsilon_j\cdot\ind[\tilde \bfx[d_{1:j-1}+i]=1])$ and assign $\bfy_t[d_{1:j-1}+i]=v_i$. Now observe that expected loss of any arm $\bfx\in\calX$ under the instance $I_{\tilde \bfx}$ is $\frac{1}{m}\sum\limits_{j=1}^m\sum\limits_{i=d_{1:j-1}+1}^{d_{1:j}}(\frac{1}{2}-\varepsilon_j\cdot\ind[\bfx[i]=\tilde \bfx[i]=1])$. For $\bfx\in \calX$, $\bfx$ is the best arm for the instance $I_{\bfx}$ and its reward is $\mu^*:=\frac{1}{2}-\frac{1}{m}\sum\limits_{j=1}^m\varepsilon_j$.
    
    In each round $t$, if $\calA$ chooses an arm $\bfx_t\in \calX$, then the regret $R_T(\bfx)$ on an instance $I_{\bfx}$, where $\bfx\in \calX$, is equal to:
    \[R_T(\bfx)=\mathbb{E}_{I_{\bfx}}\Bigg[\sum_{t=1}^T\langle \bfx_t,\bfy_t\rangle\Bigg]-T\cdot \mu^*=\frac{1}{m}\sum_{j=1}^m\varepsilon_j T-\frac{\varepsilon_j}{m}\sum_{t=1}^T\sum\limits_{j=1}^m\sum\limits_{i=d_{1:j-1}+1}^{d_{1:j}}\mathbb{P}_{I_{\bfx}}[\bfx_t[i]=\bfx[i]=1]\]
    where $\mathbb{P}_{I_{\bfx}}$ is probability law under the instance $I_{\bfx}$.

    Now for any instance $I_{\bfx}$ such that $\bfx\in \calX$, the regret can be broken down as $R_T(\bfx)=\sum_{j=1}^{m} R_T^{(j)}(\bfx)$ where 
    \[R_T^{(j)}(\bfx):=\frac{\varepsilon_j T}{m}-\frac{\varepsilon_j}{m}\sum_{t=1}^T\sum\limits_{i=d_{1:j-1}+1}^{d_{1:j}}\mathbb{P}_{I_{\bfx}}[\bfx_t[i]=1,\bfx[i]=1]\] 
    
    Let $\calI=\bigcup_{\bfx\in\calX}I_{\bfx}$. Fix an index $j\in\llbracket m\rrbracket$. We now show that $ \mathbb{E}_{I_{\bfx'}\sim \mathrm{Unif}(\calI)}[R_T^{(j)}(\bfx')]\geq c\sqrt{d_jT}$ where $c$ is some absolute constant. Let $$\calX^{(j)}:=\Bigg\{\bfx\in \{0,1\}^d: \forall i\in\llbracket m\rrbracket\setminus\{j\}\; \sum_{s=d_{1:i-1}+1}^{d_{1:i}}\bfx[s]=1,\; \sum_{s=d_{1:j-1}+1}^{d_{1:j}}\bfx[s]=0\Bigg\}.$$ For any $\bfx\in \calX^{(j)}$, let $\bfx^{(i)}$ be the vector in $\calX$ such that $\bfx^{(i)}[s]=\bfx[s]$ for all $s\notin \llbracket d_{1:j-1}+1,d_{1:j}\rrbracket$ and $\bfx^{(i)}[d_{1:j-1}+i]=1$.

    First, we consider the case where $d_j\geq 48$. Let us fix $\bfx\in \calX^{(j)}$. Now we claim that there is a set $\calS_{\bfx}\subseteq\llbracket d_j\rrbracket$ with at least $d_j/3$ indices such that for each $i\in \calS_\bfx$, we have $R_T^{(j)}(\bfx^{(i)})\geq c_0\sqrt{d_jT}$ where $c_0$ is some absolute constant.
    
    Before we prove our claim, we first show that if our claim holds true, then we have that $\mathbb{E}_{I_{\bfx'}\sim \mathrm{Unif}(\calI)}[R_T^{(j)}(\bfx')]\geq c\cdot \sqrt{KT}$ where $c$ is some absolute constant. Now we have the following:
    \begin{align*}
        \mathbb{E}_{I_{\bfx'}\sim \mathrm{Unif}(\calI)}[R_T^{(j)}(\bfx')]&=\frac{1}{\prod_{s=1}^md_s}\sum_{\bfx\in\calX^{(j)}}\sum_{i=1}^{d_j} R_T^{(j)}(\bfx^{(i)})\\
        &\geq \frac{1}{\prod_{s=1}^md_s}\sum_{\bfx\in\calX^{(j)}}\sum_{i\in \calS_{\bfx}}R_T^{(j)}(\bfx^{(i)})\\
        &\geq \frac{c_0}{\prod_{s=1}^md_s}\sum_{\bfx\in\calX^{(j)}}\sum_{i\in \calS_{\bfx}}\sqrt{d_jT}\\
        &\geq \frac{c_0}{\prod_{s=1}^md_s}\sum_{\bfx\in\calX^{(j)}}\frac{d_j}{3}\cdot \sqrt{d_jT} \\
        & = \frac{c_0}{\prod_{s=1}^md_s} \cdot \prod_{s\neq j}d_s\cdot \frac{d_j}{3}\cdot \sqrt{d_jT} \\
        & =\frac{c\sqrt{d_jT}}{3}
    \end{align*}

    Now we prove our claim. We use the following version of the chain rule in our analysis.
    \begin{lemma}[Chain Rule]
        Let $f(x_1,x_2,\ldots,x_n)$ and $g(x_1,x_2,\ldots,x_n)$ be two joint PMFs for a tuple of random variables $(X_i)_{i\in[n]}$. Let the sample space be $\Omega= \{0,1\}^{n}$. Then we have the following:
        \begin{equation*}
            \KL(f,g)=\sum\limits_{\omega\in \Omega}f(\omega)\left(\KL(f(X_1),g(X_1))+\sum_{i=2}^n \KL(f(X_i|X_{-i}=\omega_{-i}),g(X_i|X_{-i}=\omega_{-i}))\right)\;
        \end{equation*}
        where $X_{-i}=(X_1,\ldots,X_{i-1})$, $\omega_{-i}=(\omega_1,\ldots,\omega_{i-1})$.
    \end{lemma}

    For an instance $I_{\bfx^{(i)}}$, let $f_i(\ell_1,\ldots,\ell_T)$ denote the joint PMF for the tuple of loss values observed by $\calA$ in each round under the probability law $\mathbb{P}_{I_{\bfx^{(i)}}}$. Observe that our sample space is $\Omega=\{0,1\}^T$. This is a valid sample space as $\calA$ is deterministic and the probability of it seeing a loss value of $1$ in round $t$ only depends on the loss values it observed in the previous rounds. Similarly for the alternate instance  $I_{\bfx}$, let $f_0(\ell_1,\ldots,\ell_T)$ denote the joint PMF for the tuple of loss values observed by $\calA$ in each round under the probability law $\mathbb{P}_{I_{\bfx}}$.

    First observe that the instances $I_{\bfx^{(i)}}$ and $I_{\bfx}$ only differ at index $d_{1:j-1}+i$. For each $\omega\in \Omega$, let $\bfx_{1,\omega},\bfx_{2,\omega},\ldots, \bfx_{T,\omega}$ be the sequence of arms chosen by $\calA$ on $\omega$.
    Conditioning on a set of outcomes $X_1=\omega_1,X_2=\omega_2,\ldots,X_{t-1}=\omega_{t-1}$, we have $X_t\sim \mathrm{Ber}(\mu_i)$ for the instance $I_{\bfx^{(i)}}$ and $X_t\sim \mathrm{Ber}(\mu_0)$ for the instance $I_{\bfx}$ where $\mu_0-\mu_i=\frac{\varepsilon_j}{m}\cdot \bfx_{t,\omega}[d_{1:j-1}+i]$. Let $T_i=\sum_{t=1}^T\mathbf{x}_t[d_{1:j-1}+i]$.  For each $\omega\in \Omega$, let $T_{i,\omega}=\sum_{t=1}^T\mathbf{x}_{t,\omega}[d_{1:j-1}+i]$. Note that $T_i$ is a random variable and $T_{i,\omega}$ is a fixed value. Now we have the following:
    \begin{align*}
         \KL(f_0,f_i)&=\sum\limits_{\omega\in \Omega}f_0(\omega)\left(\KL(f_0(X_1),f_i(X_1))+\sum_{t=2}^T \KL(f_0(X_t|X_{-t}=\omega_{-t}),f_i(X_t|X_{-t}=\omega_{-t}))\right)\\
         &\leq\frac{4\varepsilon_j^2}{m^2} \sum\limits_{\omega\in \Omega}f_0(\omega)\sum_{t=1}^T\bfx_{t,\omega}[d_{1:j-1}+i]\\
         &=\frac{4\varepsilon_j^2}{m^2} \sum\limits_{\omega\in \Omega}f_0(\omega)T_{i,\omega}\\
         &= \frac{4\varepsilon_j^2}{m^2}\cdot \mathbb{E}_{I_{\bfx}}[T_{i}]
    \end{align*}

    Now observe that $\sum_{i=1}^{d_j}\mathbb{E}_{I_{\bfx}}[T_i]=T$. Hence, there exists a set $\calS_\bfx\subseteq\llbracket d_j\rrbracket$ with at least $d_j/3$ indices such that for each $i\in \calS_x$, we have $\mathbb{E}_{I_{\bfx}}[T_i]\leq \frac{3T}{d_j}$. Fix $\varepsilon_j=\frac{m\cdot d_j^{1/2}}{10T^{1/2}}$. Now for each $i\in \calS_\bfx$, we have $\KL(f_0,f_i)\leq\frac{12\varepsilon_j^2T}{m^2d_j}=\frac{3}{25}$. 

    Fix $i\in \calS_x$. Let $A_i$ be the event that $T_i\leq \frac{12T}{d_j}$. Due to Markov's inequality, we have $\mathbb{P}_{I_{\bfx}}(A_i)\geq \frac{3}{4}$. Now due to Pinsker's inequality we have the following:
    \begin{align*}
        \mathbb{P}_{I_{\bfx^{(i)}}}(A_i)&\geq  \mathbb{P}_{I_{\bfx}}(A_i)-\sqrt{\frac{\KL(f_0,f_j)}{2}}\\
        &\geq \frac{3}{4}-\sqrt{\frac{3}{50}}\\
        &>\frac{1}{2}
    \end{align*}

    Using the above the inequality, we get $\mathbb{E}_{I_{\bfx^{(i)}}}[T_i]\leq T\cdot \mathbb{P}_{I_{\bfx^{(i)}}}(A_i^c)+\frac{12T}{d_j}\leq\frac{3T}{4}$ when $d_j\geq 48$.
     Now we have the following:
    \begin{align*}
        R_T^{(j)}(\bfx^{(i)})&=\frac{\varepsilon_j T}{m}-\frac{\varepsilon_j}{m}\sum_{t=1}^T\mathbb{P}_{I_{\bfx^{(i)}}}[\bfx_t[d_{1:j-1}+i]=1]\\
        &=\frac{\varepsilon_j T}{m}-\frac{\varepsilon_j}{m}\mathbb{E}_{I_{\bfx^{(i)}}}\Bigg[\sum_{t=1}^T\bfx_t[d_{1:j-1}+i]\Bigg]\\
        &=\frac{\varepsilon_j T}{m}-\frac{\varepsilon_j}{m}\mathbb{E}_{I_{\bfx^{(i)}}}[T_i]\\
        &\geq \frac{\varepsilon_j T}{m}-\frac{3\varepsilon_j T}{4m}\\
        &=\frac{\sqrt{d_jT}}{40}
    \end{align*}
    
    Next we look at the case when $d_j\leq 48$. For simplicity of presentation, let us assume that $d_j=2$. Our analysis can be easily extended to any constant between $2$ and $48$.

    For any $i\in \{1,2\}$ and $t\in [T]$, let $A_{i,t}$ be the event that $\bfx[d_{1:j-1}+i]=1$. Note that $A_{1,t}=A_{2,t}^c$. Fix $\bfx\in \calX^{(j)}$. Now we claim that there an index $i\in \{1,2\}$ such that $\mathbb{P}_{I_{\bfx^{(i)}}}(A_{i,t})<\frac{3}{4}$. For the sake of contradiction, let us assume that $\mathbb{P}_{I_{\bfx^{(i)}}}(A_{i,t})\geq\frac{3}{4}$ for all $i\in \{1,2\}$. Then we have $\mathbb{P}_{I_{\bfx^{(1)}}}(A_{1,t})-\mathbb{P}_{I_{\bfx^{(2)}}}(A_{1,t})>\frac{1}{2}$. 

    For an instance $I_{\bfx^{(i)}}$, let $f_i(\ell_1,\ldots,\ell_T)$ denote the joint PMF for the tuple of loss values observed by $\calA$ in each round under the probability law $\mathbb{P}_{I_{\bfx^{(i)}}}$. Our sample space is $\Omega=\{0,1\}^T$.

    First observe that the instances $I_{\bfx^{(1)}}$ and $I_{\bfx^{(2)}}$ only differ at the indices $d_{1:j-1}+1$ and $d_{1:j-1}+2$. For each $\omega\in \Omega$, let $\bfx_{1,\omega},\bfx_{2,\omega},\ldots, \bfx_{T,\omega}$ be the sequence of arms chosen by $\calA$ on $\omega$.
    Conditioning on a set of outcomes $X_1=\omega_1,X_2=\omega_2,\ldots,X_{t-1}=\omega_{t-1}$, we have $X_t\sim \mathrm{Ber}(\mu_1)$ for the instance $I_{\bfx^{(1)}}$ and $X_t\sim \mathrm{Ber}(\mu_2)$ for the instance $I_{\bfx^{(2)}}$ where $|\mu_1-\mu_2|=\frac{\varepsilon_j}{m}$. Now we have the following:
    \begin{align*}
         \KL(f_1,f_2)&=\sum\limits_{\omega\in \Omega}f_1(\omega)\left(\KL(f_1(X_1),f_2(X_1))+\sum_{t=2}^T \KL(f_1(X_t|X_{-t}=\omega_{-t}),f_2(X_t|X_{-t}=\omega_{-t}))\right)\\
         &\leq\frac{4\varepsilon_j^2T}{m^2} \sum\limits_{\omega\in \Omega}f_1(\omega)\\
         &=\frac{4\varepsilon_j^2T}{m^2}
    \end{align*}
Fix $\varepsilon_j=\frac{m}{4\sqrt{T}}$. Due to Pinsker's inequality we arrive at the following contradiction:
    \begin{align*}
        \mathbb{P}_{I_{\bfx^{(1)}}}(A_{1,t})-\mathbb{P}_{I_{\bfx^{(2)}}}(A_{1,t})&\leq \sqrt{\frac{\KL(f_1,f_2)}{2}}\\
        &\leq \sqrt{\frac{2\varepsilon_j^2T}{m^2}}\\
        &<\frac{1}{2}
    \end{align*}
Now we have the following:
    \begin{align*}
       R_T^{(j)}(\bfx^{(1)})+ R_T^{(j)}(\bfx^{(2)})&= \frac{2\varepsilon_j T}{m}-\frac{\varepsilon_j}{m}\sum_{t=1}^T\mathbb{P}_{I_{\bfx^{(1)}}}\big[\bfx_t[d_{1:j-1}+1]=1\big]+\mathbb{P}_{I_{\bfx^{(2)}}}\big[\bfx_t[d_{1:j-1}+2]=1\big]\\
       &= \frac{2\varepsilon_j T}{m}-\frac{\varepsilon_j}{m}\sum_{t=1}^T\mathbb{P}_{I_{\bfx^{(1)}}}[A_{1,t}]+\mathbb{P}_{I_{\bfx^{(2)}}}[A_{2,T}]\\
       &> \frac{2\varepsilon_j T}{m}-\frac{7\varepsilon_j}{4m} \tag{$\mathbb{P}_{I_{\bfx^{(1)}}}[A_{1,t}]+\mathbb{P}_{I_{\bfx^{(2)}}}[A_{2,T}]<\frac{7}{4}$}\\
       &= \frac{\varepsilon_j T}{4m}\\
       &= \frac{\sqrt{T}}{16}
    \end{align*}
Now we have the following:
    \begin{align*}
        \mathbb{E}_{I_{\bfx'}\sim \mathrm{Unif}(\calI)}[R_T^{(j)}(\bfx')]&=\frac{1}{\prod_{s=1}^md_s}\sum_{\bfx\in \calX^{(j)}}\sum_{i\in \{1,2\}} R_T^{(j)}(\bfx^{(i)})\\
        &\geq \frac{1}{16\prod_{s=1}^md_s}\sum_{\bfx\in \calX^{(j)}}\sqrt{T}\\
        &=\frac{1}{16\prod_{s=1}^md_s}\cdot \prod_{s\neq j}d_s\cdot \sqrt{T}\\
        & =\frac{\sqrt{T}}{32}
    \end{align*}
Hence, our claim holds and therefore we have 
$$\mathbb{E}_{I_{\bfx'}\sim \mathrm{Unif}(\calI)}[R_T(\bfx')]=\sum_{j=1}^m \mathbb{E}_{I_{\bfx'}\sim \mathrm{Unif}(\calI)}[R_T^{(j)}(\bfx')]\geq c'\sum_{j=1}^m\sqrt{d_jT}$$ 
where $c'$ is some absolute constant. Due to Yao's lemma we have that any randomized algorithm should also have a regret of at least $c''\sum_{j=1}^m\sqrt{d_jT}$  where $c''$ is some absolute constant.

\subsection{Lower Bound for EXP3 with Kiefer-Wolfowitz Exploration}\label{appendix:exp3-lower-bound}
For any set of arms $\calX\subseteq \{0,1\}^d$ such that the dimension $\calX$ is $\Theta(d)$, EXP3 with Kiefer-Wolfowitz exploration plays a fixed distribution $\pi$ over the set of arms with probability at least $\sqrt{\frac{d\log |\calX|}{cT}}$ where $c$ is an absolute constant.

Consider the Multi-task MAB instance with set of arms $\calX$ where $d_i=2$ for all $i\in \llbracket  m-1 \rrbracket$ and $d_m=m^2$ where $m\geq 2$. Recall that $d=\sum_{j=1}^md_j$ and $d_{1:i}=\sum_{j=1}^id_j$. There exists an index $i_\star\in \llbracket d_{1:m-1}+1,d_{m}\rrbracket$, such that $\sum_{\bfx\in \calX}\pi[\bfx[i_\star]]\leq \frac{1}{2}$. For all $t\in\llbracket T\rrbracket$, we choose a loss function $\bfy_t:\llbracket d\rrbracket\to[-1,1]$ such that $\bfy_t[i]=-1$ if $i=i_\star$ and $\bfy_t[i]=0$ otherwise. It is easy to observe that EXP3 with Kiefer-Wolfowitz exploration incurs an expected regret of at least $\sqrt{\frac{d\log |\calX|}{cT}}\cdot \frac{1}{2}$ in each round. Hence, EXP3 with Kiefer-Wolfowitz exploration incurs a regret of at least $\Omega(\sqrt{dT\log |\calX|})=\Omega(\sqrt{m^3T})$.

\subsection{A Simple, Efficient Algorithm for Multi-task MAB}\label{appendix:multi-mab-simple}
Recall that in the Multi-task MAB problem, we are given a set of $m$ multi-armed bandit (MAB) problems, where the $i$-th MAB problem has $d_i$ arms. In each round, we simultaneously choose one arm from each MAB problem and receive the sum of the losses of the chosen arms as the loss feedback. The objective is to minimize regret with respect to the best arm in each MAB problem in hindsight.

Consider the following algorithm. For each $i \in \llbracket m \rrbracket$, we independently execute EXP3-IX \cite{neu2015explore} on the $i$-th MAB problem. Note that for any MAB problem with $K$ arms and losses in $[-1,1]$, the version of EXP3-IX under consideration incurs a regret of at most $c\sqrt{KT\log(K/\delta')}$ with probability at least $1 - \delta'$, where $c$ is an absolute constant.

Let $\bfy_{t,i} : \llbracket d_i \rrbracket \to \mathbb{R}$ be the loss function for the arms in the $i$-th MAB. For each $i \in \llbracket m \rrbracket$, let $I_{t,i}$ be the arm selected by the EXP3-IX algorithm for the $i$-th MAB in round $t$. We choose these recommended arms and observe the total loss $\ell_t = \sum_{i=1}^{m} \bfy_{t,i}[I_{t,i}]$, which satisfies $\ell_t \in [-1,1]$. We then provide $\ell_t$ as the loss feedback to each EXP3-IX algorithm.

If we set $\delta'=\delta/m$, then with probability at least $1-\delta$, the regret incurred in the multi-task MAB problem is upper-bounded as follows:
\begin{align*}
    \mathrm{Regret}(T)&=\max_{(j_1,j_2,\ldots,j_m)\in [d_1]\times[d_2]\times\ldots\times [d_m]}\sum_{t=1}^T\sum_{i=1}^m\bfy_{t,i}[I_{t,i}]-\sum_{t=1}^T\sum_{i=1}^m\bfy_{t,i}[j_i]\\
    &=\sum_{i=1}^m\max_{j_i\in[d_i]}\sum_{t=1}^T\bfy_{t,i}[I_{t,i}]-\sum_{t=1}^T\bfy_{t,i}[j_i]\\
    &\leq c\cdot\sum_{i=1}^m \sqrt{d_iT\log(md_i/\delta)}
\end{align*}

We obtain the last inequality because the EXP3-IX algorithm running on the $i$-th MAB effectively operates on a bandit instance where the loss of the $j$-th arm is adaptively chosen as $\bfy_{t,i}[j] + \sum_{i' \neq i} \bfy_{t,i'}[I_{t,i'}]\in [-1,1]$. Consequently, with probability at least $1-\delta/m$, we have: 
\begin{align*}
&\quad\max_{j_i \in [d_i]} \sum_{t=1}^T \bfy_{t,i}[I_{t,i}] - \sum_{t=1}^T \bfy_{t,i}[j_i]\\
&= \max_{j_i \in [d_i]} \sum_{t=1}^T \left(\bfy_{t,i}[I_{t,i}] + \sum_{i' \neq i} \bfy_{t,i'}[I_{t,i'}] \right) - \sum_{t=1}^T \left(\bfy_{t,i}[j_i] + \sum_{i' \neq i} \bfy_{t,i'}[I_{t,i'}] \right) \\
&\leq c \cdot \sqrt{d_i T \log (m d_i / \delta)}.
\end{align*}
We then obtain the high-probability regret guarantee by applying the union bound.

\textbf{Remark:} Prior to our work, \citet{zimmert2019beating} used a similar approach for Hypercube.

\subsection{Minimax Lower Bound for DAGs}\label{appendix:minimax-lower}
We prove the following theorem in this section.
\begin{theorem}\label{thm:minimax-lower-bound}
    Consider integers $d, N \geq 4$ satisfying $d \leq N \leq 2^{d/2}$. There exists a DAG $G$ with at most $d$ edges and at most $N$ paths from the source to the sink such that the regret is lower bounded by $\Omega\left( \sqrt{dT \log(N)/\log(d)} \right)$.
\end{theorem}
\begin{proof}
Consider the instance of the multi-task MAB problem where $m = \log(N)/\log(d)$ and $d_i = \frac{d}{2m}$. For simplicity, we assume that both $m$ and $\frac{d}{2m}$ are integers. As shown in \Cref{appendix:multi-MAB-lower-bound}, this instance has a regret lower bound of $\Omega\left( \sum_{i=1}^m \sqrt{d_i T} \right) = \Omega\left( \sqrt{d T \log(N)/\log(d)} \right)$.

Next, consider the reduction of this multi-task MAB problem to a directed acyclic graph (DAG) $G$, which is described in \Cref{appendix:application-multi-task}. First, note that for the graph $G$, the regret lower bound remains $\Omega\left( \sqrt{d T \log(N)/\log(d)} \right)$. Furthermore, the graph $G$ contains $d$ edges and has at most $\left( \frac{d}{2m} \right)^m \leq N$ paths from the source to the sink. This follows from the fact that $m \log \left( \frac{d}{2m} \right) \leq m \log(d) = \log(N)$. Hence, DAG $G$ is the required graph.
\end{proof}

\section{Extensive-Form Games: Additional Details}
\subsection{Linear Bandit Formulation of Extensive-Form Games}\label{appendix:extensive-linear}
First, we prove the following lemma.
\begin{lemma}\label{lem:leaf-non-leaf}
    Consider a tree such that each non-leaf node has at least 2 children. Then the number of leaf nodes in the tree at least the number of non-leaf nodes in the tree.
\end{lemma}
\begin{proof}
    Let $L_1$ be the number of leaf nodes in the tree and $L_2$ be the number of non-leaf nodes in the tree. Let $E$ be the total number of edges in the tree. Then we have $E=L_1+L_2-1$. Also, we have $E\geq 2L_2$ as each non-leaf node contributes to at least two edges as they have at least 2 children each. Hence, we have $L_1+L_2-1\geq 2L_2$ which implies that $L_1\geq L_2+1$.
\end{proof}

Now we describe the linear bandit formulation of Extensive-form games. For a configuration $\bfa=\{a_x\}_{\bfx\in\calX}$ of actions at the decision nodes, we describe a vector $s^{\bfa}\in\{0,1\}^{|\calX|+|\calY|+|\calZ|}$ indexed by the nodes in the game as follows:
\begin{align*}
    &s^{\bfa}[x^\mathsf{r}]=1\\
    &s^{\bfa}[x]=s^{\bfa}[\rho_x[a_x]] \quad \forall \bfx\in \calX\\
    &s^{\bfa}[\rho_x[a]]=0\quad \forall \bfx\in \calX,\forall a\in A_x\setminus\{a_x\}\\
    &s^{\bfa}[y]=s^{\bfa}[\rho_y[b_y]]\quad \forall y\in \calY, \forall b_y\in B_y\\
\end{align*}
Let $\calS_x$ be the set of all such vectors $s^{\bfa}$ corresponding to all possible configurations $\bfa$ of actions at the decision nodes. $\calS_x$ now is the set of arms in our linear bandit formulation. Recall the definition of $N$. It is easy to observe that $|\calS_x|=N$.

Next for a configuration $\bfb=\{b_y\}_{y\in\calY}$ of actions at the decision nodes, we describe a vector $s^{\bfb}\in\{0,1\}^{|\calX|+|\calY|+|\calZ|}$ indexed by the nodes in the game as follows:
\begin{align*}
    &s^{\bfb}[x^\mathsf{r}]=1\\
    &s^{\bfb}[y]=s^{\bfb}[\rho_y[b_y]] \quad \forall y\in \calY\\
    &s^{\bfb}[\rho_y[b]]=0\quad \forall y\in \calY,\forall b\in B_y\setminus\{b_y\}\\
    &s^{\bfb}[x]=s^{\bfb}[\rho_x[a_x]]\quad \forall \bfx\in \calX, \forall a_x\in A_x\\
\end{align*}
 
Next, for a loss function $\bfy_t:\calZ\to[-1,1]$ over the terminal nodes and a configuration $\bfb_t:=\{b_{y,t}\}_{y\in\calY}$ of actions at the observation nodes, we describe a loss function $\widehat\bfy_t:\calX\cup\calY\cup\calZ\to[-1,1]$ as follows. For all $v\in\calX\cup\calY$, we have $\widehat\bfy_t[v]:=0$. For all $z\in \calZ$, we have $\widehat \bfy_t[z]=\bfy_t[z]\cdot\ind[s^{\bfb_t}[z]=1]$. It is easy to observe that $\bfy_t[z(\bfa_t,\bfb_t)]=\langle s^{\bfa_t},\widehat\bfy_t\rangle$.

Let $\calT=(\calV,\calE)$ be the extensive-form game tree where $\calV=\calX\cup\calY\cup\calZ$ and $\calE=\{(x,\rho_x[a_x]):\bfx\in \calX,a_x\in A_x\}\cup\{(y,\rho_y[b_y]):y\in\calY,b_y\in B_y\}$. Due to Lemma \ref{lem:leaf-non-leaf}, we have $|\calX|+|\calY|\leq |\calZ|$. Hence, by using the EXP3 algorithm, we get an upper bound of $\calO(\sqrt{|\calZ|T\log(N)})$.

\subsection{Additional Details on Reduction to DAG}\label{appendix:extensive-fact}
In this section, we show that the DAG $G=(V,E)$ that we constructed during the reduction from extensive-form games has $\calO(|\calZ|)$ nodes. Let $\calT=(\calV,\calE)$ be the extensive-form game tree where $\calV=\calX\cup\calY\cup\calZ$ and $\calE=\{(x,\rho_x[a_x]):\bfx\in \calX,a_x\in A_x\}\cup\{(y,\rho_y[b_y]):y\in\calY,b_y\in B_y\}$.  Due to \Cref{lem:leaf-non-leaf}, we have $|\calX|+|\calY|\leq |\calZ|$. Hence, there are $\calO(|\calZ|)$ edges in $\calT$. 

Recall the construction of $G$. Each terminal node $z$ is associated with the edge $(z_{\mathsf{s}},z_{\mathsf{t}})$. Next each edge of the type $(x,\rho_x[a])$ is associated with the edges $(x_{\mathsf{s}},u_{\mathsf{s}})$ and $(u_{\mathsf{t}},x_{\mathsf{t}})$. Similarly, each edge of the type $(y,u^{(i)})$ is associated with the edges $(v_1,u^{(i)}_{\mathsf{s}})$ and $(u^{(i)}_{\mathsf{s}},v_2)$, where $u^{(i)}=\rho_y[b]$ for some $b\in B_y$, $v_1$ is either $y_{\mathsf{s}}$ or $u^{(i-1)}_{\mathsf{s}}$, and $v_2$ is either $y_{\mathsf{t}}$ or $u^{(i+1)}_{\mathsf{s}}$. Hence, $G$ has $\calO(|\calZ|)$ edges.

%% file: appendix-applications.tex
\section{Applications}\label{appendix:applications}
In this section, we demonstrate the application of our FTRL approach to various well-known combinatorial sets $\calX \subseteq \{0,1\}^d$. The core idea is to efficiently reduce problems involving these combinatorial sets to a problem on a directed acyclic graph (DAG) and establish the corresponding regret bound. Our method can be seen as a computationally efficient FTRL approach for these sets.  

In certain cases, we either match or improve upon the $\mathcal{O}(\sqrt{dT\log |\calX|})$ regret bound achieved by EXP3 with Kiefer-Wolfowitz exploration. In other cases, we demonstrate improvements over the best-known high-probability regret guarantees achieved by an efficient algorithm, specifically that of \cite{zimmert2022return}. We note that the high-probability regret guarantee was formally proven only for continuous sets by \citet{zimmert2022return}. However, we believe their analysis extends to discrete decision sets, such as the combinatorial sets considered, using the same techniques as \citet{abernethy2008competing}.

\subsection{Hypercube}
Let $\calX = \{0,1\}^d$ denote the combinatorial hypercube. We construct a DAG $G$ as follows. The vertex set and edge set are
$$V := \{v_0\} \cup \{v_i^\dagger, v_i\}_{i=1}^d, \quad E := \{(v_{i-1}, v_{i}), (v_{i-1}, v_{i}^\dagger), (v_{i}^\dagger, v_{i})\}_{i=1}^d$$
respectively. In the graph $G$, $\source = v_0$ is the source vertex, and $\sink = v_{d}$ is the sink vertex.

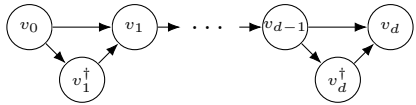
\begin{figure}[h]
    \centering
    \tikzstyle{nodelbl}=[draw,fill=white,circle,inner sep=0, minimum width=6mm]
    \begin{tikzpicture}{object/.style={thin,double,<->}}
        \node[nodelbl] (v0) at (0,0) {\tiny{$v_0$}};
        \node[nodelbl] (v11) at (0.7,-0.7) {\tiny{$v_{1}^\dagger$}};
        \node[nodelbl] (v1) at (1.4,0) {\tiny{$v_1$}};

        \node (zzz) at (2.4, 0.) {\dots};
        
        \node[nodelbl,label=center:{\tiny{$v_{d-1}$}}] (vd-1) at (3.4,0) {};
        \node[nodelbl] (vd1) at (4.1,-0.7) {\tiny{$v_{d}^\dagger$}};
        \node[nodelbl] (vd) at (4.8,0) {\tiny{$v_d$}};

        \foreach \edge in {
            (v0) -- (v1),
            (v0) -- (v11),
            (v11) -- (v1),
            (v1) -- (zzz),
            (zzz) -- (vd-1),
            (vd-1) -- (vd),
            (vd-1) -- (vd1),
            (vd1) -- (vd)
        } {
            \StandardPath \edge;
        }
    \end{tikzpicture} 
    \caption{Conversion of hypercube to DAG}
\end{figure}

Next, for any loss function $\bfy_t : \llbracket d \rrbracket \to \mathbb{R}$, we define a weight function $w_t: E \to \bbR$ as follows:
\begin{align*}
    w_t(e) = \begin{cases}
        \bfy_t[i] & \textrm{if } e=(v_{i-1},v_i^\dagger) \textrm{ for some } i\in\llbracket d\rrbracket \\
        0 & \textrm{otherwise}
    \end{cases}.
\end{align*}

We now apply our FTRL algorithm from \Cref{sec:alg-equal} to the DAG $G$. At each round $t$, if the FTRL algorithm selects a path $P_t$ in $G$, we choose $\bfx_t \in \calX$ such that for any $i \in \llbracket d \rrbracket$, $\bfx_t[i] = 1$ if the edge $(v_{i-1}, v_{i}^\dagger)$ is part of the path $P_t$, and $\bfx_t[i] = 0$ otherwise. By the construction of $w_t$, it follows that $\langle \bfx_t, \bfy_t \rangle = w_t(P_t)$. Consequently, we provide $\langle \bfx_t, \bfy_t \rangle$ as the bandit feedback for the path $P_t$ to the FTRL algorithm. The way we choose $\bfx_t$ induces a bijective mapping between the set of vectors in $\calX$ and the set of paths in $G$, ensuring correctness. Moreover, our algorithm is computationally efficient.

 Finally, observe that each path in $G$ has a length of $d$. Thus, we incur a high-probability regret of $\tilde{\calO}(d\sqrt{T})$ against an adaptive adversary, which is near-optimal for the hypercube. This also improves upon the best-known high-probability regret bound of $\tilde{\calO}(d^2\sqrt{T})$ achieved by \cite{zimmert2022return}.

\subsection{Multi-Task Multi-Armed Bandits}\label{appendix:application-multi-task}

In the Multi-task Multi-Armed Bandit problem, we are given a set of $m$ MAB problems, where in the $i$-th MAB problem there are $d_i$ arms. In each round, we choose one arm from each MAB problem simultaneously and receive the sum of the losses of the arms chosen as the loss feedback. The goal is to do regret minimization w.r.t best arm in each MAB problem in hindsight.

The multi-task MAB problem is formally formulated as follows. Let $d=\sum_{i=1}^m d_i$. Let $d_{1:i}=\sum_{j=1}^id_j$ and let $d_{1:0}=0$. The set $\calX$ of arms is defined as follows:
\begin{equation*}
    \calX=\left\{\bfx\in \{0,1\}^d: \forall j\in \llbracket m \rrbracket \sum_{i=d_{1:j-1}+1}^{d_{1:j}} \bfx[i]=1\right\}
\end{equation*}
For each round $t$, the loss function $\bfy_t:\llbracket d \rrbracket\to \mathbb{R}$ is chosen by an adversary. In each round the agent draw $\bfx_t\in \calX$ and observe loss $\langle \bfx_t,\bfy_t\rangle\in [-1,1]$. The goal is to minimize the following regret:
\begin{equation*}
    \mathrm{Regret}(T) := \sum_{t=1}^T \langle \bfx_t, \bfy_t \rangle - \min_{\bfx\in\calX}\sum_{t=1}^T \langle \bfx, \bfy_t \rangle
\end{equation*}

We now reduce the problem to online shortest path on DAG. We first construct a DAG $G$ as follow: The vertex set and edge set are
$$V := \{v_i\}_{i=0}^m \cup \{v_{i}^j \mid i \in \llbracket m \rrbracket, j \in \llbracket d_i \rrbracket \}, \qquad E := \{(v_{i-1},v_{i}^j), (v_{i}^j, v_{i}) \mid i \in \llbracket m \rrbracket, j \in \llbracket d_i \rrbracket \}$$
respectively. In graph $G$, $\source = v_0$ is the source vertex, and $\sink = v_m$ is the sink vertex.

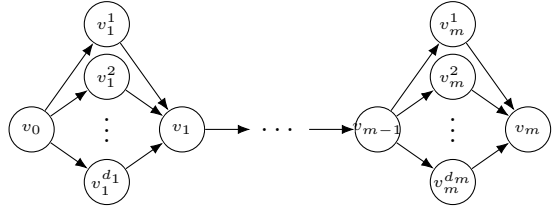
\begin{figure}[h]
    \centering
    \tikzstyle{nodelbl}=[draw,fill=white,circle,inner sep=-.5mm, minimum width=6mm]
    \begin{tikzpicture}{object/.style={thin,double,<->}}
        \node[nodelbl] (v0) at (0,0) {\tiny{$v_0$}};
        \node[nodelbl] (v11) at (1,1.4) {\tiny{$v^1_1$}};
        \node[nodelbl] (v12) at (1,0.7) {\tiny{$v^2_1$}};
        \node[nodelbl] (v1d) at (1,-0.7) {\tiny{$v^{d_1}_{1}$}};
        \node[yshift=0.25em] at ($(v12)!0.5!(v1d)$) {\vdots};
        \node[nodelbl] (v1) at (2,0) {\tiny{$v_1$}};

        \node (zzz) at (3.3, 0.) {\dots};
        
        \node[nodelbl,label=center:{\tiny{$v_{m-1}$}}] (vm-1) at (4.6,0) {};
        \node[nodelbl] (vm1) at (5.6,1.4) {\tiny{$v^1_m$}};
        \node[nodelbl] (vm2) at (5.6,0.7) {\tiny{$v^2_m$}};
        \node[nodelbl] (vmd) at (5.6,-0.7) {\tiny{$v^{d_m}_{m}$}};
        \node[yshift=0.25em] at ($(vm2)!0.5!(vmd)$) {\vdots};
        \node[nodelbl] (vm) at (6.6,0) {\tiny{$v_m$}};

        \foreach \edge in {
            (v0) -- (v11),
            (v0) -- (v12),
            (v0) -- (v1d),
            (v11) -- (v1),
            (v12) -- (v1),
            (v1d) -- (v1),
            (v1) -- (zzz),
            (zzz) -- (vm-1),
            (vm-1) -- (vm1),
            (vm-1) -- (vm2),
            (vm-1) -- (vmd),
            (vm1) -- (vm),
            (vm2) -- (vm),
            (vmd) -- (vm)
        } {
            \StandardPath \edge;
        }
    \end{tikzpicture} 
    \caption{Conversion of Multi-task MAB to DAG}
\end{figure}

Next, for any loss vector $\bfy_t : \llbracket d \rrbracket \to \mathbb{R}$, we define a weight function $w_t : E \to \mathbb{R}$ as follows:
\[
    w_t(e) = \begin{cases}
        \bfy_t[d_{1:i-1}+j] & \textrm{if } e=(v_{i-1},v_i^j) \textrm{ for some } i\in\llbracket m\rrbracket, j \in \llbracket d_i \rrbracket \\
        0 & \textrm{otherwise}
    \end{cases}.
\]

We now apply our FTRL algorithm from \Cref{sec:alg-equal} to the DAG $G$. At each round $t$, if the FTRL algorithm selects a path $P_t$ in $G$, we choose $\bfx_t \in \calX$ such that for any $i \in \llbracket m \rrbracket$ and any $j\in \llbracket d_i \rrbracket$, $\bfx_t[d_{1:i-1}+j] = 1$ if the edge $(v_{i-1}, v_i^j)$ is part of the path $P_t$, and $\bfx_t[d_{1:i-1}+j] = 0$ otherwise. By the construction of $w_t$, it follows that $\langle \bfx_t, \bfy_t \rangle = w_t(P_t)$. Consequently, we provide $\langle \bfx_t, \bfy_t \rangle$ as the bandit feedback for the path $P_t$ to the FTRL algorithm. The way we choose $\bfx_t$ induces a bijective mapping between the set of vectors in $\calX$ and the set of paths in $G$, ensuring correctness. Moreover, our algorithm is computationally efficient.

For our FTRL approach in \Cref{sec:alg-equal}, instead of equating all the coordinates of $\boldsymbol{\gamma}$ to the same fixed value, we assign it differently. Then using our analysis in \Cref{appendix:alg-equal}, one can easily show that the following holds with probability at least $1-\delta$:
\begin{align*}
    R_T(\bfx)\leq c_1\cdot d_\star\sqrt{T}+c_2\cdot \left(T\sum_{v\in V}\boldsymbol{\gamma}[v]+T\sum_{e\in E}\boldsymbol{\gamma}[e]+\sum_{v\in V}\bfx[v]\cdot\frac{\log(d/\delta)}{\boldsymbol{\gamma}[v]}+\sum_{e\in E}\bfx[e]\cdot\frac{\log(d/\delta)}{\boldsymbol{\gamma}[e]}\right)
\end{align*}
where $d_\star:=\max_{\bfx\in \operatorname{co}(\calX) }\sum_{e\in E}\sqrt{\bfx[e]}+\sum_{v\in V} \sqrt{\bfx[v_i]}$

We now assign values to each coordinate of $\boldsymbol{\gamma}$ as follows. For all $i\in \llbracket 0,m\rrbracket$, we have:
\begin{equation*}
    \boldsymbol{\gamma}[v_i]=\sqrt{\frac{\log(d/\delta)}{T}}.
\end{equation*}
Next for all $i\in \llbracket m\rrbracket$ and all $j\in\llbracket d_i\rrbracket$, we have:
\begin{equation*}
    \boldsymbol{\gamma}[v_i^j]=\boldsymbol{\gamma}[(v_{i-1}, v_i^j)]=\boldsymbol{\gamma}[(v_i^j, v_i)]=\sqrt{\frac{\log(d/\delta)}{d_iT}}.
\end{equation*}

Having defined the vector $\boldsymbol{\gamma}$, we now upper bound the second term in the regret above. First, we have:
\begin{equation*}
    T\sum_{i=1}^m\boldsymbol{\gamma}[v_i]+\sum_{i=1}^m\frac{\log(d/\delta)}{\boldsymbol{\gamma}[v_i]}\leq 2m\sqrt{T\log(d/\delta)}\leq 2\sum_{i=1}^m\sqrt{d_iT\log(d/\delta)}
\end{equation*}
where we get the inequality due to the fact that $d_i\geq 1$ for all $i\in\llbracket m \rrbracket$.

Next, we have
\begin{align*}
    &\quad T\sum_{i=1}^m\sum_{j=1}^{d_i}\left(\boldsymbol{\gamma}[v_i^j]+\boldsymbol{\gamma}[(v_{i-1},v_i^j)]+\boldsymbol{\gamma}[(v_i^j,v_i)]\right)\\
    &+\sum_{i=1}^m\sum_{j=1}^{d_i}\left(\frac{\bfx[v_i^j]\cdot\log(d/\delta)}{\boldsymbol{\gamma}[v_i^j]}+\frac{\bfx[(v_{i-1},v_i^j)]\cdot\log(d/\delta)}{\boldsymbol{\gamma}[(v_{i-1},v_i^j)]}+\frac{\bfx[(v_i^j,v_i)]\cdot\log(d/\delta)}{\boldsymbol{\gamma}[(v_i^j,v_i)]}\right)\\
     &= T\sum_{i=1}^m\sum_{j=1}^{d_i}\left(\boldsymbol{\gamma}[v_i^j]+\boldsymbol{\gamma}[(v_{i-1},v_i^j)]+\boldsymbol{\gamma}[(v_i^j,v_i)]\right)\\
    &+\sum_{i=1}^m\sum_{j=1}^{d_i}\left(\frac{\bfx[v_i^j]\cdot\log(d/\delta)}{\boldsymbol{\gamma}[v_i^j]}+\frac{\bfx[v_i^j]\cdot\log(d/\delta)}{\boldsymbol{\gamma}[(v_{i-1},v_i^j)]}+\frac{\bfx[v_i^j]\cdot\log(d/\delta)}{\boldsymbol{\gamma}[(v_i^j,v_i)]}\right)\tag{as $\bfx[v_i^j]=\bfx[(v_{i-1},v_i^j)]=\bfx[(v_i^j,v_i)]$}\\
    &= 3T \sum_{i=1}^m\sum_{j=1}^{d_i} \sqrt{\frac{\log(d/\delta)}{d_iT}}+3\sum_{i=1}^m \sqrt{d_iT\log(d/\delta)}\\
    &= 6\sum_{i=1}^m \sqrt{d_iT\log(d/\delta)}
 \end{align*}
where we get the second equality due to the fact that for any $i\in \llbracket m\rrbracket$, there exists exactly one index $j\in\llbracket d_i\rrbracket$ such that $\bfx[v_i^j]=1$.

Hence, the second term of the regret above is upper bounded by $8\sum_{i=1}^m \sqrt{d_iT\log(d/\delta)}$.

Next, we upper bound $d_\star$. Fix any flow $\bfx\in \operatorname{co}(\calX)$. First observe that $\bfx[v_i]=1$ for all $i\in\llbracket 0,m \rrbracket$. Due to the flow constraints, we also have $\sum_{j=1}^{d_i}\bfx[(v_{i-1},v_i^j)]=1$ for any $i\in \llbracket m \rrbracket$. Next observe that for any $i \in \llbracket m \rrbracket$ and any $j\in \llbracket d_i \rrbracket$, $\bfx[(v_{i-1},v_i^j)]=\bfx[v_i^j]=\bfx[(v_i^j,v_i)]$. Now we have the following:
\begin{align*}
    d_\star&= 3\sum_{i=1}^m\sum_{j=1}^{d_i} \sqrt{\bfx[(v_{i-1},v_i^j)]}+m+1 \leq 3\sum_{i=1}^m\sqrt{d_i}+m+1\leq 5\sum_{i=1}^m\sqrt{d_i}
\end{align*}
where we get the first inequality due to Cauchy-Swartz and we get the second equality as $m\leq\sum_{i=1}^m\sqrt{d_i}$. 

Hence, we obtain a high-probability regret upper bound of $\tilde{\mathcal{O}}(\sum_{i=1}^m\sqrt{d_i T})$. In \Cref{appendix:multi-MAB-lower-bound}, we show that this bound is nearly tight by proving a lower bound of $\Omega(\sum_{i=1}^m\sqrt{d_i T})$. Moreover, this bound can be significantly better than the $\mathcal{O}(\sqrt{dT\log |\mathcal{X}|})$ bound obtained using EXP3 with Kiefer-Wolfowitz exploration. For instance, if $d_i = 2$ for all $i \in \llbracket m-1 \rrbracket$ and $d_m = m^2$, our regret bound is $\tilde{\mathcal{O}}(\sqrt{m^2T})$, whereas EXP3 with Kiefer-Wolfowitz exploration incurs a regret of at least $\Omega(\sqrt{m^3T})$. We refer the reader to \Cref{appendix:exp3-lower-bound} for a detailed discussion of this lower bound for EXP3 with Kiefer-Wolfowitz exploration.  

In \Cref{appendix:multi-mab-simple}, we present a much simpler approach to solving the multi-task MAB problem. In \Cref{appendix:minimax-lower}, we establish a minimax regret lower bound on DAGs using the reduction presented in this section.

\subsection{Extensive-form games under Bandit feedback}
Extensive-form games under Bandit feedback can be modeled as follows. There is a set of decision nodes $\mathcal{X}$, a set of observation nodes $\mathcal{Y}$, and a set of terminal nodes $\mathcal{Z}$. Each decision node $x \in \mathcal{X}$ is associated with a set of actions $A_x$, while each observation node $y \in \mathcal{Y}$ is associated with a set of actions $B_y$. W.l.o.g let us assume that $|A_x|>1$ for all $\bfx\in\calX$ and $|B_y|>1$ for all $y\in\calY$. The non-terminal nodes are governed by injective transition functions: $\rho_x : A_x \to \mathcal{Y} \cup \mathcal{Z}$ for decision nodes, and $\rho_y : B_y \to \mathcal{X}\cup \mathcal{Z}$ for observation nodes. For any two distinct nodes $v_1,v_2$, the ranges of $\rho_{v_1}$ and $\rho_{v_2}$ have empty intersection.

At each round $t$, for every decision node $x \in \mathcal{X}$, we select an action $a_{x,t} \in A_x$. Similarly, for every observation node $y \in \mathcal{Y}$, an adversary selects an action $b_{y,t} \in B_y$. The adversary also specifies a loss function $\mathbf{y}_t : \mathcal{Z} \to [-1, 1]$ for the terminal nodes. Starting from the root node $x^\mathsf{r} \in \mathcal{X}$, the game proceeds as follows: 

\begin{itemize}
    \item If a decision node $x$ is reached, the next node that is visited is $\rho_x[a_{x,t}]$.
    \item If an observation node $y$ is reached, the next node that is visited is $\rho_y[b_{y,t}]$.
    \item If a terminal node $z$ is reached, the process terminates, and we incur a loss of $\mathbf{y}_t[z]$.
\end{itemize} 
Note that in an extensive-form game, no node is visited more than once.

Let $\mathbf{a}_t := \{a_{x,t}\}_{x \in \mathcal{X}}$ and $\mathbf{b}_t := \{b_{y,t}\}_{y \in \mathcal{Y}}$ represent the configurations of actions at decision and observation nodes, respectively. Define $z(\mathbf{a}_t, \mathbf{b}_t)$ as the terminal node reached when transitioning according to $\mathbf{a}_t$ and $\mathbf{b}_t$. As a decision maker, we observe only the loss $\bfy_t[z(\mathbf{a}_t, \mathbf{b}_t)]$ incurred at the terminal node $z(\mathbf{a}_t, \mathbf{b}_t)$; the sequence of nodes visited during the process remains unobserved. Let $\mathcal{A}$ denote the set of all possible configurations of actions at decision nodes. The objective is to minimize the regret:

\[
\mathrm{Regret}(T) := \max_{\mathbf{a} \in \mathcal{A}} \sum_{t=1}^T \big(\bfy_t[z(\mathbf{a}_t, \mathbf{b}_t)] - \bfy_t[z(\mathbf{a}, \mathbf{b}_t)]\big).
\]

One can reformulate this problem as an adversarial linear bandit problem and use EXP3 with Kiefer-Wolfowitz exploration to get a regret upper bound of $\calO(\sqrt{|\calZ|T\log(N)})$ where $N$ is defined as follows. For each terminal node $z$ we define $n(z):=1$. For each decision node $x$, we define $n(x)=\sum_{a\in A_x}n(\rho_x[a])$, and for each observation node $y$, we define $n(y)=\prod_{a\in A_y}n(\rho_y[a])$. Now define $N:=n(x^\mathsf{r})$ where $x^\mathsf{r}$ is the root node. We refer the reader to \Cref{appendix:extensive-linear} for more details. We now reduce extensive-form games to a problem on DAG and show that our approach incurs a regret of $\tilde \calO(\sqrt{|\calZ|T\log(N)})$.

We define a DAG $G=(V,E)$ as follows. Let $V=\{u_{\mathsf{s}},u_{\mathsf{t}}:u\in \calX\cup\calY\cup\calZ\}$. For each terminal node $z\in\calZ$, we add the edge $(z_{\mathsf{s}},z_{\mathsf{t}})$ to $E$. Next for each decision node $\bfx\in \calX$, we add the set of edges $\{(x_{\mathsf{s}},u_{\mathsf{s}}),(u_{\mathsf{t}},x_{\mathsf{t}}): u\in \{\rho_x[a] :a\in A_x\}\}$ to $E$. Finally for each observation node $y\in \calY$, let us index the nodes in $\{\rho_y[b] :b\in B_x\}$ as $\{u^{(1)},u^{(2)},\ldots,u^{(\ell)}\}$. We first add the set of edges $\{(y_{\mathsf{s}},u^{(1)}_{\mathsf{s}}),(u^{(\ell)}_{\mathsf{t}},y_{\mathsf{t}})\}$ to $E$. If $\ell>1$, we then add the set of edges $\{(u^{(i)}_{\mathsf{t}},u^{(i+1)}_{\mathsf{s}}):i\in \llbracket \ell-1 \rrbracket\}$ to $E$. It is easy to observe that $x^\mathsf{r}_\mathsf{s}$ is the source node of $G$ and $x^\mathsf{r}_\mathsf{t}$ is the sink node of $G$. We refer the reader to \Cref{fig:extensive-1,fig:extensive-2} for two examples of this construction.

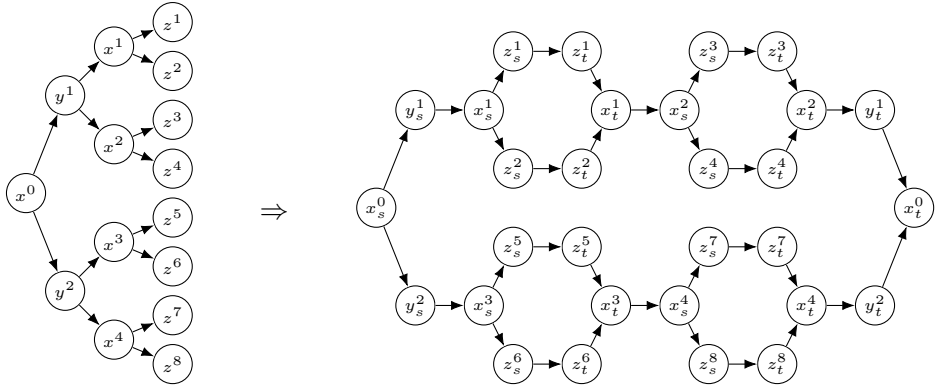
\begin{figure}[h]
    \centering
    \tikzstyle{nodelbl}=[draw,fill=white,circle,inner sep=-.5mm, minimum width=5.2mm]
    \begin{tikzpicture}[scale=1.3]
        \node[nodelbl] (x0) at (.1,0) {\tiny{$x^0$}};
        \node[nodelbl] (y1) at (.5,1) {\tiny{$y^1$}};
        \node[nodelbl] (y2) at (.5,-1) {\tiny{$y^2$}};
        \node[nodelbl] (x1) at (1,1.5) {\tiny{$x^1$}};
        \node[nodelbl] (x2) at (1,0.5) {\tiny{$x^2$}};
        \node[nodelbl] (x3) at (1,-0.5) {\tiny{$x^3$}};
        \node[nodelbl] (x4) at (1,-1.5) {\tiny{$x^4$}};
        \node[nodelbl] (z1) at (1.6,1.75) {\tiny{$z^1$}};
        \node[nodelbl] (z2) at (1.6,1.25) {\tiny{$z^2$}};
        \node[nodelbl] (z3) at (1.6,0.75) {\tiny{$z^3$}};
        \node[nodelbl] (z4) at (1.6,0.25) {\tiny{$z^4$}};
        \node[nodelbl] (z5) at (1.6,-0.25) {\tiny{$z^5$}};
        \node[nodelbl] (z6) at (1.6,-0.75) {\tiny{$z^6$}};
        \node[nodelbl] (z7) at (1.6,-1.25) {\tiny{$z^7$}};
        \node[nodelbl] (z8) at (1.6,-1.75) {\tiny{$z^8$}};

        \foreach \edge in {
            (x0) -- (y1),
            (x0) -- (y2),
            (y1) -- (x1),
            (y1) -- (x2),
            (y2) -- (x3),
            (y2) -- (x4),
            (x1) -- (z1),
            (x1) -- (z2),
            (x2) -- (z3),
            (x2) -- (z4),
            (x3) -- (z5),
            (x3) -- (z6),
            (x4) -- (z7),
            (x4) -- (z8)
        } {
            \StandardPath \edge;
        }
    \end{tikzpicture} 
    \raisebox{2.2cm}{$\qquad\Rightarrow\qquad$}
    \begin{tikzpicture}[scale=1.3]

        \node[nodelbl] (x0s) at (0.1,0)   {\tiny{$x^0_s$}};
        \node[nodelbl] (x0t) at (5.6,0)   {\tiny{$x^0_t$}};
        
        \node[nodelbl] (y1s) at (0.5,1)   {\tiny{$y^1_s$}};
        \node[nodelbl] (x1s) at (1.2,1)   {\tiny{$x^1_s$}};
        \node[nodelbl] (x1t) at (2.5,1)   {\tiny{$x^1_t$}};
        \node[nodelbl] (x2s) at (3.2,1)  {\tiny{$x^2_s$}}; 
        \node[nodelbl] (x2t) at (4.5,1)  {\tiny{$x^2_t$}}; %
        \node[nodelbl] (y1t) at (5.2,1)  {\tiny{$y^1_t$}}; %
        
        \node[nodelbl] (y2s) at (0.5,-1)   {\tiny{$y^2_s$}};
        \node[nodelbl] (x3s) at (1.2,-1)   {\tiny{$x^3_s$}};
        \node[nodelbl] (x3t) at (2.5,-1)   {\tiny{$x^3_t$}};
        \node[nodelbl] (x4s) at (3.2,-1)   {\tiny{$x^4_s$}};
        \node[nodelbl] (x4t) at (4.5,-1)   {\tiny{$x^4_t$}};
        \node[nodelbl] (y2t) at (5.2,-1)   {\tiny{$y^2_t$}};

        \node[nodelbl] (z1s) at (1.5,1.6)   {\tiny{$z^1_s$}};
        \node[nodelbl] (z1t) at (2.2,1.6)   {\tiny{$z^1_t$}}; %
        \node[nodelbl] (z3s) at (3.5,1.6)   {\tiny{$z^3_s$}}; %
        \node[nodelbl] (z3t) at (4.2,1.6)   {\tiny{$z^3_t$}}; %
        
        \node[nodelbl] (z2s) at (1.5,0.4)   {\tiny{$z^2_s$}};
        \node[nodelbl] (z2t) at (2.2,0.4)   {\tiny{$z^2_t$}};
        \node[nodelbl] (z4s) at (3.5,0.4)   {\tiny{$z^4_s$}};
        \node[nodelbl] (z4t) at (4.2,0.4)   {\tiny{$z^4_t$}};
        
        \node[nodelbl] (z5s) at (1.5,-0.4)  {\tiny{$z^5_s$}};
        \node[nodelbl] (z5t) at (2.2,-0.4)  {\tiny{$z^5_t$}};
        \node[nodelbl] (z7s) at (3.5,-0.4)  {\tiny{$z^7_s$}};
        \node[nodelbl] (z7t) at (4.2,-0.4)  {\tiny{$z^7_t$}};
        
        \node[nodelbl] (z6s) at (1.5,-1.6)  {\tiny{$z^6_s$}};
        \node[nodelbl] (z6t) at (2.2,-1.6)  {\tiny{$z^6_t$}};
        \node[nodelbl] (z8s) at (3.5,-1.6)  {\tiny{$z^8_s$}};
        \node[nodelbl] (z8t) at (4.2,-1.6)  {\tiny{$z^8_t$}};

        \foreach \edge in {
            (x0s) -- (y1s),
            (y1t) -- (x0t),
            (x0s) -- (y2s),
            (y2t) -- (x0t),
            (y1s) -- (x1s),
            (x1t) -- (x2s),
            (x2t) -- (y1t),
            (y2s) -- (x3s),
            (x3t) -- (x4s),
            (x4t) -- (y2t),
            (x1s) -- (z1s),
            (z1t) -- (x1t),
            (x1s) -- (z2s),
            (z2t) -- (x1t),
            (x2s) -- (z3s),
            (z3t) -- (x2t),
            (x2s) -- (z4s),
            (z4t) -- (x2t),
            (x3s) -- (z5s),
            (z5t) -- (x3t),
            (x3s) -- (z6s),
            (z6t) -- (x3t),
            (x4s) -- (z7s),
            (z7t) -- (x4t),
            (x4s) -- (z8s),
            (z8t) -- (x4t),
            (z1s) -- (z1t),
            (z2s) -- (z2t),
            (z3s) -- (z3t),
            (z4s) -- (z4t),
            (z5s) -- (z5t),
            (z6s) -- (z6t),
            (z7s) -- (z7t),
            (z8s) -- (z8t),
        } {
            \StandardPath \edge;
        }
    \end{tikzpicture} 
    \caption{Example extensive-form game 1}\label{fig:extensive-1}
\end{figure}
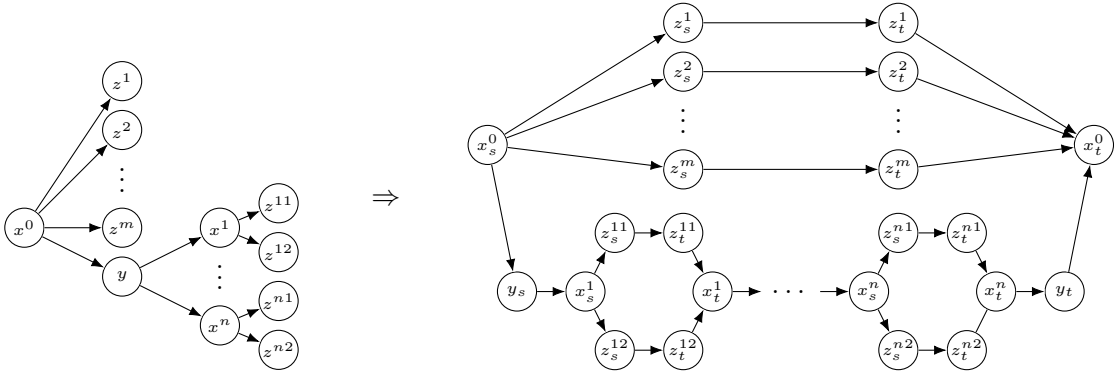
\begin{figure}[h]
    \centering
    \tikzstyle{nodelbl}=[draw,fill=white,circle,inner sep=-.5mm, minimum width=5.2mm]
    \begin{tikzpicture}[scale=1.3]
        \node[nodelbl] (x0) at (0,-1) {\tiny{$x^0$}};
        \node[nodelbl,label=center:{\tiny{$z^{1}$}}] (z01) at (1.,0.5) {};
        \node[nodelbl,label=center:{\tiny{$z^{2}$}}] (z02) at (1.,0.) {};
        \node[nodelbl,label=center:{\tiny{$z^{m}$}}] (z0m) at (1.,-1) {};
        \node[yshift=0.25em] at ($(z02)!0.5!(z0m)$) {\vdots};
        \node[nodelbl] (y) at (1.,-1.5) {\tiny{$y$}};
        \node[nodelbl] (x1) at (2,-1.) {\tiny{$x^1$}};
        \node[nodelbl] (xn) at (2,-2) {\tiny{$x^n$}};
        \node[yshift=0.25em] at ($(x1)!0.5!(xn)$) {\vdots};
        \node[nodelbl,label=center:{\tiny{$z^{11}$}}] (z11) at (2.6,-0.75) {};
        \node[nodelbl,label=center:{\tiny{$z^{12}$}}] (z12) at (2.6,-1.25) {};
        \node[nodelbl,label=center:{\tiny{$z^{n1}$}}] (zn1) at (2.6,-1.75) {};
        \node[nodelbl,label=center:{\tiny{$z^{n2}$}}] (zn2) at (2.6,-2.25) {};

        \foreach \edge in {
            (x0) -- (z01),
            (x0) -- (z02),
            (x0) -- (z0m),
            (x0) -- (y),
            (y) -- (x1),
            (y) -- (xn),
            (x1) -- (z11),
            (x1) -- (z12),
            (xn) -- (zn1),
            (xn) -- (zn2),
        } {
            \StandardPath \edge;
        }
    \end{tikzpicture} 
    \raisebox{2.2cm}{$\qquad\Rightarrow\qquad$}
    \begin{tikzpicture}[scale=1.3]
        \node[nodelbl] (x0s) at (0.7,0.) {\tiny{$x^0_s$}};
        \node[nodelbl] (x0t) at (6.9,0.) {\tiny{$x^0_t$}};
        \node[nodelbl,label=center:{\tiny{$z^{1}_s$}}] (z01s) at (2.7,1.25) {};
        \node[nodelbl,label=center:{\tiny{$z^{1}_t$}}] (z01t) at (4.9,1.25) {};
        \node[nodelbl,label=center:{\tiny{$z^{2}_s$}}] (z02s) at (2.7,0.75) {};
        \node[nodelbl,label=center:{\tiny{$z^{2}_t$}}] (z02t) at (4.9,0.75) {};
        \node[nodelbl,label=center:{\tiny{$z^{m}_s$}}] (z0ms) at (2.7,-0.25) {};
        \node[nodelbl,label=center:{\tiny{$z^{m}_t$}}] (z0mt) at (4.9,-0.25) {};
        \node[yshift=0.25em] at ($(z02s)!0.5!(z0ms)$) {\vdots};
        \node[yshift=0.25em] at ($(z02t)!0.5!(z0mt)$) {\vdots};
        \node[nodelbl] (ys) at (1.,-1.5) {\tiny{$y_s$}};
        \node[nodelbl] (yt) at (6.6,-1.5) {\tiny{$y_t$}};
        \node[nodelbl] (x1s) at (1.7,-1.5) {\tiny{$x^1_s$}};
        \node[nodelbl] (x1t) at (3.0,-1.5) {\tiny{$x^1_t$}};
        \node[nodelbl,label=center:{\tiny{$z^{11}_s$}}] (z11s) at (2.0,-0.9) {};
        \node[nodelbl,label=center:{\tiny{$z^{11}_t$}}] (z11t) at (2.7,-0.9) {};
        \node[nodelbl,label=center:{\tiny{$z^{12}_s$}}] (z12s) at (2.0,-2.1) {};
        \node[nodelbl,label=center:{\tiny{$z^{12}_t$}}] (z12t) at (2.7,-2.1) {};
        
        \node (zzz) at (3.8, -1.5) {\dots};

        \node[nodelbl] (xns) at (4.6,-1.5) {\tiny{$x^n_s$}};
        \node[nodelbl,label=center:{\tiny{$z^{n1}_s$}}] (zn1s) at (4.9,-0.9) {};
        \node[nodelbl,label=center:{\tiny{$z^{n1}_t$}}] (zn1t) at (5.6,-0.9) {};
        \node[nodelbl,label=center:{\tiny{$z^{n2}_s$}}] (zn2s) at (4.9,-2.1) {};
        \node[nodelbl,label=center:{\tiny{$z^{n2}_t$}}] (zn2t) at (5.6,-2.1) {};
        \node[nodelbl] (xnt) at (5.9,-1.5) {\tiny{$x^n_t$}};

        \foreach \edge in {
            (x0s) -- (z01s),
            (z01s) -- (z01t),
            (z01t) -- (x0t),
            (x0s) -- (z02s),
            (z02s) -- (z02t),
            (z02t) -- (x0t),
            (x0s) -- (z0ms),
            (z0ms) -- (z0mt),
            (z0mt) -- (x0t),
            (x0s) -- (ys),
            (yt) -- (x0t),
            (ys) -- (x1s),
            (x1s) -- (z11s),
            (z11s) -- (z11t),
            (z11t) -- (x1t),
            (x1s) -- (z12s),
            (z12s) -- (z12t),
            (z12t) -- (x1t),
            (x1t) -- (zzz),
            (zzz) -- (xns),
            (xns) -- (zn1s),
            (zn1s) -- (zn1t),
            (zn1t) -- (xnt),
            (xns) -- (zn2s),
            (zn2s) -- (zn2t),
            (zn2t) -- (xnt)
            (xnt) -- (yt),
        } {
            \StandardPath \edge;
        }
    \end{tikzpicture} 
    \caption{Example extensive-form game 2}\label{fig:extensive-2}
\end{figure}

Next, for any loss function $\bfy_t : \mathcal{Z} \to [-1, 1]$ and configuration of actions $\bfb_t$, we define a corresponding weight function $w_t : E \to \mathbb{R}$ for the edges of $G$ as follows. For each edge $e=(v_{\mathsf{s}},v_{\mathsf{t}})$ such that $\exists \bfa\in\calA$ such that $v=z(\bfa,\bfb_t)$, we define $w_t(e):=\bfy_t[z(\bfa,\bfb_t)]$. For rest of the edges $e$, we define $w_t(e):=0$.

We now apply our FTRL algorithm from \Cref{sec:alg-centroid} to the DAG $G$. At each round $t$, if the FTRL algorithm selects a path $P_t$ in $G$, we choose $\bfa_t \in \calA$ as follows. For each $\bfx\in \calX$, if $(x_{\mathsf{s}},u_{\mathsf{s}})\in P_t$ for some $u\in \calY\cup\calZ$, we assign $a_{x,t}=\rho_x^{-1}[u]$; otherwise we arbitrarily choose $a_{x,t}$ as the node $x$ will not be reached in the current time-step by the decision process.  By the construction of $w_t$, it follows that $\bfy_t[z(\bfa_t,\bfb_t)] = w_t(P_t)$. Consequently, we provide $\bfy_t[z(\bfa_t,\bfb_t)]$ as the bandit feedback for the path $P_t$ to the FTRL algorithm. The way we choose $\bfa_t$ induces a bijective mapping between the set of configurations in $\calA$ and the set of paths in $G$, ensuring correctness. Moreover, our algorithm is computationally efficient.

It can be shown that the number of edges in $G$ is $\calO(|\calZ|)$ and number of paths from the source node to sink node is $N$. We refer the reader to \Cref{appendix:extensive-fact} for the proof of this fact. Hence, we incur a high-probability regret of $\tilde \calO(\sqrt{|\calZ|T\log(N)})$ against an adaptive adversary. This is the first efficient algorithm to match the regret of EXP3 with Kiefer-Wolfowitz exploration, upto logarithmic factors. Our algorithm contributes to the long line of research on learning in extensive-form games under bandit feedback. See \citet{farina2021bandit} for learning in this setting without access to trajectory information (which matches our setting) and \citet{fiegel2023adapting} for learning with additional trajectory information (which differs from our setting).

\subsection{Shortest Walk in Directed Graphs}
Finding the shortest simple path in directed graphs with cycles when edges can have negative weights is $\mathsf{NP}$-hard. This result extends to finding the shortest trail as well. Therefore in this section, we focus on the online shortest walk in directed graphs. In a walk, repeated vertices or edges are permitted. Given a weight function $w_t: E \to \mathbb{R}$, the weight of a walk $W = (v_0, e_1, v_1, \dots, e_k, v_k)$ is defined as the sum of the weights of its edges:
\[
w(W) := \sum_{i=1}^k w(e_i).
\]
Let $G = (V, E)$ be a directed graph with source $\source$ and sink $\sink$. In each round $t$, an agent selects a walk $W_t$ of length at most $K\leq |E|$ from $\source$ to $\sink$, and the environment simultaneously chooses a weight function $w_t(\cdot)$. The agent's goal is to minimize cumulative regret relative to the optimal path:
\[
\mathrm{Regret}(T) := \sum_{t = 1}^T w_t(W_t) - \min_{W \in \mathcal{W}} \sum_{t = 1}^T w_t(W),
\]
where $\mathcal{W}$ is the set of all walks from $\source$ to $\sink$ of length at most $K\leq |E|$. This problem was first studied by \citet{awerbuch2004adaptive} and can be represented as an adversarial linear bandit problem in $\mathbb{R}^{|E|}$.

We solve the online shortest walk problem in directed graph, by reducing the problem to online shortest path problem on DAG. Given the graph $G = (V, E)$, we construct a layered DAG $G^\dagger = (V^\dagger, E^\dagger)$. The vertex set is defined as:
\[
    V^\dagger := \{v^{(i)} : v \in V, i \in \llbracket 0, K \rrbracket\}.
\]

The edge set is defined as:
\[
   E^\dagger := \left\{(u^{(i-1)}, v^{(i)}) : (u, v) \in E\cup\{(\sink,\sink)\}, i \in \llbracket K \rrbracket\right\}.
\]

For the weight function $w_t: E \to \bbR$, we construct a weight assignment $w^\dagger_t:E^\dagger\to \bbR$ as follow:
\[
    w^\dagger_t\big((u^{(i-1)}, v^{(i)})\big) := \ind[u\neq v]\cdot w_t\big((u, v)\big)
\]
for all $(u, v) \in E\cup\{(\sink,\sink)\}$ and $i \in \llbracket K\rrbracket$.

We now consider the online shortest path problem in the DAG $G^\dagger$, where the set of paths consists of those from the source node $\source^{(0)}$ to the sink node $\sink^{(K)}$. We apply our FTRL algorithm from \Cref{sec:alg-equal} to $G^\dagger$, but first, we preprocess $G^\dagger$ to discard redundant nodes and edges that will never be part of any path from $\source^{(0)}$ to $\sink^{(K)}$.  

At each round $t$, if the FTRL algorithm selects a path $P_t$ in $G^\dagger$, we choose $W_t \in \calW$ as follows: for each edge $(u^{(i)}, v^{(i+1)}) \in P_t$ with $u \neq v$, we add $(u, v)$ as the $i$-th edge of the walk $W_t$. By the construction of $w_t^\dagger$, it follows that $w_t[W_t] = w_t^\dagger[P_t]$. Consequently, we provide $w_t[W_t]$ as the bandit feedback for the path $P_t$ to the FTRL algorithm. The way we choose $W_t$ induces a bijective mapping between the set of walks in $\calW$ and the set of paths in $G^\dagger$, ensuring correctness. Moreover, our algorithm is computationally efficient.

Observe that the number of edges in $G^\dagger$ is $\calO(K|E|)$ and length of any path from the source $\source^{(0)}$ to sink $\sink^{(K)}$ is $K$. Hence we incur a high-probability regret of at most $\tilde\calO(\sqrt{K^2|E|T})$ against an adversary. This improves upon the best-known high-probability regret bound of $\tilde{\mathcal{O}}(|E|^2\sqrt{T})$ achieved by \citet{zimmert2022return}.  

An open question remains: is there an efficient algorithm that matches the high-probability regret bound of $\tilde{\mathcal{O}}(\sqrt{K|E|T})$ achieved by EXP3 with Kiefer-Wolfowitz exploration?  

\subsection{Colonel Blotto game}

In a Colonel Blotto game, two players, $A$ and $B$, have $N$ and $M$ soldiers, respectively, which they must allocate across $K$ battlefields. In each round $t$, player $A$ selects an allocation of $N$ soldiers, denoted as $\mathbf{a}_t = (a_{t,1},\ldots,a_{t,K})$, where $a_{t,i} \geq 0$ represents the number of soldiers assigned to battlefield $i$, and the total allocation satisfies $\sum_{i=1}^{K} a_{t,i} = N$. Simultaneously, player $B$ chooses an allocation of $M$ soldiers, given by $\mathbf{b}_t = (b_{t,1},\ldots,b_{t,K})$.  

At each battlefield $i$, player $A$ incurs a loss of $\mathbf{y}_{t,i}[a_{t,i},b_{t,i}]$. Our goal is to control player $A$ and minimize the following regret:

\[
\mathrm{Regret}(T) := \max_{\mathbf{a} \in \mathcal{A}} \sum_{t=1}^T \sum_{i=1}^K \left(\mathbf{y}_{t,i}[a_{t,i},b_{t,i}] - \mathbf{y}_{t,i}[a_{i},b_{t,i}] \right),
\]
where $\mathcal{A}$ denotes the set of all possible allocations of $N$ soldiers across the $K$ battlefields by player $A$. We assume that for any $\mathbf{a} =(a_1,\ldots,a_K)\in \mathcal{A}$, the total loss satisfies $\sum_{i=1}^K \mathbf{y}_{t,i}[a_{i},b_{t,i}] \in [-1,1]$, and player $A$ observes only $\sum_{i=1}^K \mathbf{y}_{t,i}[a_{t,i},b_{t,i}]$ as bandit feedback at the end of round $t$. It is easy to see that this problem can be represented as a combinatorial bandit problem in $\{0,1\}^{KN}$. 

This problem can be reduced to a problem on directed acyclic graphs (DAGs), as first shown by \citet{behnezhad2023fast}. For completeness, we provide the reduction here.  

We construct a DAG $G = (V, E)$ as follows. Define the vertex set as:  
\[
V = \{ v_{0}^0, v_{K}^N \} \cup \{ v_{i}^j \mid i \in \llbracket K-1 \rrbracket, j \in \llbracket 0, N \rrbracket \}.
\]

Next, we add the following sets of edges to $E$:  
\begin{itemize}
    \item \(\{ (v_{0}^0, v_{1}^j) \mid j \in \llbracket 0, N \rrbracket \}\),  
    \item \(\{ (v_{i-1}^{j_0}, v_{i}^{j_1}) \mid i \in \llbracket 1,K-1 \rrbracket, j_0 \in \llbracket 0, N \rrbracket, j_1 \in \llbracket j_0, N \rrbracket \}\),
    \item \(\{ (v_{K-1}^j, v_{K}^N) \mid j \in \llbracket 0, N \rrbracket \}\).  
\end{itemize}
Observe that $\source = v_{0}^0$ serves as the source node and $\sink = v_{K}^N$ as the sink node.  

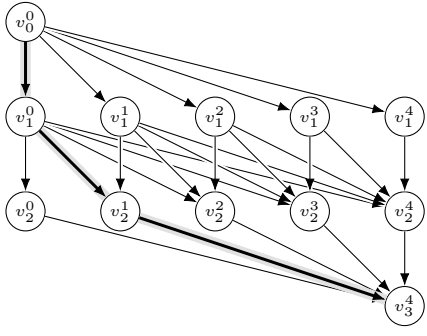
\begin{figure}[h]
    \centering
    \tikzstyle{nodelbl}=[draw,fill=white,circle,inner sep=-.5mm, minimum width=5.2mm]
    \begin{tikzpicture}[scale=1.8]
        \node[nodelbl] (v00) at (0., 0.) {\tiny{$v_0^0$}};
        \node[nodelbl] (v10) at (0., -0.7) {\tiny{$v_1^0$}};
        \node[nodelbl] (v11) at (0.7, -0.7) {\tiny{$v_1^1$}};
        \node[nodelbl] (v12) at (1.4, -0.7) {\tiny{$v_1^2$}};
        \node[nodelbl] (v13) at (2.1, -0.7) {\tiny{$v_1^3$}};
        \node[nodelbl] (v14) at (2.8, -0.7) {\tiny{$v_1^4$}};
        \node[nodelbl] (v20) at (0., -1.4) {\tiny{$v_2^0$}};
        \node[nodelbl] (v21) at (0.7, -1.4) {\tiny{$v_2^1$}};
        \node[nodelbl] (v22) at (1.4, -1.4) {\tiny{$v_2^2$}};
        \node[nodelbl] (v23) at (2.1, -1.4) {\tiny{$v_2^3$}};
        \node[nodelbl] (v24) at (2.8, -1.4) {\tiny{$v_2^4$}};
        \node[nodelbl] (v34) at (2.8, -2.1) {\tiny{$v_3^4$}};
        
        \foreach \edge in {
            (v00) -- (v11),
            (v00) -- (v12),
            (v00) -- (v13),
            (v00) -- (v14),
            (v10) -- (v20),
            (v10) -- (v22),
            (v10) -- (v23),
            (v10) -- (v24),
            (v11) -- (v21),
            (v11) -- (v22),
            (v11) -- (v23),
            (v11) -- (v24),
            (v12) -- (v22),
            (v12) -- (v23),
            (v12) -- (v24),
            (v13) -- (v23),
            (v13) -- (v24),
            (v14) -- (v24),
            (v20) -- (v34),
            (v22) -- (v34),
            (v23) -- (v34),
            (v24) -- (v34),
        } {
            \StandardPath \edge;
        }
        \foreach \edge in {
            (v00) -- (v10),
            (v10) -- (v21),
            (v21) -- (v34),
        } {
            \HighlightPath \edge;
        }
    \end{tikzpicture} 
    \caption{Dag for Blotto game with $N = 4$ soldiers and $K = 3$ battlefields. The shaded path corresponds to the allocation $\bfa=(0,1,3)$.}
\end{figure}

For any set of loss functions $\mathbf{y}_{t,i}$ and allocation $\mathbf{b}_t$, we define a corresponding weight function $w_t: E \to \mathbb{R}$ on the edges of $G$ as follows. For any edge $e = (v_{i-1}^{j_0}, v_{i}^{j_1})$, we set  
\[
w_t(e) := \mathbf{y}_{t,i}[j_1 - j_0, b_{t,i}].
\]

We now apply our FTRL algorithm from \Cref{sec:alg-equal} to the DAG $G$. At each round $t$, if the FTRL algorithm selects a path $P_t$ in $G$, we determine $\mathbf{a}_t \in \mathcal{A}$ as follows. Fix an index $i \in \llbracket K \rrbracket$ and consider the edge $e = (v_{i-1}^{j_0}, v_{i}^{j_1}) \in P_t$. We then set  $a_{t,i}=j_1 - j_0.$ By the construction of $w_t$, it follows that $\sum_{i=1}^K \mathbf{y}_{t,i}[a_{t,i}, b_{t,i}] = w_t(P_t).$ Thus, we provide $\sum_{i=1}^K \mathbf{y}_{t,i}[a_{t,i}, b_{t,i}]$ as the bandit feedback for the path $P_t$ to the FTRL algorithm. The way we choose $\bfa_t$ induces a bijective mapping between the set of allocations in $\calA$ and the set of paths in $G$, ensuring correctness. Moreover, our algorithm is computationally efficient.

Observe that the number of edges in $G$ is $\mathcal{O}(K^2N)$, and the length of any path from the source to the sink is $K$. Consequently, we incur a high-probability regret of at most $\tilde{\mathcal{O}}(\sqrt{K^3NT})$ against an adaptive adversary. This improves upon the best-known high-probability regret bound of $\tilde{\mathcal{O}}(K^2N^2\sqrt{T})$ achieved by \citet{zimmert2022return}.  

An open question remains: is there an efficient algorithm that matches the high-probability regret bound of $\tilde{\mathcal{O}}(\sqrt{K^2NT})$ achieved by EXP3 with Kiefer-Wolfowitz exploration?  
\subsection{m-sets}

An $m$-set is the set of vectors $\mathcal{X} := \{ x \in \{0,1\}^d : \|x\|_1 = m \}.$ We construct a DAG $G$ as follows. The vertex set of $G$ is  
\[
V = \{ v_i^j : i \in \llbracket 0, d-m \rrbracket, j \in \llbracket 0, m \rrbracket \}.
\]
The edge set is defined as  
\[
E = \{ (v_i^{j-1}, v_i^j) : i \in \llbracket 0, d-m \rrbracket, j \in \llbracket m \rrbracket \} 
\cup \{ (v_{i-1}^j, v_i^j) : i \in \llbracket d-m \rrbracket, j \in \llbracket 0, m \rrbracket \}.
\]
Note that $\source = v_0^0$ is the source node of $G$, and $\sink = v_{d-m}^m$ is the sink node.  

\begin{figure}[h]
    \centering
    \tikzstyle{nodelbl}=[draw,fill=white,circle,inner sep=-.5mm, minimum width=5.2mm]
    \begin{tikzpicture}[scale=1.8]
        \node[nodelbl] (v00) at (0., 0.) {\tiny{$v_0^0$}};
        \node[nodelbl] (v01) at (0.7, 0.) {\tiny{$v_1^0$}};
        \node[nodelbl] (v02) at (1.4, 0.) {\tiny{$v_2^0$}};
        \node[nodelbl] (v03) at (2.1, 0.) {\tiny{$v_3^0$}};
        \node[nodelbl] (v10) at (0., -0.7) {\tiny{$v_0^1$}};
        \node[nodelbl] (v11) at (0.7, -0.7) {\tiny{$v_1^1$}};
        \node[nodelbl] (v12) at (1.4, -0.7) {\tiny{$v_2^1$}};
        \node[nodelbl] (v13) at (2.1, -0.7) {\tiny{$v_3^1$}};
        \node[nodelbl] (v20) at (0., -1.4) {\tiny{$v_0^2$}};
        \node[nodelbl] (v21) at (0.7, -1.4) {\tiny{$v_1^2$}};
        \node[nodelbl] (v22) at (1.4, -1.4) {\tiny{$v_2^2$}};
        \node[nodelbl] (v23) at (2.1, -1.4) {\tiny{$v_3^2$}};
        
        \foreach \edge in {
            (v02) -- (v03),
            (v10) -- (v11),
            (v11) -- (v12),
            (v20) -- (v21),
            (v21) -- (v22),
            (v22) -- (v23),
            (v00) -- (v10),
            (v10) -- (v20),
            (v01) -- (v11),
            (v11) -- (v21),
            (v12) -- (v22),
            (v03) -- (v13),
        } {
            \StandardPath \edge;
        }
        \foreach \edge in {
            (v00) -- (v01),
            (v01) -- (v02),
            (v02) -- (v12),
            (v12) -- (v13),
            (v13) -- (v23),
        } {
            \HighlightPath \edge;
        }
    \end{tikzpicture} 
    \caption{DAG for $m$-set with $m=2$ and $d=5$. The shaded path corresponds to the vector $\bfx$ such that $\bfx[3]=\bfx[5]=1$ and $\bfx[1]=\bfx[2]=\bfx[4]=0$.}
\end{figure}
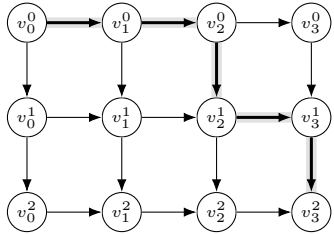

Next, for any loss function $\bfy_t : \llbracket d \rrbracket \to \mathbb{R}$, we define a corresponding weight function $w_t : E \to \mathbb{R}$ for the edges of $G$ as follows. For an edge $e = (v_{i}^j, v_{i+1}^j)$, we set $w_t(e) := 0$. For an edge $e = (v_{i}^{j-1}, v_{i}^{j})$, we define $w_t(e) := \bfy_t[i+j]$.

We now apply our FTRL algorithm from \Cref{sec:alg-equal} to the DAG $G$. At each round $t$, if the FTRL algorithm selects a path $P_t$ in $G$, we choose $\bfx_t \in \calX$ such that for any $k \in \llbracket d \rrbracket$, we set $\bfx_t[k] = 1$ if there exists an edge $(v_{i}^{j-1}, v_{i}^{j}) \in P_t$ with $k = i + j$, and $\bfx_t[k] = 0$ otherwise. By the construction of $w_t$, it follows that $\langle \bfx_t, \bfy_t \rangle = w_t(P_t)$. Consequently, we provide $\langle \bfx_t, \bfy_t \rangle$ as the bandit feedback for the path $P_t$ to the FTRL algorithm. The way we choose $\bfx_t$ induces a bijective mapping between the set of vectors in $\calX$ and the set of paths in $G$, ensuring correctness. Moreover, our algorithm is computationally efficient. 

Finally, observe that each path in $G$ has a length of $d$ and the number of edges in $G$ is $\calO(md)$. Thus, we incur a high-probability regret of $\tilde\calO(\sqrt{md^2T})$ against an adaptive adversary. This improves upon the best-known high-probability regret bound of $\tilde{\mathcal{O}}(d^2\sqrt{T})$ achieved by \citet{zimmert2022return}.  

An open question remains: is there an efficient algorithm that matches the high-probability regret bound of $\tilde{\mathcal{O}}(\sqrt{mdT})$ achieved by EXP3 with Kiefer-Wolfowitz exploration?